\documentclass{article}
\usepackage[top=3cm, bottom=3cm, left=2.5cm, right=2.5cm]{geometry}

\usepackage[utf8]{inputenc}

\usepackage{float}
\usepackage{graphicx}
\usepackage{amsthm}
\usepackage{amsmath, amssymb, mathabx, mathrsfs}
\usepackage[table]{xcolor}
\usepackage{array, multirow}
\usepackage[font=small,labelfont=bf]{caption} 
\usepackage{amsfonts}
\usepackage{tabularx}
\usepackage{booktabs}
\usepackage{graphicx}
\usepackage{epstopdf}
\usepackage{algorithmic}
\usepackage{mathtools}
\usepackage{todonotes}
\usepackage{wrapfig}

\usepackage{subcaption}

\usepackage{enumitem}

    \definecolor{darkgreen}{rgb}{0, .5, 0}
    \definecolor{darkcerulean}{rgb}{0.03, 0.27, 0.49}
    \definecolor{smokyblack}{rgb}{0.06, 0.05, 0.03}
    \definecolor{warmblack}{rgb}{0.0, 0.26, 0.26}
    \definecolor{cobalt}{rgb}{0.0, 0.28, 0.67}
    \definecolor{aoEnglish}{rgb}{0.0, 0.5, 0.0}
    \definecolor{carribeangreen}{rgb}{0.0, 0.8, 0.6}
    \definecolor{persiangreen}{rgb}{0.0, 0.65, 0.58}
\usepackage[pdfauthor={Ruiyang Hong, Anastasis Kratsios},
            pdftitle={Bridging the Gap Between Approximation and Learning via Optimal Approximation by ReLU MLPs of Maximal Regularity},
            pdfsubject={Foundations of Deep Learning},
            pdfkeywords={Lipschitz Neural Networks, Optimal approximation, Generalization Bounds, Constructive Approximation},
            colorlinks=true,
            linkcolor=cobalt,
            urlcolor=darkcerulean,
            citecolor=warmblack]{hyperref}

\usepackage{appendix}
\usepackage[numbers]{natbib}

\usepackage[nameinlink]{cleveref}
\crefname{equation}{}{}
\crefname{section}{section}{sections}
\crefname{figure}{figure}{figures}
\crefname{table}{table}{tables}
\crefname{example}{example}{examples}
\crefname{proposition}{proposition}{propositions}
\Crefname{section}{Section}{Sections}
\Crefname{figure}{Figure}{Figures}
\Crefname{table}{Table}{Tables}
\Crefname{definition}{Definition}{Definitions}
\Crefname{theorem}{Theorem}{Theorems}
\Crefname{corollary}{Corollary}{Theorems}
\Crefname{remark}{Remark}{Remarks}
\Crefname{example}{Example}{Examples}
\Crefname{lemma}{Lemma}{Theorems}
\numberwithin{equation}{section}
\crefname{assumption}{Assumption}{assumptions}

\newtheorem*{theorem*}{Theorem}
\newtheorem{theorem}{Theorem}[section]
\newtheorem{lemma}{Lemma}[section]
\newtheorem{corollary}{Corollary}[section]
\newtheorem{proposition}{Proposition}[section]
\theoremstyle{definition}
\newtheorem{definition}{Definition}[section]
\newtheorem{remark}{Remark}[section]
\newtheorem{example}{Example}[section]
\newtheorem{assumption}{Assumption}[section]

\newcounter{termcounter}
\renewcommand{\thetermcounter}{\Roman{termcounter}}
\crefname{term}{term}{terms}
\creflabelformat{term}{#2\textup{(#1)}#3}

\makeatletter
\def\term{\@ifnextchar[\term@optarg\term@noarg}
\def\term@optarg[#1]#2{%
  \textup{#1}%
  \def\@currentlabel{#1}%
  \def\cref@currentlabel{[][2147483647][]#1}%
  \cref@label[term]{#2}}
\def\term@noarg#1{%
  \refstepcounter{termcounter}%
  \textup{(\thetermcounter)}%
  \cref@label[term]{#1}}
\makeatother

\newcommand{\eqdef}{\ensuremath{\stackrel{\mbox{\upshape\tiny def.}}{=}}}

\definecolor{bulgarianrose}{rgb}{0.28, 0.02, 0.03}

    \makeatletter
\def\@fnsymbol#1{\ensuremath{\ifcase#1\or \dagger\or \ddagger\or
   \mathsection\or \mathparagraph\or \|\or **\or \dagger\dagger
   \or \ddagger\ddagger \else\@ctrerr\fi}}
    \makeatother

\RenewCommandCopy{\overbrace}{\LaTeXoverbrace}
\RenewCommandCopy{\underbrace}{\LaTeXunderbrace}

\title{Bridging the Gap Between Approximation and Learning via Optimal Approximation by ReLU MLPs of Maximal Regularity}

\author{Ruiyang Hong\thanks{1) McMaster University, Hamilton, Canada and 2) Vector Institute, Toronto, Canada}
\,\thanks{\textit{Corresponding author:} hongr5@mcmaster.ca} 
\and Anastasis Kratsios$^{\dagger}$\,%
\thanks{kratsioa@mcmaster.ca}}
\date{}

\begin{document}



\maketitle






\begin{abstract}
The foundations of deep learning are supported by the seemingly opposing perspectives of approximation or learning theory. The former advocates for large/expressive models that need not generalize, while the latter considers classes that generalize but may be too small/constrained to be universal approximators. Motivated by real-world deep learning implementations that are both expressive and statistically reliable, we ask: \textit{``Is there a class of neural networks that is both large enough to be universal but structured enough to generalize?''}

This paper constructively provides a positive answer to this question by identifying a highly structured class of ReLU multilayer perceptions (MLPs), which are optimal function approximators and are statistically well-behaved. We show that any $(L,\alpha)$-H\"{o}lder function from $[0,1]^d$ to $[-n,n]$ can be approximated to a uniform $\mathcal{O}(1/n)$ error on $[0,1]^d$ with a sparsely connected ReLU MLP with the same H\"{o}lder exponent $\alpha$ and coefficient $L$, 
of width $\mathcal{O}(dn^{d/\alpha})$, depth $\mathcal{O}(\log(d))$, with $\mathcal{O}(dn^{d/\alpha})$ nonzero parameters, and whose weights and biases take values in $\{0,\pm 1/2\}$ except in the first and last layers which instead have magnitude at-most $n$.
Unlike previously known ``large'' classes of universal ReLU MLPs, the empirical Rademacher complexity of our class remains bounded even when its depth and width become arbitrarily large. Our class of MLPs achieves a near-optimal sample complexity of $\mathcal{O}(\log(N)/\sqrt{N})$ when given $N$ i.i.d. normalized sub-Gaussian training samples.  

We achieve this through a new construction that perfectly fits together linear pieces using Kuhn triangulations, along with a new proof technique which shows that our construction preserves the regularity of not only the H\"{o}lder functions, but also any uniformly continuous function.  Our results imply that neural networks can solve the McShane extension problem on suitable finite sets.

\end{abstract}




\section{Introduction}
\label{s:Intro}

The foundations of deep learning are typically investigated from two seemingly opposing perspectives: either analytically in terms of the approximation power of neural networks~\cite{yarotsky2017error,petersen2018optimal,elbrachter2021deep,ShenYangZhang_JMPA_OptApprx_ReLU,siegel2023optimal} or statistically by guaranteeing that these models generalize beyond the training data~\cite{bartlett2017spectrally,BartlettHarveyChristopherMehrabian_VCDimReLUMLPs_JMLR_2019}.  
On the one hand, the (universal) approximation literature studies MLPs that approximate continuous functions using a minimal number of neurons, and it typically overlooks the regularity needed for them to exhibit reliable statistical behaviour.  
In contrast, the learning theory literature studies restricted classes of MLPs which generalize beyond their training data by limiting their expressivity either through weights and biases restrictions~\cite{neyshabur2015pathSGD_NIPS,miyato2018spectral,virmaux2018lipschitz} or Lipschitz constraints~\cite{gouk2021regularisation}.
Practical deep-learning implementations, however, exhibit both of these characteristics: They are powerful approximators and reliably generalize. This means that the deep learning theory community still has not identified a class of MLPs that reflects the analytic and statistical properties of real-world deep learning implementations.

This paper addresses this gap between theory and practice by identifying a sub-class of deep ReLU MLPs that exhibits both \textit{optimal approximation rates} when approximating continuous functions on $[0,1]^d$ and which generalizes well, even when overparameterized.
Informally, our main finding is that deep ReLU MLPs of maximal regularity are optimal (universal) approximators.
Our sub-class of maximally regular ReLU MLPs exhibits several familiar properties of real-world deep learning implementations, such as sample interpolation~\cite{jin2023implicit,taheri2024generalization} and small derivatives~\cite{mulayoff2021implicit,heiss2023implicit}.


We find that it is enough to analyze the sub-class of real-valued ReLU MLPs $\mathcal{NN}_{\Delta,W}^{L,n}$ on $\mathbb{R}^d$ whose \textit{size} is constrained by a prescribed depth $\Delta$ and width $W$, and whose \textit{regularity} is limited by restricting their maximal $\alpha$-H\"{o}lder constant $L\ge 0$.  Furthermore, the trainable parameters in $\mathcal{NN}_{\Delta,W}^{L,\alpha,n}$ are highly structured in that, all the weights and biases in their first and last layer are at most $n$, while in all other intermediate layers, all biases are $0$ and all weights are in $\{ 0, \pm 1/2 \}$\footnote{When $\alpha=1$, our networks are $L$-Lipschitz and we drop the notational emphasis on $\alpha$.}.

Our main result shows that structural constraints, imposed by $L$ and $n$, does not hinder the optimal approximation power of $\mathcal{NN}_{\Delta,W}^{L,\alpha,n}$.
\begin{theorem}[Optimal Approximation by ReLU MLPs with Maximal Regularity]
\label{thm:main_IntroVersion}
For any $\alpha$-H\"{o}lder continuous functions $f:[0,1]^d \rightarrow[-n,n]$ (where $\alpha\in(0,1])$ with coefficient $L\ge 0$, there exists an \textbf{$\alpha$-H\"{o}lder continuous ReLU MLP $\Phi\in\mathcal{NN}_{W,D}^{L,\alpha,n}$ with the same coefficient $L$}, which has width $W \le 8d \left(
\left\lceil 
    \frac{1}{2} \left( \frac{L n}{d^{1-\alpha}} \right)^{1/\alpha} \right\rceil
+1 \right)^d$, depth $D \le \lceil \log_2{d} \rceil + 4$, at most 
$16d\left(
\left\lceil 
    \frac{1}{2} \left( \frac{L n}{d^{1-\alpha}} \right)^{1/\alpha} \right\rceil
+1 \right)^d$ nonzero parameters, and it satisfies
\[
        \max_{\mathbf{x}\in [0,1]^d}\,
            |f(\mathbf{x})-\Phi(\mathbf{x})|
    \le 
        \frac{d}{n}
.
\]
\end{theorem}
In fact, the full version of Theorem~\ref{thm:main_IntroVersion} (formulated in Theorem~\ref{thm:MainTheorem_StableUniversalApproximation}) guarantees that $f$ may also be interpolated on a sufficiently fine dyadic grid in $\mathbb{Z}^d$ while exactly matching the H\"{o}lder exponent and constant of the target function.  Therefore, a direct consequence of our main result (Corollary~\ref{cor:McShane}) shows that neural networks can solve the McShane \textit{extension problem}, see~\cite{McShaneOGPaper1993}, on finite dyadic subsets of the cube $[0,1]^d$.  Indeed, such Lipschitz, H\"{o}lder, and uniformly-continuous extension problems have a long and active studied mathematical history~\cite{flett1974extensions,johnson1984extensions,matouvskova2000extensions,lee2005extending,oberman2008explicit,naor2017lipschitz,brue2021linear,beer2020mcshane,ambrosio2020linear}.  The following version of our (Corollary~\ref{cor:McShane}) is the first such extension guarantee for neural networks.
\begin{corollary}[Neural Networks Can Extend H\"{o}lder Functions on Finite Subsets of Dyadic Grids]
\label{cor:McShane__IntroVersion}
Fix any $L\ge 0$, $\alpha\in (0,1]$, any finite $A\subseteq [0,1]^d\bigcap \bigcup_{n \in \mathbb{N}_+}\, \frac1{2^n} \mathbb{Z}^d$, $\alpha$-H\"{o}lder $f:A\to \mathbb{R}$ with coefficient $L$, there exists an $\alpha$-H\"{o}lder ReLU MLP $\Phi:\mathbb{R}^d\to \mathbb{R}$ with coefficient $L$ extending $f$; i.e.\
\[
f(x)=\Phi(x) 
\mbox{ for all } x\in A
.
\]
\end{corollary}

If we relax the Lipschitz and weights constraints in Theorem~\ref{thm:main_IntroVersion}, by sending both $L$ and $n$ to infinity, then we recover the optimal approximation guarantees for \textit{unconstrained} ReLU MLPs, which are now well-known in the approximation theory literature.  However, Theorem~\ref{thm:main_IntroVersion} shows that optimal approximation rates can be achieved while simultaneously imposing regularity on $\Phi$.

Our main result, namely Theorem~\ref{thm:main_IntroVersion}, confirmed that the class of MLPs $\mathcal{NN}_{\Delta,W}^{L,n}$ is rich/large enough to be a universal approximator.  The statistical viability of this class is guaranteed by our second result, which shows this class is tame/small enough to generalize well.

We consider training data for a classification task $\{(X_n,Y_n)\}_{n=1}^N$ in $\mathbb{R}^d\times [0,1]$.  Our second result quantifies the (uniform) \textit{generalization gap} for the class $\mathcal{NN}_{\Delta,W}^{L,n}$, which is defined as the largest absolute difference between the true risk $\mathcal{R}(\Phi)$, computed over the true distribution $(X_1,Y_1)\sim \mathbb{P}$, and the empirical risk $\mathcal{R}_N(\Phi)$, computed on the training data, for any ReLU MLP $\Phi$ in $\mathcal{NN}_{\Delta,W}^{L,n}$, where
\begin{equation*}
    \mathcal{R}(\Phi) \eqdef \mathbb{E}_{(X,Y)\sim \mathbb{P}}\big[\ell(\Phi(X),Y)\big]
    ,\quad
    \mathcal{R}_N(\Phi) \eqdef \frac1{N}\sum_{n=1}^N\ell(\Phi(X_n),Y_n)
.
\end{equation*}
\begin{theorem}[Nearly Optimal Sample Complexity Without Explosion for Deep and Wide MLPs]
\label{thm:RademacherBound__IntroVersion}
Given the sample set $\{(X_n,Y_n)\}_{n=1}^N$ where $\{X_n\}_{n=1}^N$ are i.i.d.\ centered, sub-Gaussian random variables with normalized covariance $\mathbb{E}[X_1X_1^{\top}]=\frac{1}{N}\mathbf{I}_d$, and $\mathbf{I}_d$ is the $d\times d$ identity matrix then, for each $\delta \in (0,1)$ 
\begin{equation*}
\begin{aligned}
    \sup_{\Phi\in \mathcal{NN}_{\Delta,W}^{L,1}} \big|\mathcal{R}(\Phi)-\mathcal{R}_N(\Phi)\big|
\in 
    \tilde{\mathcal{O}}\Biggl(
            \frac{
                \sqrt{\log(4/\delta)}
            }{\sqrt{N}}
    +
        \min\biggl\{
                \frac{
                    W^{3\Delta/2}
                }{
                    2^{\Delta}
                    \sqrt{N}
                }
        ,
                \frac{
                    L^{d/(d+3)}
                }{\sqrt[d+3]{N}}
        \biggr\}
    \Biggr)
\end{aligned}
\end{equation*}
holds with probability at least $1-\delta$.
\end{theorem}
Our generalization bound in Theorem~\ref{thm:RademacherBound__IntroVersion} guarantees a nearly optimal sample complexity of $\mathcal{O}(\log(N)/\sqrt{N})$ observed in parametric generalization bounds of ReLU MLPs, e.g.~\cite{neyshabur2015pathSGD_NIPS,bartlett2017spectrally,BartlettHarveyChristopherMehrabian_VCDimReLUMLPs_JMLR_2019}.  However, unlike parametric generalization bounds, our bound in Theorem~\ref{thm:RademacherBound__IntroVersion} converges even in the overparameterized regime where $\Delta$ and $W$ are allowed to be arbitrarily large (compared to the sample size $N$).  Typically, only non-parametric generalization bounds, such as e.g.~\cite{HouKratsios_JMLR_2023}, do not explode when $\Delta$ and $W$ are taken to be arbitrarily large since they can account for the enumerable parametric symmetries in large neural networks; see e.g.~\cite{entezari2021role,ainsworth2023git} for a discussion on these parametric symmetries.  However, unlike the non-parametric bounds, the sample complexity of generalization bound for $\mathcal{NN}_{\Delta,W}^{L,n}$ converges at a \textit{dimension-free} rate.

\begin{figure}[ht]
    \centering
    \begin{subfigure}[b]{0.45\textwidth}
    \centering
        \includegraphics[width=.8\linewidth]{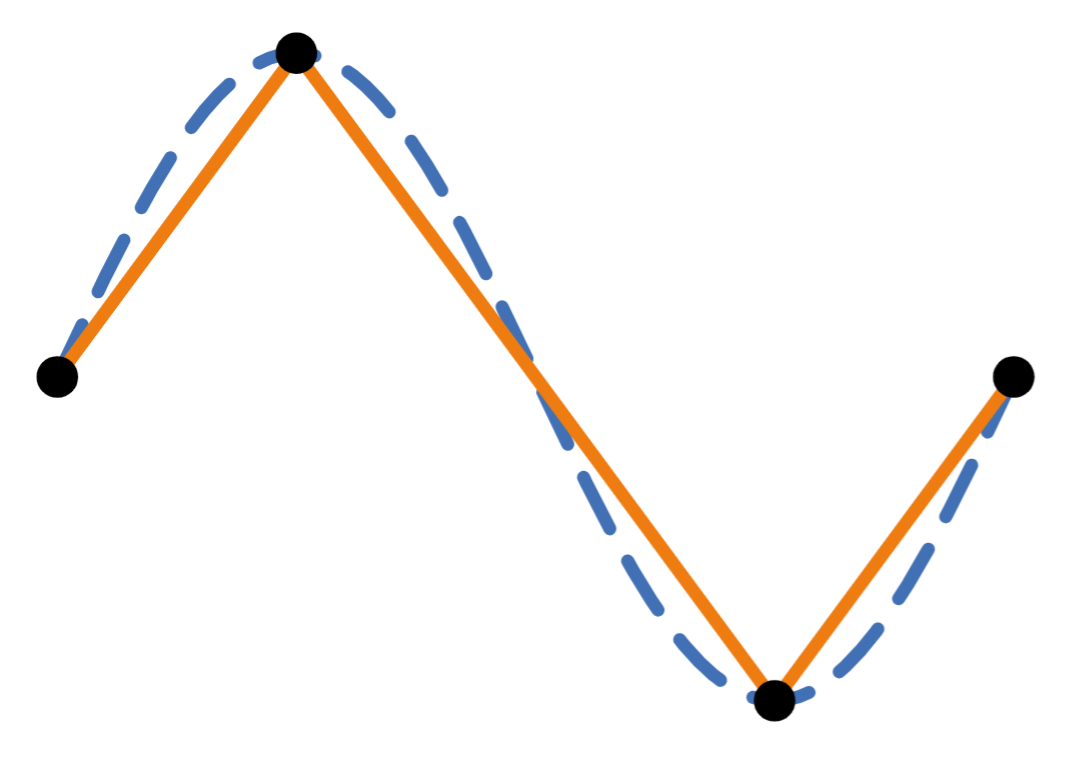}
        \caption{Our Strategy - ``Maximally Regular'' Networks: The networks constructed in Theorem~\ref{thm:main_IntroVersion} interpolate a given grid linearly with the optimal slope between any two grid points.  These ReLU MLPs constructed from finite sample values of the target function can never have Lipschitz constant (or, more generally, modulus of regularity, as shown in Lemma \ref{lem:CPWLApprox_wRegControl}) exceeding that of the target function.   }
        \label{fig:Proof_Strategy_Idea__NEW}
    \end{subfigure}
    \hspace{5pt}
    \begin{subfigure}[b]{0.45\textwidth}
    \centering
        \includegraphics[width=0.8\linewidth]{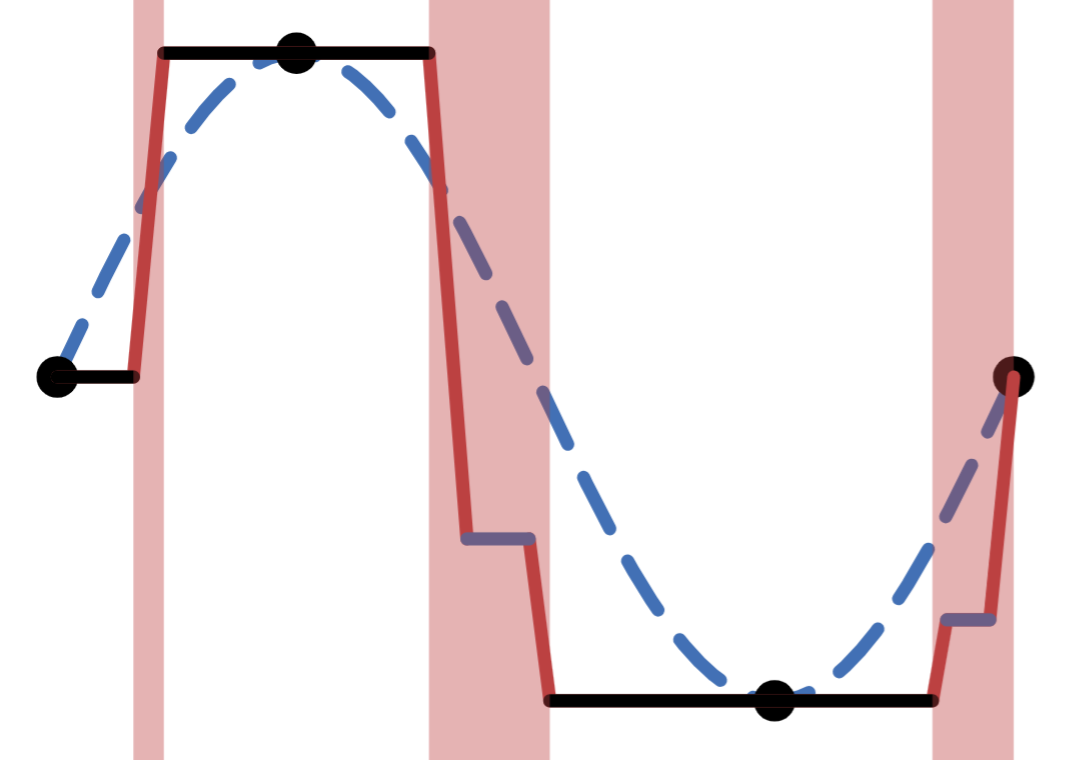}
        \caption{Standard Approach - ``Spiky'' Networks: Following methods such as~\cite{yarotsky2017error,ShenYangZhang_JMPA_OptApprx_ReLU}, ReLU MLPs are constructed by first computing piecewise constant approximators of the target function matching its values at grid points (\textbf{flat black lines}) outside of so-called \textit{trifling regions} (shaded in red).  The piecewise constant approximators are then ``glued together'' via linear interpolation (steep red lines). }
        \label{fig:Proof_Strategy_Idea__OLD}
    \end{subfigure}
    \caption{\textbf{How our method works:} 
    Our method (left) against the usual approach (right) of~\cite{yarotsky2017error,ShenYangZhang_JMPA_OptApprx_ReLU} for ReLU MLPs achieving the optimal convergence rate when approximating a (blue dashed) target function. Both methods memorize the value of the target function (\textbf{black dots}) at a specific set of grid points.
    \hfill\\
    \textbf{What issue does our construction resolve?} 
    The trouble in standard construction (right) is that as approximation becomes more accurate, these small trifling regions become very small, which can result in extremely steep red interpolating segments.  
    Using our Kuhn triangulation-based construction (for the multi-dimensional case), we can construct optimal ReLU approximators with \textit{no trifling regions}, meaning that we do not need these steep red segments; therefore, our ReLU approximators are not irregular.}
    \label{fig:Proof_Strategy_Idea}
\end{figure}

\subsection{Technical Contributions: New Proof Techniques}
We point out that the case where $f$ is only Lipschitz is relatively trivial; see~\cite{KARNIK2024115557,pan2024metricentropylimitsnonlineardynamical,riegler2024generatingrectifiablemeasuresneural,yarotsky_OptimalApproximation_VeryDeep_PMLP_COLT_2018} for similar constructions and see Figure \ref{fig:different_hats} for their illustrations. However, Theorem~\ref{thm:main_IntroVersion} applies to a much broader class of functions (e.g. H\"{o}lder functions), which are beyond the scope of~\cite{KARNIK2024115557,pan2024metricentropylimitsnonlineardynamical,riegler2024generatingrectifiablemeasuresneural} and improved the result of~\cite{yarotsky_OptimalApproximation_VeryDeep_PMLP_COLT_2018}.  
Indeed, the sub-Lipschitz case is of high interest as H\"{o}lder functions are standard in stochastic analysis since almost every path of any non-degenerate continuous semi-martingale (e.g.\ diffusion process) is $\alpha$-H\"{o}lder for any $\alpha<1/2$ and all of rough path theory is concerned with integration against paths of sub-Lipschitz regularity~\cite{hambly2010uniqueness,friz2020course}, and sub-Lipschitz functions are standard in inverse-problems~\cite{frankowska1990some}.  Moreover, in analogy with challenges in rough path theory which was exclusively developed to overcome the challenges of pathwise integration against sub-Lipschitz functions, the Lipschitz case of Theorem~\ref{thm:main_IntroVersion} is relatively straightforward. In contrast, the (sub-)H\"{o}lder case is highly non-trivial. This is because, in the Lipschitz case, one only need to guarantee local Lipschitz-ness (i.e. Lipschitz on each simplex) to have global Lipschitz-ness. However, situations are far more complicated in the general case. For example, local H\"{o}lder continuity (i.e. H\"{o}lder continuous on each simplex) does not guarantee global H\"{o}lder continuity with the same H\"{o}lder constants. To prove this result, we invented a new technique based on path arguments; see the proof of Lemma~\ref{lem:CPWLApprox_wRegControl}.

\begin{figure}[!htbp]
    \centering
    \includegraphics[width=1.0\linewidth]{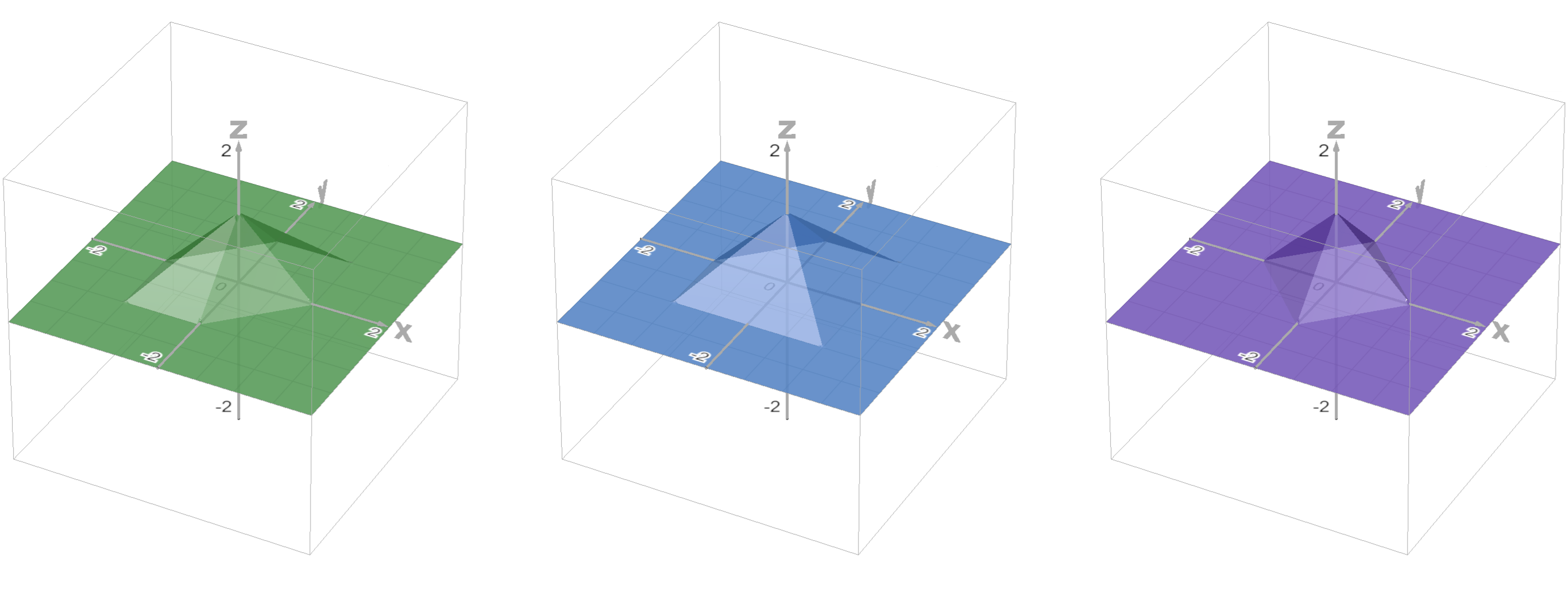}
    \caption{
    Different 2D ``hat functions'' in the literature. These hat functions, or sometimes called ``spike functions'', are used in the partition of unity of $\mathbb{R}^d$. \textbf{Left:} the 2D hat function in \cite{KARNIK2024115557,pan2024metricentropylimitsnonlineardynamical,riegler2024generatingrectifiablemeasuresneural}. \textbf{Middle:} the 2D hat function in \cite{yarotsky_OptimalApproximation_VeryDeep_PMLP_COLT_2018}. \textbf{Right:} our 2D hat function from \textit{this paper}. Unlike the hat functions in this literature, the shape of our hat function varies across different locations; see Figure \ref{fig:hat_2d} for further illustrations and Equation \ref{eq:hat} for their formulas.
    }
    \label{fig:different_hats}
\end{figure}

Our construction of universal MLPs, with one-dimensional cartoons illustrated by~Figure~\ref{fig:Proof_Strategy_Idea}, relies on a new proof technique based on the \textit{Kuhn triangulation} of the hypercube $[0,1]^d$, see~\cite{KuhnTriangulationPaper_1960}, and not on the trifling regions perfected by~\cite{ShenYangZhang_JMPA_OptApprx_ReLU}.  
This new geometric construction allows us to construct approximating ReLU MLPs of maximal regularity.  Surprisingly, Theorem~\ref{thm:KuhnUnique} shows that there is no other triangulation (up to reflections) which can be used to construct an approximating ReLU MLP with a minimal Lipschitz constant; thus, our new construction is essentially unique.

\begin{remark}[Approximation of Smooth Functions]
\label{rem:On__fig:Proof_Strategy_Idea}
The approximate piecewise polynomial variant of this construction sketched in Figure~\ref{fig:Proof_Strategy_Idea__OLD}, which applies for smooth enough functions and which was developed by~\cite{yarotsky_OptimalApproximation_VeryDeep_PMLP_COLT_2018} and subsequently refined by~\cite{yarotsky2020phase,petersen2018optimal,lu2021deep}, uses approximate implementations of piecewise polynomials instead of piecewise constant functions ``outside trifling regions''.   Nevertheless, these methods still ``glue together'' several of these (approximate) local polynomial approximators using steep interpolators on these little ``trifling regions''.  Thus, they, too, still can have little regions where the resulting ReLU MLPs have very steep linear segments.
\end{remark}

\begin{remark}[Approximation in Lipschitz Norm]
\label{rem:LipImpossible}
Theorem~\ref{thm:main_IntroVersion} guarantees that any bounded $L$-Lipschitz function can be uniformly approximated by bounded ReLU MLPs whose Lipschitz constant \textit{never exceeds} $L$. Thus, the approximation happens within the compact (thus separable), by Arzel\`a-Ascoli theorem, set of uniformly bounded $L$-Lipschitz functions in the Banach space of continuous functions on the $d$-dimensional cube $C([0,1]^d)$. 
Our result \textit{does not} claim that one can 
approximate any Lipschitz function $f:[0,1]^d\to \mathbb{R}$ with ReLU MLPs with respect to the Lipschitz norm $\|f\|_{Lip} \eqdef \max_{\mathbf{x} \in [0,1]^d} |f(\mathbf{x})| + \operatorname{Lip}(f)$; which is impossible as the space of the Banach space of Lipschitz functions on $[0,1]^d$ with this norm is not separable whereas the set of ReLU MLPs is.
\end{remark}

\subsection{Further Results}
\label{s:Intro__ss:FurtherResults}

Our analysis yields additional results concerning the expressivity of regular ReLU MLPs.
\paragraph{Memorization/Interpolation with ReLU MLPs of Minimal Lipschitz Constant}
The memorization (i.e.\ interpolation) problem for the MLP model dates back, at least, to~\citep[Theorem 5.1]{Pinkus_MLPApproximation_ActaNumerica_1999} where the author showed, amongst other things, that real-valued shallow MLPs with non-polynomial activation function and width $n$ can memorize/interpolate $n+1$ points in a domain of $\mathbb{R}^d$. 
The following is a shallow version of our main \textit{universal interpolation} result (Theorem~\ref{prop:UnivMemorization_1D}), which guarantees interpolation of $N$ data points with minimal parameter usage and minimal Lipschitz constant is possible with ReLU MLPs depending on $\mathcal{O}(N)$ trainable parameters, constant depth, and $\mathcal{O}(\sqrt{N})$ width.
  
\begin{theorem}[Optimal Interpolation with ReLU MLPs of Maximal Regularity (Shallow Version)]
\label{thrm:min_size_reg_interpol}
Let $((x_n,y_n))_{n=1}^N$ be distinct pairs of training datapoints in $\mathbb{R}\times \mathbb{R}$ (with $x_n<x_{n+1}$ for $n=1,\dots,N-1$).  There exists a ReLU MLP $\Phi:\mathbb{R}\to \mathbb{R}$ of width at most $2 \lceil \sqrt{N} \rceil$, depth $2$, and with at most $2N+8\lceil \sqrt{N} \rceil$ nonzero parameters such that
\[
    \Phi(x_n) = y_n
    \qquad
    \mbox{ for }n=1,\dots,N
\]
Furthermore, $\Phi$ is linear on the intervals $[x_n,x_{n+1}]$ for $n=1,\dots,N-1$, and constant on $(-\infty,x_1]$ and $[x_N,\infty)$.
\end{theorem}
This (shallow) version of our main memorization/interpolation theorem matches the optimal (in the sense of VC-dimension) parameter usage of a ReLU interpolator, as in~\cite{vardi2021optimal}.  The construction of~\cite{vardi2021optimal}, however, can easily be seen to have a large Lipschitz constant, whereas the Lipschitz constant of our interpolating ReLU MLPs cannot be improved as they are exactly the best piecewise linear interpolator of the training data.  
It is worth noting that, in~\citep[Theorem 9.6]{petersen2024mathematical}, the authors recently constructed interpolating ReLU MLPs, which also obtain the optimal Lipschitz constant (with respect to the $\ell^1$ norm) in the multi-dimensional case.  Those networks, however, require depth $\mathcal{O}(\log(N))$ and width $\mathcal{O}(N)$, making them sub-optimal in terms of parameter usage.  Instead, the memorizers of Theorem~\ref{thrm:min_size_reg_interpol} are optimal for both literature streams as they both have minimal parameter usage and minimal Lipschitz constant. 

Our main one-dimensional \textit{universal interpolation theorem} (Theorem~\ref{prop:UnivMemorization_1D}) is a deep version of Theorem~\ref{thrm:min_size_reg_interpol}, where the user can \textit{exactly specify} the depth $L\ge 4$ as well as the layer widths of the interpolating ReLU MLP, subject to the restriction that no MLP layer (bottleneck) is less than $12$ and the total layer widths are at-least $\mathcal{O}(\sqrt{N}/L)$.  What is most interesting about that result is that, unlike the memorizers constructed in~\cite{vershynin2020memory}, which require very large information bottlenecks if the training data points are close together, both Theorem~\ref{thrm:min_size_reg_interpol} and its deep generalization in Theorem~\ref{prop:UnivMemorization_1D}, show that no such restriction is needed.  Moreover, the result shows that for most depth and width specifications, one can interpolate the training data.  

\paragraph{Optimal Global Lipschitz Constant}
The full version of Theorem~\ref{thm:main_IntroVersion}, namely Theorem~\ref{thm:MainTheorem_StableUniversalApproximation}, allows for functions of arbitrary regularity.  Furthermore, Corollary~\ref{cor:main_StableUniversalApproximation___with_OptimalLipschitzExtrapolation} guarantees that $\Phi$ can be chosen to be globally $L$-Lipschitz, not only $L$-Lipschitz on $[0,1]^d$, by increasing its depth by $1$.  A consequence of this is Corollary~\ref{cor:Extrapolation}, which shows that for any $L$-Lipschitz function there is an at-most $L$-Lipschitz ReLU MLP of comparable depth and width to that of Theorem~\ref{thm:main_IntroVersion}, such that
\[
    \sup_{\mathbf{x}\in \mathbb{R}^d}
    \,
        \underbrace{
            |f(\mathbf{x})-\Phi(\mathbf{x})|
        }_{\text{Approximation Error in }[0,1]^d}
        -
        \underbrace{
            L\max_{i=1,\dots,n} [(x_i-1)_+ + (-x_i)_+]
        }_{\text{Extrapolation Error beyond }[0,1]^d}
    \lesssim 
        \frac{1}{n}
\]
where for each $u\in\mathbb{R}$, $u_+\eqdef \max\{0,u\}$ and $\lesssim$ hides a dimensional constant of the order of $\mathcal{O}(Ld)$; here $\max_{i=1,\dots,n} [(x_i-1)_+ + (-x_i)_+]$ is simply the $\ell^{\infty}$ distance to from a point $\mathbf{x}\in \mathbb{R}^d$ to the hypercube $[0,1]^d$.
This is an additive formulation of the multiplicative global approximation result of~\cite{cuchiero2023global}.

\paragraph{Best Achievable Approximation of Discontinuous Target Function}
Our main version (Theorem~\ref{thm:MainTheorem_StableUniversalApproximation}) of Theorem~\ref{thm:RademacherBound__IntroVersion} even applies to ``regular'' but discontinuous target functions.  In these cases, the approximation error does not converge to $0$ but, rather, to the minimal achievable approximation error (details in Section~\ref{s:Prelim__ss:Regularity___sss:Disc}).  This latter result is rather interesting since it allows us to quantify the best \textit{uniform} approximation of discontinuous target functions instead of having to rely on an $L^p$, for $1\le p<\infty$, relaxation of the notion of approximability or via randomization as in~\cite{kratsios2022universal}.  However, as one would expect from the Uniform Limit Theorem; see e.g.~\cite[Theorem 21.6]{MunkresTopBook}, this approximation error \textit{does not converge to $0$} if the target function is genuinely discontinuous; however, we can quantify how small it can be made.

\paragraph{{Additional Structure of the weights and biases in the class $\mathcal{NN}_{\Delta,W}^{L,n}$}}
Proposition~\ref{proposition:ParamterEstimates__RefinedVersion} provides details on the weights and biases, as well as the encoding scheme, used to construct the universal sub-class of $\mathcal{NN}_{\Delta,W}^{L,n}$ described in Theorem~\ref{thm:main_IntroVersion}. 

\subsection{Organization of Paper}
\label{s:Intro__ss:Organization}
Our paper is organized as follows.  
Section~\ref{s:Related} overviews related results in the approximation theory literature, focusing on our main result (Theorem~\ref{thm:RademacherBound__IntroVersion}).
Section~\ref{s:Prelim} contains the preliminary notation, terminology, and background required for the formulation of our main results and their proofs.
Section~\ref{s:Main} contains our main results, these are the full version of our ``regular'' approximation theorem (Theorem~\ref{thm:main_IntroVersion}) as well as the version of our main generalization bound (Theorem~\ref{thm:RademacherBound}) with explicit constants.

Section~\ref{s:Consequences} discusses consequences of our main result, such as extrapolation rates (Corollary~\ref{cor:Extrapolation}), global Lipschitz regularity (Corollary~\ref{cor:main_StableUniversalApproximation___with_OptimalLipschitzExtrapolation}).  A discussion on how our result fit into the modern theory landscape is given in Section~\ref{s:Discussion}; where we also discuss the uniqueness (up to symmetries) of the Kuhn triangulation we used to construct our optimal ReLU MLPs of maximal regularity.  This later result shows that our new geometric argument is essentially unique and cannot be improved.

All proofs are relegated to Section~\ref{s:Proofs}.  Funding and acknowledgments are discussed in Theorem~\ref{s:Acknowledgments}.

\section{Related Literature}
\label{s:Related}

\subsection{Lipschitzness in Neural Networks}
\label{s:Related__ss:LipschitzNeuralNetworks}

Neural networks of a prescribed Lipschitz regularity are common in various areas of deep learning; with applications ranging from generative adversarial learning~\cite{arjovsky2017wasserstein,cao2019multi,korotinneural}, conditional distribution estimation~\cite{benezet2024learning}, to certifiable deep learning~\cite{fazlyab2024certified}, amongst many applications.  This has led to several optimization pipelines, e.g.~\cite{neyshabur2015pathSGD_NIPS,tsuzuku2018lipschitz,qu2023data,gouk2021regularisation,pauli2021training}, and architectural designs, e.g.~\cite{li2019orthogonal,lezcano2019cheap,araujo2023a}, enforcing Lipschitzness of trained neural networks.  Various computational tools, e.g.~\cite{jordan2020exactly,bhowmick2021lipbab,xue2022chordal}, have also been developed to efficiently estimate the Lipschitz constant of neural networks.  Despite extensive work, there is no guarantee that there will be no loss of model expressivity when imposing Lipschitz constraints; our main result fills this gap, thus adding additional theoretical foundations to deep learning areas relying on Lipschitz neural networks.  We mention that the Lipschitz constant of (untrained) randomly initialized neural networks has recently been studied in~\cite{geuchen2023upper}.

\subsection{Approximators with Parameter Restrictions}
\label{s:Related__ss:Related__WeightBounds}
To the best of the authors' knowledge, the tightest available approximation results in the literature, which provide weight and bias size limits, are given in~\cite [Theorem 3.1]{petersen2018optimal}.  There, the authors consider an $L^p$-type approximation, for $0<p<\infty$, of any $\alpha$-H\"{o}lder function on $[0,1]^d$, with $\alpha$-H\"{o}lder coefficient at-most $1$ is shown to be possible using a ReLU MLP depending on $\mathcal{O}(\varepsilon^{-\alpha/d})$ parameters, organized into at-most $(2 + \lceil \log_2(\alpha) \rceil ) (11+\alpha/d)$ layers, and whose weights all belong to $[-\varepsilon^{-s},\varepsilon^{-s}]$; where $s\ge 1$ is an integer depending at-least on the H\"{o}lder exponent ($\alpha$) of the target function and on the dimension.  For simplicity, examining the proofs of \citep[Lemmata A.3]{petersen2018optimal} one sees that $s\ge 7$.  Now,~\citep[Theorem 1]{park2021minimum} guarantees that the width of any universal (in the $L^p$ sense) class of ReLU MLPs must have a width at least $d+1$.  Therefore, these observations, together with the fact that the operator norm $\|\mathbf{A}\|_{op:2\to 2}$ of any $d\times d$ matrix by its componentwise $2$-norm bound (and the elementary bound on $\|\cdot\|_2\le \|\cdot\|_1$): $\|\mathbf{A}\|_{op:2\to 2}\le d \, \max_{i,j=1,\dots,d}\,|\mathbf{A}_{i,j}|$ implies that the upper-bound on the Lipschitz constant $\operatorname{Lip}(\hat{f})$ of ReLU MLP $L^p$ ($p<\infty$) approximator $\hat{f}$ which can be deduced from~\citep[Theorem 3.1]{petersen2018optimal} must be at least
\begin{equation}
\label{eq:LipschitzConstantWeightBounds}
    \frac{
        d^{(2 + \lceil \log_2(\alpha) \rceil ) (11+\alpha/d)}
    }{
        \varepsilon^{7\, 
            (2 + \lceil \log_2(\alpha) \rceil ) (11+\alpha/d)
        }
    }
\ge 
    \frac{
        d^{33}
    }{
        \varepsilon^{231 + 21\,\alpha/d}
    }
.
\end{equation}
Comparing the Lipschitz constant of a neural network whose weights achieve the lower-bound in~\eqref{eq:LipschitzConstantWeightBounds}, to the Lipschitz constant of our main result (Theorem~\ref{thm:main_IntroVersion}) shows that the latter significantly improve the guarantees in the literature; as~\eqref{eq:LipschitzConstantWeightBounds} diverges as the approximation error tends to $0$ while our Lipschitz constant remain at the optimum.  

Approximation of Lipschitz functions with values in $[-1,1]$ by MLPs with the $\operatorname{ReQU}$ activation function and weights in $[-1,1]$ is possible, however, there are no such guarantees for MLPs with the most standard $\operatorname{ReLU}$ activation function.  We mention the work of~\cite{belomestny2023simultaneous} which controls the statistical properties of $\operatorname{ReLU}$ MLPs with bounded weights and biases using their best approximation (assuming it is exogenously bounded) and a correction term.

\subsection{Approximation of Functions and Their Derivatives}
\label{s:Related__ss:Deriv}
The approximation of a function and its derivatives has drawn significant attention in the deep learning for partial differential equations (PDEs) literature~\cite{MR4376568,guhring2021approximation,MR4376559,MR3736669,MR3847747,lei2022solving,MR3847747,MR4514159,MR4539210,MR4555161}.
Guarantees that networks can approximate a function while also approximating its (at least first) derivative date back to~\cite{hornik1990universal_derivatives} with more recent quantitative guarantees being given by \cite{MR4131039,guhring2021approximation,belomestny2023simultaneous,neufeld2023universal} (not all of which use the $\operatorname{ReLU}$ activation function).  Recall that, by the mean value theorem, every once continuously differentiable function is Lipschitz on $[0,1]^d$ with its Lipschitz constant given by the maximum norm of its gradient thereon; and, as a partial converse, every Lipschitz function on $\mathbb{R}^d$ is differentiable almost-everywhere on $[0,1]^d$; see~\cite[Theorem 3.1.6]{Federer_GeometricMeasureTheory_Book_1968}.  
Thus, for the subclass of Lipschitz functions which are once continuously differentiable, these results guarantee that MLPs can approximate these maps on $[0,1]^p$ while also approximating their Lipschitz constant.  

However, there is no guarantee that MLPs can approximate these functions while exactly implementing their Lipschitz constant even locally.  More generally, there is no result that MLPs can approximate functions of lower regularity (e.g.\ H\"{o}lder of sub-H\"{o}lder functions) while also exactly matching their Lipschitz constant \textit{globally}. 

\subsection{Global Universal Approximation}
\label{s:Related__ss:GLobalApproximation}
Theorem~\ref{thm:MainTheorem_StableUniversalApproximation} provides exactly this guarantee showing, in particular, that $\operatorname{ReLU}$ MLPs can \textit{globally} implement concave moduli of continuity of any uniformly continuous functions (e.g.\ H\"{o}lder or Lipschitz functions) while locally approximating them in $[0,1]^d$.  A global guarantee of the extrapolation rate for an approximation on $[0,1]^d$ will also be provided in Corollary~\ref{cor:Extrapolation}, and this is possible due to our guarantee that we may exactly and globally match the modulus of continuity of the target function being approximated.

We note that there do exist \textit{qualitative} global approximation theorems in the deep learning~\cite{cuchiero2023global}, reservoir computing~\cite{LuydmillaJP_UAT_2018}, and Stone-Weirestrass type~\cite{giles1971generalization} approximation literature.  However, although each of those results is qualitative, it is currently not known what the extrapolation rates are for deep learning models, which are only guaranteed to provide a uniform approximation on $[0,1]^d$.


\subsection{Deficits of Parametric Generalization Bounds for Overparameterized MLPs}
Under the assumptions of Theorem~\ref{thm:RademacherBound__IntroVersion}, the Rademacher complexity bounds of~\citep[Theorems 8 and 12]{bartlett2002rademacher} and~\cite{bartlett2017spectrally}, together with some result on random matrices with independent rows in~\citep[Theorem 4.6.1]{Vershynin_HighDimensionalProbBook_2018} (see Appendix~\ref{s:learning_theory_proofs} for details) can only imply the following much weaker bound: for each $0<\delta<1$ the following holds
\begin{equation}
\label{eq:RademacherBound}
    \sup_{\Phi\in \mathcal{NN}_{\Delta,W}^{L,n}} \big|\mathcal{R}(\Phi)-\mathcal{R}_N(\Phi)\big|
\in 
    \mathcal{O}\Big(
        \frac{\sqrt{\log(4/\delta)}}{\sqrt{N}}
    \Big)
    +
    \tilde{\mathcal{O}}\biggl(
            \frac{1}{N^{3/2}}
        +
            \frac{
                W^{3\Delta/2}
                \sqrt{\log(4/\delta)}
            }{\sqrt{N}}
        \biggr)
\end{equation}
with probability at least $1-\delta$.
We also note that, the nearly optimal VC-bounds derived in~\cite{BartlettHarveyChristopherMehrabian_VCDimReLUMLPs_JMLR_2019} for ReLU MLPs also diverge when $W$ and $\Delta$ tend to infinity.

\section{Preliminaries}
\label{s:Prelim}
This section contains the preliminaries needed to formulate our results.  These include both the notation used in this manuscript, as well as the background terminology relating to multilayer perceptrons.

\subsection{Notation}
\label{s:Prelim__ss:Notation}
We use $\mathbb{R},\mathbb{N},\mathbb{N}_+,\mathbb{Z}$ to denote the set of real numbers, non-negative integers, positive integers, and integers, respectively.  For any $r>0$ we write $r\mathbb{Z}^d\eqdef \{x\in \mathbb{R}^d:\, \exists z\in \mathbb{Z}^d\, x=rz\}$.

For any $x\in\mathbb{R}$, let $\lfloor x \rfloor \eqdef \max\{n\in\mathbb{Z} : n \le x \}$ denote the floor of $x$, and $\lceil x \rceil \eqdef \min\{n\in\mathbb{Z} : n \ge x \}$ denote the ceiling of $x$. For any $n\in\mathbb{N}_+$, let $[n]$ denote the set of integers $\{1,2,\cdots,n\}$.

We denote the rectified linear unit (ReLU) activation function by $\sigma:\mathbb{R}\ni t \mapsto \max\{0,t\}\in \mathbb{R}$.  It will always be applied componentwise to any vector, by which we mean $\sigma(\mathbf{x})\eqdef (\sigma(x_i))_{i=1}^d$ for each $\mathbf{x}=(x_1,\cdots,x_d)\in\mathbb{R}^d$ and every $d\in \mathbb{N}_+$.  

Vectors in dimensions above $2$ are denoted by bold lowercase letters. Matrices are denoted by bold uppercase letters.

\subsection{Uniform Regularity}
\label{s:Prelim__ss:Regularity}
We first overview the standard notions of uniform continuity before introducing their discontinuous generalizations for which our approximation results hold.
\subsubsection{Uniform, Lipschitz, and H\"{o}lder Continuity}
\label{s:Prelim__ss:Regularity___sss:UniformContinuity}
A map $\omega:[0,\infty)\to [0,\infty)$ is called a modulus of continuity if $\omega$ is monotonically increasing and continuous (from the right) at $0$ with $\omega(0)=0$.  
Let $(X,\rho_X)$ and $(Y,\rho_Y)$ be metric spaces and $f:X\to Y$.  We say that $\omega$ is a modulus of continuity for $f$ if: for each $x,\tilde{x}\in X$
\[
    \rho_Y(f(x),f(\tilde{x}))\le \omega(\rho_X(x,\tilde{x}))
.
\]

We often consider uniformly continuous functions which are either Lipschitz and H\"{o}lder continuous.
If $L\ge 0$, $0<\alpha \le 1$, and $\omega(x)=Lx^{\alpha}$, then we will say that $f$ is $\alpha$-H\"{o}lder with constant $L$. 
The class of all such $f:X \supseteq\Omega \rightarrow \mathbb{R}$ is denoted by $\operatorname{Hol}(L,\alpha,\Omega)$.
If $\alpha=1$, we say that $f$ is $L$-Lipschitz.  Given any real-valued function $f$ on a subset $X$ of $\mathbb{R}^d$, we use $\operatorname{Lip}(f)$ to denote its optimal Lipschitz constant with respect to the restriction of the $\ell^1$-norm to $X$; i.e.
\[
    \operatorname{Lip}(f) \eqdef 
    \sup\limits_{
        \substack{ \mathbf{x},\mathbf{y} \in X \\
                   \mathbf{x} \ne \mathbf{y}
        }
    }
    \frac{|f(\mathbf{x}) - f(\mathbf{y})|}{\|\mathbf{x} - \mathbf{y}\|_{1}}
.
\]
We are mostly interested in the case where $X$ is a subset of the space $(\mathbb{R}^d,\ell_1)$ and $(Y,\rho_Y)$ is a Euclidean line; where, for $1\le p\le \infty$, $\ell_p$ indicates that we are equipping $\mathbb{R}^d$ with the metric induced by the $\ell^p$ norm $\|\mathbf{x}\|_p^p\eqdef \sum_{i=1}^p\,|x_i|^p$ if $p$ is finite and $\|\mathbf{x}\|_{\infty}\eqdef \max_{i=1,\dots,d}\, |x_i|$ otherwise.  

\subsubsection{Uniform Regularity, Best Achievable Approximation Error, and Discontinuity}
\label{s:Prelim__ss:Regularity___sss:Disc}

We consider the following broad class of (possibly discontinuous) target functions, including uniformly continuous functions. 
The reader which is only interested in approximation guarantees for continuous functions, \textit{where an asymptotic error of $0$ is achievable by ReLU MLPs}, is encouraged to skip this section and swap the term ``modulus of regularity'' for ``concave modulus of continuity'' in the remainder of the paper.
We rely on the following weakened version of a concave modulus of continuity.  


\begin{definition}[Modulus of Regularity]
\label{defn:modulus_regularity}
Let $I=[0,\infty)$ or $I=[0,T]$ for some $T\ge 0$.
A function $\omega:I\to [0,\infty)$ is called a modulus of regularity if: 
\begin{enumerate}
    \item $\omega(0)=0$,
    \item $\omega$ is monotone increasing,
    \item $\omega$ is concave. 
\end{enumerate}  
\end{definition}

Moduli of regularity allow us to quantify the best achievable approximation error, when uniformly approximating discontinuous target functions.  
Before formalizing this, we recall that concave functions on bounded domains $\Omega\subset \mathbb{R}^d$ are continuous on their interior but need not be on the boundary $\partial \Omega$ of $\Omega$.  For example, the indicator function $I_{(0,\infty)}$ of $(0,\infty)$ is concave on $\Omega = [0,\infty)$, continuous on $(0,\infty)$, and has a discontinuity on $\partial \Omega =\{0\}$.
We will often consider the \textit{best achievable approximation error} $\omega^{\star}\in [0,\infty)$, when minimizing such a modulus, which is given by
\[
    \omega^{\star} \eqdef \lim\limits_{t\downarrow 0}\, \omega(t) 
.
\]
\begin{example}[Concave Moduli of Continuity are Moduli of Regularity]
\label{ex:mod_continuity}
If $\omega:[0,\infty)\to[0,\infty)$ is a concave modulus of continuity, then $\omega$ is a modulus of regularity, with $\omega(0)=0$, and $\omega$ is continuous at $0$; i.e.\ $\lim\limits_{t\downarrow 0}\,\omega(t)=0$ and in particular $\omega^{\star}=0$.
\end{example}
The previous example illustrates the main difference between moduli of continuity and moduli of regularity; namely, the value of $0$ need not be achievable by minimizing $\omega$ from the right.  Indeed, if $\omega$ is a modulus of regularity, then it is non-negative and it fixes $0$; thus, $\omega(t) \ge \omega(0)=0$ for all positive values of $t$.  Consequentially, 
\begin{equation}
\label{eq:best_achievable_apprximation_error__vs_0_error}
        \omega(0)
    =
        0
    \le 
        \omega^{\star} 
\end{equation}
for any modulus of regularity.  This inequality can be strict for discontinuous functions.
\begin{example}[Beyond Moduli of Continuity]
\label{ex:modulus_of_heavyside_function}
Fix $M>0$ and $L\ge 0$.  
The function $\omega:[0,\infty)\to [0,\infty)$ given for each $t\in [0,\infty)$ by
\[
    \omega(t)
    \eqdef 
    \begin{cases}
        M+ L\,t & \mbox{ if }  t>0 \\
        0 & \mbox{ if } t=0
    \end{cases}
\]
is a modulus of regularity but not a modulus of continuity, and the inequality in~\eqref{eq:best_achievable_apprximation_error__vs_0_error} is strict.
\end{example}
The class of continuous functions between metric spaces are precisely those functions for which: for every compact subset of their domain, there is a modulus of continuity bounding the distance between the images of all pairs of points therein, see e.g.~\citep[Theorem 27.6]{MunkresTopBook}, and such that \textit{the images of arbitrarily close points are themselves arbitrarily close}.   

If one relaxes the second condition, we are left with the following functions considered in our paper.  Our interest in this class stems from~\eqref{eq:best_achievable_apprximation_error__vs_0_error}, wherein our main result (Theorem~\ref{thm:MainTheorem_StableUniversalApproximation}) shows that for $\omega$-regular function can be approximated by a ReLU MLP with at-most $\omega$-regularity, up to the best achievable approximation error $\omega^{\star}$.  In the very special case where the target function is additionally continuous, we can guarantee the usual type of conclusion: our ReLU MLP approximators can achieve \textit{zero approximation error} asymptotically. 

\begin{figure}[H]
    \centering
\begin{align*}
\vcenter{\hbox{\includegraphics[width=.25\textwidth]{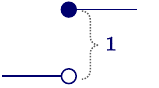}}}
&\qquad\qquad
\begin{aligned}
    h(x) 
\eqdef 
    \begin{cases}
        1, \quad x\ge 0  \\
        0, \quad x < 0
    \end{cases}
\end{aligned}\\
\vcenter{\hbox{\begin{minipage}{4cm}
\end{minipage}}}
& \notag
\end{align*}
\caption{The heaviside function $h$ is $\omega$-regular with $\omega(t)=1$ if $t>0$ and $0$ otherwise.}
\label{fig:mod_disc}
\end{figure}

\begin{definition}[$\omega$-Regular Functions]
\label{defn:omega_regularity}
Let $\omega$ be a modulus of regularity, $(X,\rho)$ and $(Y,\rho')$ be metric spaces.  A map $f:X\to Y$ is said to have $\omega$-regularity at a pair of points $x,x'\in X$ if
\[
        \rho'\big(
            f(x)
            ,
            f(x')
        \big)
    \le 
        \omega\big(
            \rho(x,x')
        \big)
.
\]
If, moreover, $f$ has $\omega$-regularity at all pairs of points $x,x'\in X$ then, we way that $f$ is $\omega$-regular (on $X$).  The set of all $\omega$-regular functions from $(X,\rho)$ to $(\mathbb{R},|\cdot|)$ is denoted by $\operatorname{Lip}_{\omega}(X)$.
\end{definition}
We will say that a function is \textit{regular} if it is $\omega$-regular for some modulus of regularity $\omega$.  
Geometric examples of discontinuous regular functions include quasi-isometric~\cite{LangPlaut_2001_DedicGeomBilipSpaceformEmbeddings} and certain coarse embeddings~\cite{NowakCoarseEmbeddins_2005,MendelNaor_CoarseSuperExpanders_2014}.  Discontinuous regular maps can also be constricted as additive perturbations of Lipschitz functions by some ``deterministic bounded noise'', as follows.

\begin{example}[Additive Perturbations of Lipschitz Functions]
\label{ex:constructions_regular_but_discontinuous}
Fix $L,M\ge 0$ and fix any function $\eta:\mathbb{R}\to [-M/2,M/2]$.  For any $L$-Lipschitz function $f:\mathbb{R}^d\to \mathbb{R}$ define the perturbed function $f_{\eta}:\mathbb{R}^d\to\mathbb{R}$ as sending any $\mathbf{x}\in \mathbb{R}^d$ to
\[
    f_{\eta}(\mathbf{x})\eqdef f(\mathbf{x}) + \eta(\mathbf{x})
.
\]
Then, $f_{\eta}$ is $\omega$-regular with modulus of regularity as in Example~\ref{ex:modulus_of_heavyside_function}. 
\end{example}
Regular functions need not be well-behaved measure-theoretically.  For instance, in the context of Example~\ref{ex:constructions_regular_but_discontinuous} if $A$ is the Vitali set in $\mathbb{R}$ and $\eta=\frac{M}{2}\,I_{A}$ then we see that $f_{\eta}$ need not even be Lebesgue measurable even if $f$ was Lipschitz.  

\begin{example}[Not all Functions are Regular]
Not all functions are regular.  The following map is not regular: $f:[0,1]\to [0,\infty)$ given for each $x\in [0,1]$ by
\[
    f(x)
    \eqdef
    \begin{cases}
        \frac{1}{x} & \mbox{ if } x>0\\
        0 & \mbox{ if } x=0
.
    \end{cases}
\]
\end{example}
However, any bounded function is regular.  Nevertheless, their best achievable approximation error can be large.

\subsection{Multilayer Perceptrons (MLPs) with ReLU Activation Function}
\label{s:Prelim__ss:Notation___sss:MLP}
We now define multilayer perceptions.  We then formulate notions of parameteric and functional regularity often encountered in the literature considered herein.  
\begin{definition}[Multilayer Perceptrons with ReLU Activation Function (ReLU MLPs)]
\label{defn:ReLUMLP}
Let $\Delta \in \mathbb{N}_+$ and consider a multi-index $\mathbf{d}\eqdef [d_1,\dots,d_{\Delta+1}]\in \mathbb{N}_+^d$.  
The class $\mathcal{NN}(\mathbf{d})$ consists of all multilayer perceptrons with ReLU activation function (ReLU MLPs) $\Phi:\mathbb{R}^{d_1}\to\mathbb{R}^{d_{\Delta+1}}$ admitting the following iterative representation
\begin{equation}
\label{eq:MLPRepresentation}
    \begin{aligned}
        \Phi(\mathbf{x}) & = \mathbf{W}^{(\Delta)} \mathbf{x}^{(\Delta)}+\mathbf{b}^{(\Delta)}\\
     \mathbf{x}^{(l+1)}& \eqdef \operatorname{ReLU} \bullet\big(\mathbf{W}^{(l)}\,\mathbf{x}^{(l)} + \mathbf{b}^{(l)}\big) 
        \qquad
        \mbox{ for } l=1,\dots,\Delta-1
    \\
    \mathbf{x}^{(1)} & \eqdef  \mathbf{x}.
    \end{aligned}
\end{equation}
where for $l=1,\dots,\Delta$, $\mathbf{W}^{(l)}$ is a $d_{l+1}\times d_l$-matrix and $\mathbf{b}^{(l)}\in \mathbb{R}^{d_{l+1}}$, and $\operatorname{ReLU}\bullet$ denotes componentwise application of the $\operatorname{ReLU}$ function.  We denote $\operatorname{widthvec}(f)\eqdef[d_2,\dots,d_{\Delta}]$.

Given $L,W,d,D\in \mathbb{N}_+$, we use $\mathcal{NN}(d,D)$ to denote the class of maps $f:\mathbb{R}^d\to \mathbb{R}^D$, $\mathcal{NN}_{\Delta,W}(d,D)$ the subset of $\mathcal{NN}(d,D)$ of ReLU MLPs with depth at-most $\Delta$ and width at-most $W$, and $\mathcal{NN}_{\Delta,W}^L(d,D)$ the subset of $\mathcal{NN}_{\Delta,W}(d,D)$ of $L$-Lipschitz ReLU MLPs therein.  When clear from the context, we suppress the notational dependence on $d$ and $D$.
\end{definition}

Given a ReLU MLP $f$ with representation~\ref{eq:MLPRepresentation}, we call the integers $\Delta-1$ and $\max_{l=2,\dots,\Delta}\,d_l$ its depth and width, respectively.
Let $P(\mathbf{d})  \eqdef  \sum_{l = 1}^{\Delta} d_l (d_{l+1} + 1)$. When discussing the stability of the $\Phi$ on the \textit{parameters} defining it via this representation, we will rely on the following vectorization of the parameters defining its weight matrices (the $\mathbf{W}^{(l)}$) and its biases (the $\mathbf{b}^{(l)}$)
\begin{equation}
\label{eq:identification}
    \mathbb{R}^{P(\mathbf{d})}\,\reflectbox{$\in$}\,\theta \iff (\mathbf{W}^{(l)}, \mathbf{b}^{(l)})_{l=1}^{\Delta}
    \in \prod_{l=1}^{\Delta}\, \big(\mathbb{R}^{d_{l+1}\times d_l}\times\mathbb{R}^{d_{l+1}} \big).
\end{equation}
In particular, this vectorization of the parameters of $\Phi$ allows us to define the maximum norm of a set of weights and biases defining $\Phi$, in representation~\eqref{eq:MLPRepresentation}, as the $\ell^{\infty}$ norm of its parameter vector $\theta$ in~\eqref{eq:identification}; via
\begin{equation}
\label{eq:parameter_norm}
    |\Phi|_{\operatorname{par}}
    \eqdef 
    \max_{i=1,\dots,P(\mathbf{d})}\, |\theta_i|
    ,
\end{equation}
where the operation $|\Phi|_{\operatorname{par}}$ is defined \textit{given} the representation of $\Phi$ in~\eqref{eq:MLPRepresentation}.  
That is, different \textit{representations} of the same function $\Phi$ may have different maximum parameter sizes, meaning that, $|\Phi|_{\operatorname{par}}$ is not a canonical intrinsic quantification of the regularity of $\Phi$.  
However, this is \textit{not} the case for the optimal Lipschitz constant of $\Phi$ which is independent of any parameterization thereof. Thus, the optimal Lipschitz constant of $\Phi$ is a more natural measure of regularity of $\Phi$ than norms on its parameter vectors; given a particular representation.

\section{Main Results}
\label{s:Main}
We now present our main approximation result (Theorem~\ref{thm:main_IntroVersion}) and then the full-version of our statistical guarantee (Theorem~\ref{thm:RademacherBound}).

\subsection{Approximation Guarantees}
\label{s:Main__ss:Approximation}

The following is the full version of our optimal approximation theorem with maximal regularity.
\begin{theorem}[Optimal Regular Approximation by Sample-Interpolating ReLU MLPs]
\label{thm:MainTheorem_StableUniversalApproximation}
    Let $f$ be a function from $[0,1]^d$ to $\mathbb{R}$, and let $\omega:[0,d]\rightarrow [0,\infty)$ be a modulus of regularity of $f$. Then, for any $n\in\mathbb{N}_+$, there exists a ReLU MLP $\Phi$ on $[0,1]^d$ with width at most $8d(n+1)^d$, depth at most $\lceil \log_2{d} \rceil + 4$, at most $16d(n+1)^d$ nonzero parameters, which satisfies the approximation guarantee 
    \[
        \| f-\Phi \|_{L^{\infty}([0,1]^d)} 
        \le 
        \omega\left(\frac{d}{2n}\right)
    \]
    as well as the sample-interpolation guarantee
    \[
        f(\mathbf{x})=\Phi(\mathbf{x}) 
        \,\,
        \mbox{ for all } \mathbf{x}\in
        \left\{ \frac{0}{n}, \frac{1}{n}, \cdots, \frac{n}{n} \right\}^d
    .
    \]
    Furthermore, $\omega$ is a modulus of regularity of $\Phi$ on $[0,1]^d$.
\end{theorem}
\begin{remark}
Theorem~\ref{thm:MainTheorem_StableUniversalApproximation} provided guarantees on the modulus of regularity of $\Phi$ on $[0,1]^d$; this guarantee can be extended to all of $\mathbb{R}^d$ at a cost of only a factor of $\mathcal{O}(1)$ more non-zero parameters (Corollary~\ref{cor:main_StableUniversalApproximation___with_OptimalLipschitzExtrapolation} below).  In this way, one may obtain a global extrapolation variant of this result (Corollary~\ref{cor:Extrapolation} below).
\end{remark}
\begin{remark}
\label{rem:structure_EncoderDecoder}
Only the last \textit{linear layer} of $\Phi$ in Theorem~\ref{thm:MainTheorem_StableUniversalApproximation} depends on $f$, and, all the hidden layers of $\Phi$ are independent of $f$.
\end{remark}

Our next consequence of Theorem~\ref{thm:MainTheorem_StableUniversalApproximation} shows that neural networks can indeed solve McShane's extension problem for finite subsets of dyadic lattices in Euclidean space; the following generalizes Corollary~\ref{cor:McShane__IntroVersion}.
\begin{corollary}[Uniformly Continuous Extension by Neural Networks on Finite Subsets of Dyadic Grids]
\label{cor:McShane}
For any finite $A\subseteq [0,1]^d\bigcap \bigcup_{n \in \mathbb{N}_+}\, \frac1{2^n} \mathbb{Z}^d$, any concave modulus of continuity $\omega$, and any $\omega$-uniformly continuous $f:A\to \mathbb{R}$ there is an \textbf{$\omega$-uniformly continuous} ReLU MLP $\Phi:\mathbb{R}^d\to \mathbb{R}$ extending $f$; i.e.\
\[
f(x)=\Phi(x) 
\mbox{ for all }x\in A
.
\]
\end{corollary}

\begin{table}[htp!]%
    \centering
    \begin{tabular}{@{}lc@{}}
    \toprule
    \textbf{Description of MLP} & \textbf{Estimate}\\
    \arrayrulecolor{lightgray}\midrule 
    \textbf{Parameteric Complexity} & \\
    Depth & $\lceil \log_2(d)\rceil +4$
    \\
    Width & $8d(n+1)^d$
    \\
    Nonzero parameters & $16d(n+1)^d$
    \\
    \arrayrulecolor{lightgray}\midrule 
    \arrayrulecolor{lightgray}\midrule 
    \textbf{Regularity} & 
    \\
    Modulus of Regularity 
    & $\Phi \in \operatorname{Lip}_{\omega}(\{\mathbf{x}_i\}_{i=1}^{N(n)})$
    \\
    Parameteric Regularity $|\Phi|_{\operatorname{par}}$ & $\max\{n,\max_{i\in [N(n)]}\,|f(\mathbf{x}_i)|\}$
    \\
    \arrayrulecolor{lightgray}\midrule 
    \arrayrulecolor{lightgray}\midrule 
    \textbf{Misc.\ Structure} & \\
    Encoding Dimension $\dim(\operatorname{dom}(\mathcal{D}))$ & $(n+1)^d$\\
    Support & $\operatorname{supp}(\Phi) \subseteq [-1/n,1+1/n]^d$
    \\
    \bottomrule
    \end{tabular}
    \caption{\textbf{ReLU MLP ($\Phi$) constructed in Theorem~\ref{thm:MainTheorem_StableUniversalApproximation} approximating $f\in \operatorname{Lip}_{\omega}(\{\mathbf{x}_i\}_{i=1}^{N(n)})$ on $[0,1]^d$:} All maxima are indexed over $i$ in $\{1,\dots,N(n)\}$.
    The depth, width, and maximum parameter size $|\Phi|_{\operatorname{par}}$ are defined in~\eqref{eq:parameter_norm}.
    The regularity of the encoder is with respect to the $\ell^{\infty}$ norm on $\mathbb{R}^{N(n)}$.
    By the support $\operatorname{supp}(\Phi)$ of $\Phi$, we mean the set points in $\mathbb{R}^d$ where it is non-zero.
    }
\label{tab:MLP_Description}
\end{table}
Note that, the path-norm bound in Table~\ref{tab:MLP_Description} implies that most weights defining the ReLU MLP must be small.  Indeed, a detailed inspection of the construction of $\Phi$, see Remark~\ref{remark:ParamterEstimates__RefinedVersion}, shows that most weights and biases defining $\Phi$ are in $\{0,\pm 1/2\}$.

A closer look at the proof of Theorem~\ref{thm:MainTheorem_StableUniversalApproximation} reveals much more information about the structure of the weights and biases of the approximating ReLU MLP constructed therein, as well as their dependence on the target function.  We summarize these additional facts in the following Proposition; which is proven during the course of our main result.   The following result thus provides in-depth details of the structure of the decoder mapping the latent code $\mathcal{E}(f)$ to our constructed ReLU MLP.
\begin{proposition}[{Estimates for the Weights and Biases in the Construction of $\Phi$ in Theorem~\ref{thm:MainTheorem_StableUniversalApproximation}}]
\label{proposition:ParamterEstimates__RefinedVersion}
Consider the setting of Theorem~\ref{thm:MainTheorem_StableUniversalApproximation}, represent $\Phi$ as in~\eqref{eq:MLPRepresentation}; note that $1<\Delta\le \lceil \log_2{d} \rceil + 4$.
\begin{enumerate}
    \item[(i)] \textbf{Target Dependant Layers:} The weights $\mathbf{W}^{(\Delta)}$ depend on $f$,
    \item[(ii)] \textbf{Target Independents Layers:} The weights $\{\mathbf{W}^{(l)}\}_{l=1}^{\Delta-1}$ do not depend on $f$ but only depend on $d$, and the biases $\{\mathbf{b}^{(l)}\}_{l=2}^{\Delta}$ are all $0$,
    \item[(iii)] \textbf{Typical Weights:} For each $l=2,\dots,\Delta-1$, $\mathbf{W}^{(l)}_{i,j} \in \{0,\pm 1/2\}$ for each $i,j$,
    \item[(iv)] \textbf{Initial and Terminal Weights:} $\displaystyle
    \|\mathbf{W}^{(1)}\|_{\ell^{\infty}} \le 1, 
    \|\mathbf{W}^{(\Delta)}\|_{\ell^{\infty}} \le \max_{i\in [N(n)]} |f(\mathbf{x}_i)|, 
    \|\mathbf{b}^{(l)}\|_{\ell^{\infty}} \le n$.
\end{enumerate}
\end{proposition}

We conclude this section by noting that the shape of the optimal ReLU MLPs in Theorem~\ref{thm:MainTheorem_StableUniversalApproximation}
can be very flexible.  As the following variant of that result shows, we may specify ReLU MLPs of virtually any width or depth to obtain our optimal approximation guarantee with maximal regularity. 

\begin{proposition}[Theorem \ref{thm:MainTheorem_StableUniversalApproximation} with Variable Width and Depth]
\label{prop:MainTheorem_AdjustableSize}
    Let $f$ be a function from $[0,1]^d$ to $\mathbb{R}$, and let $\omega:[0,d]\rightarrow [0,\infty)$ be a modulus of regularity of $f$. Then, for any master parameter $n\in \mathbb{N}_+$, depth parameter $L\in\mathbb{N}_+$, and width parameters $m_1,\cdots,m_L\in\mathbb{N}_+$ satisfying
    \[
        m_1 + \cdots + m_L = (n+1)^d
    \]
    there exists a ReLU MLP $\Phi$ on $[0,1]^d$ with width at most $8d \max\{m_1,\cdots,m_L\} + d + 2$, depth at most $L (\lceil \log_2{d} \rceil + 4)$, at most $16d(n+1)^d + L(d+2)$ nonzero parameters, which satisfies the approximation guarantee 
    \[
        \| f-\Phi \|_{L^{\infty}([0,1]^d)} 
        \le 
        \omega\left(\frac{d}{2n}\right)
    \]
    Furthermore, $\omega$ is a modulus of regularity of $\Phi$ on $[0,1]^d$.
\end{proposition}

In Proposition~\ref{prop:MainTheorem_AdjustableSize}, if we take $L=d$ and $m_1 = m_2 = \cdots = m_d = \lceil (n+1)^d/d \rceil$, then we can conclude that the same approximator $\Phi$ in Theorem \ref{thm:MainTheorem_StableUniversalApproximation} can be implemented by another ReLU MLP with width $\mathcal{O}(n^d)$ and depth $\mathcal{O}(d \log{d})$. 


\subsection{Statistical Guarantees}
\label{s:Main__ss:Stats}

\label{s:Main__ss:Consequences___LearningTheory_I}
Next, we investigate the learning theoretic implications of our main result.  
We consider i.i.d.\ training data for a classification problem $(X_1,Y_1),\dots,(X_N,Y_N)\sim \mathbb{P}$ in $\mathbb{R}^d\times [0,1]$ and a loss function $\ell:\mathbb{R}\times \mathbb{R}\to[0,1]$.  Our objective is to describe the worst-case \textit{generalization gap} for any ReLU MLP $\Phi$ in the class studied in Theorem~\ref{thm:MainTheorem_StableUniversalApproximation}.  The generalization gap for any such MLP is defined as the absolute difference between the true risk $\mathcal{R}(\Phi)$ over the (unseen) test set and the empirical risk $\mathcal{R}(\Phi)$ computed on the training data; where
\begin{equation*}
    \mathcal{R}(\Phi) \eqdef \mathbb{E}_{(X,Y)\sim \mathbb{P}}\big[\ell(\Phi(X),Y)\big]
    ,\quad
    \mathcal{R}_N(\Phi) \eqdef \frac1{N}\sum_{n=1}^N\ell(\Phi(X_n),Y_n)
.
\end{equation*}
Theorem~\ref{thm:MainTheorem_StableUniversalApproximation} and Proposition~\ref{proposition:ParamterEstimates__RefinedVersion} identifies a highly structured subclass of the set of neural networks with depth $\Delta$ and width $W$, for $\Delta,W\in \mathbb{N}_+$, which are \textit{expressive enough} to approximate any $L$-Lipschitz functions with range in $[0,1]$ to a uniform precision of $\omega(d/2n)$ for any prescribed $n\in \mathbb{N}_+$.  Specifically, the class $\mathcal{NN}_{\Delta,W}^{L,n}$ consists of all ReLU MLPs $\Phi:\mathbb{R}^d\to \mathbb{R}$ in $\mathcal{NN}_{\Delta,W}^{L}(d,1)$, see Definition~\ref{defn:ReLUMLP}, for which the weights $\mathbf{W}^{(1)},\dots,\mathbf{W}^{(\Delta)}$ in~\eqref{eq:MLPRepresentation} satisfy
\begin{equation}
\label{eq:weight_constants}
        \|\mathbf{W}^{(1)}\|_{\infty}\le 1
    ,\,
        \max_{l=2,\dots,L-1}\, \|\mathbf{W}^{(l)}\|_{\infty}\le \frac{1}{2}
    \mbox{, and }
        \|\mathbf{W}^{(L)}\|_{\infty} \le n 
.
\end{equation}
\begin{assumption}[Normalized Sub-Gaussian Training Data]
\label{ass:NormalizedData}
The $\mathbb{R}^d\times [0,1]$-valued random variables $(X_1,Y_1)$,$\dots$,$(X_N,Y_N)$ are i.i.d and $X_1$ is centered, sub-Gaussian, with normalized covariance
\[
    \mathbb{E}\big[X_nX_n^{\top}\big] = \frac1{N}\mathbf{I}_d
\]
and are defined on a common probability space $(\Omega,\mathcal{A},\mathbb{P})$; where $\mathbf{I}_d$ is the $d\times d$ identity matrix.
\end{assumption}
Unlike the generalization bound in~\eqref{eq:RademacherBound}, deep and wide neural networks generalize well in practice.  On the contrary, several recent results studying gradient dynamics of ``infinitely wide'' neural networks suggest that highly overparameterized neural networks (which may be extremely deep and wide) tend to generalize well and tend to converge to highly regular networks after being optimized by gradient descent on the training data.  This is not captured by the generalization bound in~\eqref{eq:RademacherBound} since, when either $\Delta$ or $W$ become large, then the generalization bound in~\eqref{eq:RademacherBound} diverges, even if $L$ is held constant.  

The divergence of~\eqref{eq:RademacherBound} is counter-intuitive since $\mathcal{NN}_{\Delta,W}^{L,n}$ is contained in the class of $L$-Lipschitz functions on $\mathbb{R}^d$ with image in $[0,1]$; which is totally bounded by Arzel\`a-Ascoli theorem, and thus its Rademacher complexity should be controllable by Dudley's entropy integral estimate (see e.g.~\citep[Corollary 2.2.8]{VanderVaart_WeakConvEmpiricalProcesses}).  The reason is that the bound in~\eqref{eq:RademacherBound} is inherited from bounding the Rademacher complexity of the larger class $\mathcal{NN}_{\Delta,W}$ which becomes unbounded as the depth and width parameters grow.  Upon noting that the Rademacher complexity of the class $\mathcal{NN}_{\Delta,W}^{L,n}$ must be no larger than that of the class of $L$-Lipschitz functions with values in $[0,1]$ and no larger than that of the class $\mathcal{NN}_{\Delta,W}$, then using the bound on the former computed for instance in~\cite[Lemma 25]{HouKratsios_JMLR_2023}, we may improve~\eqref{eq:RademacherBound} so that it remains bounded as $\Delta$ and $W$ tend to infinity.
\begin{theorem}[Non-Exploding Generalization Bounds for Arbitrarily Deep and Wide ReLU MLPs]
\label{thm:RademacherBound}
Let $\Delta,W,n,d\in \mathbb{N}_+$, $L,L_{\ell}>0$, and consider a $L_{\ell}$-Lipschitz loss function $\ell:\mathbb{R}^2\to [0,1]$.
Under Assumption~\ref{ass:NormalizedData}, for each $\delta \in (0,1)$ the generalization gap $\sup_{\Phi\in \mathcal{NN}_{\Delta,W}^{L,n}} \big|\mathcal{R}(\Phi)-\mathcal{R}_N(\Phi)\big|$ is bounded above by
\begin{equation*}
\resizebox{.99\hsize}{!}{$%
\begin{aligned}
    \underbrace{
        \frac{
            \sqrt{8\log(4/\delta)}
        }{
            \sqrt{N}
        }
    }_{\text{Prob. Satisfaction}}
    +
    2\,L_{\ell}\,
    \min\biggl\{
            \underbrace{
             \frac{4}{N^{3/2}}
                        +
                        \frac{
                            26\log(N)\log(2W)
                            \,
                            W^{3\Delta/2}
                        }{
                            2^{\max(\{0,\Delta-2\}}
                            N
                        }
                        \,
                        \big(
                            \sqrt{d}
                            +
                            C_{d,X_1}\,\sqrt{N}
                            +
                            C_{d,X_1}\,\sqrt{\ln(4/\delta)}
                        \big)
            }_{\text{Parameter Space}}
    ,
            \underbrace{
                C_d\,\frac{L^{\frac{d}{d+3}}}{N^{\frac{1}{d+3}}}
            }_{\text{Function Space}}
    \biggr\}
\end{aligned}
$}%
\end{equation*}
with probability at least $1-\delta$; where $C_d\eqdef \big(8(d+1)^2 (16)^d\big)^{1/(d+3)} + \frac{2^{5/2}\,16^{d/(d+3)}}{(18(d+1))^{(d+1)/(d+3)}}$ and $C_{d,X_1}>0$ is a constant depending only on $X_1$ and $d$.
\end{theorem}

\section{{Implications of Theorem~\ref{thm:MainTheorem_StableUniversalApproximation}}}
\label{s:Consequences}
We now showcase the breadth of Theorem \ref{thm:MainTheorem_StableUniversalApproximation} by examining some of its implications from approximation theory to learning theory.  

\subsection{Approximation Theory: Extrapolation Rates}
\label{s:Main__ss:Consequences___Extrapolation}
Our main result guarantees that the approximating ReLU MLP $\Phi$ has a (modulus of) regularity which is no worse than that of the target function.  This allows us to deduce a sharp extrapolation rate for our approximator outside of $[0,1]^d$.

We can thus directly obtain a version of a \textit{global universal approximation theorem}.  However, our result differs from that of~\cite{cuchiero2023global}, which leverage topologies stronger than the uniform convergence on compacts topology on $C(\mathbb{R}^d)$ but weaker than the uniform topology thereon on suitable subspaces; e.g.\ the strict topology on continuous bounded functions studied by~\cite{ToddStoneWeirestrassStricTop_1965,giles1971generalization,Timofte2Khan_SWExtNLC_2018}.  Instead, we do not discount/compress our error using weight functions outside the unit cube but rater we quantify how fast it grows, using the moduli of \textit{both} the target function and its neural network approximator.  

The first step towards this result is given by our first consequence of Theorem~\ref{thm:MainTheorem_StableUniversalApproximation}.  Namely, this first corollary shows that we can easily extend the domain of the interpolator $\Phi$ in Theorem \ref{thm:MainTheorem_StableUniversalApproximation} to all of $\mathbb{R}^d$ while maintaining its regularity.
To formulate this consequence, we consider the distance $\operatorname{dist}_{[0,1]^d}$ from any point to the $d$-dimensional unit cube, defined for each $\mathbf{x}\in \mathbb{R}^d$ by
\[
    \operatorname{dist}_{[0,1]^d}(\mathbf{x})\eqdef \min_{\mathbf{z}\in 
        [0,1]^d
    }\, 
        \|\mathbf{z}-\mathbf{x}\|_1
.
\]

\begin{corollary}[Theorem~\ref{thm:MainTheorem_StableUniversalApproximation} with Optimal Global Regularity]
\label{cor:main_StableUniversalApproximation___with_OptimalLipschitzExtrapolation}
    For any $n\in\mathbb{N}_+$ and $f:[0,1]^d \rightarrow\mathbb{R}$ with modulus of regularity $\omega:[0,d] \rightarrow\mathbb{R}$, there exists a ReLU MLP $\Phi: \mathbb{R}^d \rightarrow \mathbb{R}$ with width at most $8d(n+1)^d$, depth at most $\lceil \log_2{d} \rceil + 5$, and no more than $18d(n+1)^d$ nonzero parameters such that 
    \[
        \|f-\Phi\|_{L^{\infty}([0,1]^d)} 
        \le 
        \omega\left(\frac{d}{2n}\right)
    \]
    Moreover, the following extension of $\omega$ is a modulus of regularity of $\Phi$ on $\mathbb{R}^d$:
    \begin{equation}
    \label{eq:omega_ext}
        \bar{\omega}(x) \eqdef
        \begin{cases}
            \omega(x), & \text{if } x\in[0,d]  \\
            \omega(d), & \text{if } x > d 
        \end{cases},
        \quad \forall x\in[0,\infty)
    \end{equation}
\end{corollary}

Using Corollary~\ref{cor:Extrapolation}, we are able to deduce the following ``global universal approximation theorem'' with \textit{additive correction term}, instead of the usual \textit{multiplicative corrective weight} used in the global approximation literature~\cite[Theorem 4.13]{cuchiero2023global}, in the reservoir computing literature, e.g.~\cite[Corollary 9]{LuydmillaJP_UAT_2018}, or in the deep learning for dynamical systems literature; e.g.~\cite[Theorem 4.11]{AcciaioKratsiosPammer_GHT_2024}. 

\begin{corollary}[Extrapolation Bounds for ReLU MLP Approximators]
\label{cor:Extrapolation}
For any $n\in\mathbb{N}_+$ and $f:\mathbb{R}^d \rightarrow\mathbb{R}$ with modulus of regularity $\omega:[0,d]\rightarrow\mathbb{R}$ on $[0,1]^d$, 
there is a ReLU MLP $\Phi:\mathbb{R}^d\to \mathbb{R}$ with width at most $8d(n+1)^d$, depth at most $\lceil \log_2{d} \rceil + 5$, and no more than $18d(n+1)^d$ nonzero parameters satisfying
\[
    \sup_{\mathbf{x}\in \mathbb{R}^d}
    \,
        \underbrace{
            |f(\mathbf{x})-\Phi(\mathbf{x})|
        }_{\text{Approximation in }[0,1]^d}
        -
        \underbrace{
            \bar{\omega}(\operatorname{dist}_{[0,1]^d}(\mathbf{x}))
        }_{\text{Extrapolation: beyond }[0,1]^d}
    \le 
        \bar{\omega}\left(\frac{d}{2n}\right)
.
\]
where $\bar{\omega}$ is defined in \eqref{eq:omega_ext}. Furthermore, $\bar{\omega}$ is a modulus of regularity of $\Phi$ on $\mathbb{R}^d$.
\end{corollary}

\section{Discussion}
\label{s:Discussion}

We now discuss some technical points of our main result, explaining how certain steps in our construction cannot be improved on.  The proofs of each of these results are relegated to Section~\ref{s:Proofs__ss:Discussion} below.

\subsection{The Kuhn Triangulation is Unique for Regular Approximation}
\label{s:Discussion__ss:KuhnTriUnique}

The multi-dimensional analogue of the cartoon of our one-dimensional regular ReLU approximator illustrated in Figure~\ref{fig:Proof_Strategy_Idea__NEW}, relies on a specific triangulation of the $d$-dimensional cube $[0,1]^d$.  This triangular called the \emph{Kuhn triangulation}, defined shortly afterwards, is the most ``regular'' triangulation in the sense that it is the only triangulation whose continuous piecewise linear approximators always preserve the regularity of the function that it approximates under some mild conditions; see~\cite[Lemma 1]{kim2023algebraic}.  This triangulation is used in one of the key technical steps in proving Theorem~\ref{thm:MainTheorem_StableUniversalApproximation}; namely, Lemma~\ref{lem:CPWLApprox_wRegControl}.

\begin{figure}[!htbp]
\centering

\includegraphics[width=0.3\linewidth]{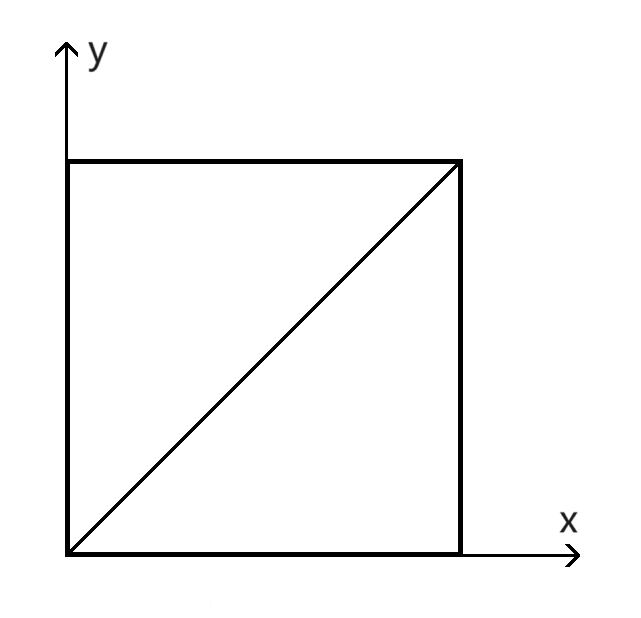}
\hspace{10pt}
\includegraphics[width=0.3\linewidth]{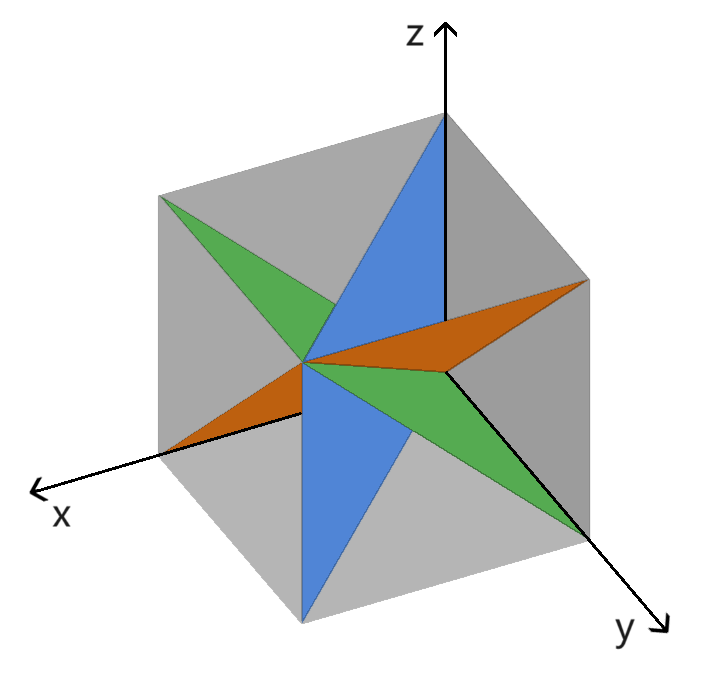}

\caption{Kuhn triangulations of a 2D (left) and a 3D (right) unit cube. }
\label{fig:kuhn_local}
\end{figure}

\begin{definition}[Kuhn triangulation]
\label{defn:Kuhn_triangulation}
    For any $d\in\mathbb{N}_+$, the $d!$ different simplices 
    \begin{equation}
    \label{eq:kuhn_triang_d!_simplices}
        \operatorname{conv}\left(\left\{
        \sum_{i=1}^k \mathbf{e}_{\tau(i)}: 0 \le k \le d
        \right\}\right),
        \quad \tau\in S_d
    \end{equation}
    form a triangulation of $[0,1]^d$, where $S_d$ is the symmetric group on $d$ letters, and $\mathbf{e}_1,\cdots,\mathbf{e}_d$ are the standard basis vectors of $\mathbb{R}^{d}$. This triangulation is called the Kuhn triangulation.
\end{definition}

It is natural to ask if another triangulation could have been used during the proof of our main result.  Surprisingly, this is not the case as any other triangulation, other than reflections of the Kuhn triangulation, yield approximators with larger Lipschitz constants.  This is the content of the following result which serves as a type of converse to Lemma \ref{lem:CPWLApprox_wRegControl}.
\begin{theorem}[Uniqueness of the Kuhn Triangulation - up to Reflections]
\label{thm:KuhnUnique}
    The Kuhn triangulation is the only triangulation (up to reflections) that makes Lemma \ref{lem:CPWLApprox_wRegControl} true for the case $n=1$. 
\end{theorem}

\subsection{Parameter Usage Compared to State-of-the-Art Approximation}
\label{s_Discussion__ss:SOTA_w_RegulartyTracked}
This section contains a statement and derivation of the main result of the state-of-the-art (SOTA) approximation theorem for ReLU MLPs~\cite{ShenYangZhang_JMPA_OptApprx_ReLU} which \textit{includes estimates on the Lipschitz constant of the ReLU MLP constructed therein}.  We emphasize that the optimality criterion in the following approximation theorem and its predecessor was parameter usage, not regularity.

\begin{theorem}[{SOTA Universal Approximation Theorem of~\cite{ShenYangZhang_JMPA_OptApprx_ReLU} with Regularity Quantification}]
\label{thm:main_SOTA_Benchmark}
    For any $2\le n\in\mathbb{N}_+$ and $f\in\operatorname{Hol}(\nu,\alpha,[0,1]^d)$ with $\alpha\in(0,1]$, we have:
    \begin{enumerate}
        \item If $d=1$, then $\exists\Phi\in\mathcal{NN}(\#\text{input}=1; \operatorname{widthvec}=[2n+1,2n])$ with $\mathcal{O}(n^2)$ nonzero parameters such that 
        $$\|f-\Phi\|_{L^{\infty}([0,1])} \le \frac{\nu}{\left(2n^2\right)^{\alpha}}$$
        $$\Phi\in\operatorname{Hol}(\nu,\alpha,[0,1])$$
        \item If $d>1$, then $\exists\Phi\in\mathcal{NN}(\#\text{input}=d; \text{width} \le \max \{(2n+3)d, 6d+3, 2\lceil (n+1)^{d/2} \rceil+2d \};\text{depth} \le 23d+9)$ with $\mathcal{O}(dn^d)$ nonzero parameters such that
        $$\|f-\Phi\|_{L^{\infty}([0,1]^d)} \le \nu \left( 
        \frac{d^2}{n(2d+1)}
        \right)^{\alpha}$$
        $$\operatorname{Lip}(\Phi) \le 
        \nu (n+1)^d (2d+1) d^{\alpha}$$
    \end{enumerate}
\end{theorem}
Observe that the minor improvement of the main result of~\cite{ShenYangZhang_JMPA_OptApprx_ReLU}, given in Theorem~\ref{thm:main_SOTA_Benchmark}, yields significantly more irregular MLPs than our main theorem due to sharp descents between trifling regions (illustrated by the red bands in Figure~\ref{fig:Proof_Strategy_Idea__OLD}).  
In contrast, our Kuhn triangulation-based construction resolves this issue optimally since the ReLU MLPs constructed using it do not require any trifling region.

\begin{remark}[{Improvement of Constant for Variant of~\cite[Theorem 1.1]{ShenYangZhang_JMPA_OptApprx_ReLU}}]
In the multi-dimensional case, our version of the main result of~\cite{ShenYangZhang_JMPA_OptApprx_ReLU} achieves the optimal approximation rate with a more efficient parameter usage than the original formulation in~\cite{ShenYangZhang_JMPA_OptApprx_ReLU}.  
In our version of their result, we achieve the same approximation rate while dropping the dependence of the constant on $d$ from exponential to linear.
\end{remark}


\subsection{Path Norms and Lipschitz Constants}
 We mention that $|\Phi|_{\operatorname{par}}$, defined in~\eqref{eq:parameter_norm}, is closely related to the so-called ``path-norm'' seen in the literature, defined by
\begin{equation*}
    |\Phi|_{\operatorname{path}} \eqdef \prod_{l=1}^{\Delta}\,\|\mathbf{W}^{(l)}\|_{\operatorname{op}}
\end{equation*}
where $\|\cdot\|_{\operatorname{op}}$ denotes the spectral norm of a matrix.  The path norm is typically used to quantify the regularity of a parametric representation of a neural network~\cite{NeyshaburetALPathnorm_NIPS2017}, and which is often used as a regularizer during training~\cite{neyshabur2015pathSGD_NIPS}.  The path-norm provides readily computable upper bound on the Lipschitz constant of $\Phi$, which can be easily computed from the parameter space of $\Phi$.  This is in stark contrast to the exact Lipschitz constant of ReLU MLPs, which can be difficult to exactly in practice~\cite{virmaux2018lipschitz}.  Here, we mention that $|\Phi|_{\operatorname{par}}$ can easily be estimated from the values in Table~\ref{tab:MLP_Description}; however, one readily sees that the path norm bounded in this way can severely overestimate the Lipschitz constant of the function \textit{implemented} by the ReLU MLP in Theorem~\ref{thm:main_IntroVersion}.

\section{Proofs}
\label{s:Proofs}

\subsection{Additional Notation During Proofs}
\label{s:Proofs__ss:FurtherNotation}
During our proofs, we will also adopt the following notational conventions, in addition to those described in Section~\ref{s:Prelim__ss:Notation}.
A single layer of an MLP is considered as a column vector. 
For each $m\in \mathbb{N}_+$, the median of the real numbers $x_1,x_2,\cdots,x_m$, denoted by $\operatorname{median}(x_1,x_2,\cdots,x_m)$, is
\begin{equation}
\label{eq:Median_Function}
    \operatorname{median}(x_1,x_2,\cdots,x_m) 
\eqdef 
    \begin{cases}
        x_{(m+1)/2} & \mbox{ if } m \mbox{ is odd} \\
        (x_{(m/2)}+x_{(m/2+1)})/2 & \mbox{ if } m \mbox{ is even} \\
    \end{cases}
\end{equation}
where $\{x_{(i)}\}_{i=1}^m=\{x_i\}_{i=1}^m$ and $x_{(1)}\le \dots \le x_{(m)}$.
With a slight abuse of notation, we sometimes abbreviate a column vector in the following way: for example, if $\mathbf{x} = \begin{bmatrix} x_1 & x_2 \end{bmatrix}^{\top}$ and $\mathbf{y} = \begin{bmatrix} y_1 & y_2 & y_3 \end{bmatrix}^{\top}$, then 
$$\begin{bmatrix} 
    \mathbf{x} & \sigma(\pm a \pm b) & \mathbf{y} 
\end{bmatrix}^{\top}
\eqdef
\begin{bmatrix} 
    x_1 & x_2 
    & \sigma(a+b) & \sigma(a-b) & \sigma(-a+b) & \sigma(-a-b)
    & y_1 & y_2 & y_3
\end{bmatrix}^{\top}
.
$$
For any $p\in[1,\infty)$, the $\ell^p$ norm of a vector $\mathbf{x}=(x_1,x_2,\cdots,x_d)\in\mathbb{R}^d$ is defined as 
$$\| \mathbf{x} \|_p \eqdef (|x_1|^p + |x_2|^p + \cdots + |x_d|^p)^{1/p}$$
and its $\ell^{\infty}$ norm is defined as 
$$\| \mathbf{x} \|_{\infty} \eqdef \max\{|x_1|,|x_2|,\cdots,|x_d|\}$$
The $L^p$ norm of a function $f:\Omega\subseteq\mathbb{R}^d \rightarrow \mathbb{R}$ is defined as 
$$\| f \|_{L^p(\Omega)} \eqdef 
\left( \int_{\Omega} |f|^p d\mu \right) ^ {1/p}$$
where $\mu$ is the Lebesgue measure, and its $L^{\infty}$ norm is defined as 
$$\| f \|_{L^{\infty}(\Omega)} \eqdef 
\inf\{ C\ge 0: |f| \le C \text{ almost everywhere on } \Omega \}.$$
This paper uses the $\ell^{1}$ norm on the domain; for instance, when defining Lipschitz constants and regularity. We use the (uniform) $L^{\infty}$ norm when quantifying approximation estimates.

We use $\mathcal{O}$ for the big $O$ notation, which is written as $\tilde{\mathcal{O}}$ when logarithmic factors are ignored. We use $\Omega$ for the big $\Omega$ notation when it is clear from the context. 

For a set $S$, let $\operatorname{card}(S)$ denote its cardinality. For two sets $S_1$ and $S_2$, let $S_1 \times S_2 \eqdef \{(s_1,s_2) : s_1 \in S_1, s_2 \in S_2\}$ be the Cartesian product of $S_1$ and $S_2$, and denote the Cartesian product of $S_1$ with itself $d$ times for $d\in\mathbb{N}_+$ as $S_1^d$, i.e. 
$S_1^d \eqdef \prod_{i=1}^d\, S_1= \{ (s_1,s_2,\cdots,s_d) : s_1,s_2,\cdots,s_d \in S_1 \}
$.

For a finite subset $S=\{\mathbf{v}_1,\mathbf{v}_2,\cdots,\mathbf{v}_n\}$ of $\mathbb{R}^d$, let $\operatorname{conv}(S)$ denote the convex hull of $S$, i.e. 
$$\operatorname{conv}(S) \eqdef 
\{ \lambda_1 \mathbf{v}_1 + \lambda_2 \mathbf{v}_2 
+ \cdots
+ \lambda_n \mathbf{v}_n :
\lambda_1,\cdots,\lambda_n \ge 0,
\lambda_1 + \lambda_2 + \cdots + \lambda_n = 1\}
.
$$
For a set of sets $S$, let $\cup S \eqdef \bigcup_{S'\in S} S'$ denote the union of all the sets in $S$.

We say that a $1$-dimensional function $f:\Omega\subseteq\mathbb{R}\rightarrow\mathbb{R}$ is continuous piecewise linear if there exists a finite collection of intervals $I_1,\cdots,I_m$ such that $\Omega\subseteq I_1\cup\cdots\cup I_m$, and $f$ is linear on $I_1,\cdots,I_m$. We say that $x\in\Omega$ is a \textit{break point} of $f$ if the left and right-hand derivatives of $f$ at $x$ are different. For example, $\sigma$ is a continuous piecewise linear function on $\mathbb{R}$, and $x=0$ is the only break point of $\sigma$.

In what follows, given a ReLU MLP $\Phi$ with representation~\eqref{eq:MLPRepresentation}; when convenient, we will use $\operatorname{depth}(\Phi)$ to denote its \textit{depth} of a network $\Phi$ and $\operatorname{width}(\Phi)$ to denote its width.
Emulating~\cite{ShenYangZhang_JMPA_OptApprx_ReLU}, we use $\mathcal{NN}(c_1;c_2;\cdots;c_m)$ to denote the class of functions implemented by ReLU MLPs which satisfy conditions $c_1,c_2,\cdots,c_m$. For example, $\mathcal{NN}(\#\text{input}=2;\text{depth}\le 3;\text{width} = 4;\#\text{output}=5)$ denote the class of functions from $\mathbb{R}^2$ to $\mathbb{R}^5$ implemented by ReLU MLPs which have depth at most $3$ and width equal to $4$.

One example of this is the notation $\mathcal{NN}(\mathbf{d})$ introduced in Definition~\ref{defn:ReLUMLP}.  In this case, the condition is $\operatorname{widthvec}(\Phi)=\mathbf{d}$.

\paragraph{Compositional Notation}
During the course of our analysis, it will be convenient to describe ReLU MLPs via the role of each of their (sets of) layers.  Specifically, the structure of a ReLU MLP $\Phi$ is represented in the following way: suppose $\Phi = \mathcal{L}_m \circ (\sigma \circ \mathcal{L}_{m-1}) \circ \cdots \circ (\sigma \circ \mathcal{L}_2) \circ (\sigma \circ \mathcal{L}_1)$ where the $\mathcal{L}_i$'s are affine transformations, then we express it as 
\begin{align*}
    \mathbf{x}
    \ &\Longrightarrow \
    (\sigma \circ \mathcal{L}_1) (\mathbf{x})
    \ \Longrightarrow \
    (\sigma \circ \mathcal{L}_2) \circ (\sigma \circ \mathcal{L}_1) (\mathbf{x})
    \\
    \ &\Longrightarrow \
    (\sigma \circ \mathcal{L}_{m-1}) \circ \cdots \circ (\sigma \circ \mathcal{L}_2) \circ (\sigma \circ \mathcal{L}_1) (\mathbf{x})
    \\
    \ &\Longrightarrow \
    \mathcal{L}_m \circ (\sigma \circ \mathcal{L}_{m-1}) \circ \cdots \circ (\sigma \circ \mathcal{L}_2) \circ (\sigma \circ \mathcal{L}_1) (\mathbf{x})
    \\
    \ &= \ \Phi(\mathbf{x})
\end{align*}
In other words, if $\mathbf{x}_1,\mathbf{x}_2,\cdots,\mathbf{x}_{m-1}$ are the $1,2,\cdots,(m-1)$-th hidden layer of $\Phi$ and $\mathbf{x}_m$ is the output layer, then the structure of $\Phi$ is expressed as 
$$\mathbf{x} \ \Longrightarrow \ 
\mathbf{x}_1 \ \Longrightarrow \
\mathbf{x}_2 \ \Longrightarrow \
\cdots \ \Longrightarrow \
\mathbf{x}_{m-1} \ \Longrightarrow \
\mathbf{x}_m
\ = \ \Phi(\mathbf{x})$$
If the structure of the network $\Phi$ with input $\mathbf{x}$ and output $\Phi(\mathbf{x})$ is clear from the context, then in the following expression, we mean that $\mathbf{x}$ and $\Phi(\mathbf{x})$ are connected via the network $\Phi$: 
\[
    \cdots
    \ \Longrightarrow \
    \begin{bmatrix}
        \vdots \\ \mathbf{x} \\ \vdots
    \end{bmatrix}
    \ \Longrightarrow \ 
    \begin{bmatrix}
        \vdots \\ \Phi(\mathbf{x}) \\ \vdots
    \end{bmatrix}
    \ \Longrightarrow \
    \cdots
\]
In this case, there may be additional hidden layers between the two layers shown above. Finally, if one of the layers we constructed in fact did not use any activation functions, then it can be integrated with the layer after it, an does not require an extra layer to process. We color them in green to indicate that we can ignore these layers when estimating the size of the network. For example, if $\Phi \eqdef \mathcal{L}_4 \circ (\sigma \circ \mathcal{L}_3) \circ (\mathcal{L}_2) \circ (\sigma \circ \mathcal{L}_1)$ is constructed as a network with three hidden layers $\mathbf{x}_1 \eqdef \sigma \circ \mathcal{L}_1 (\mathbf{x})$, $\mathbf{x}_2 \eqdef \mathcal{L}_2 (\mathbf{x}_1)$ and $\mathbf{x}_3 \eqdef \sigma \circ \mathcal{L}_3 (\mathbf{x}_2)$, and $\mathbf{x}_2$ does not use any activation functions, then $\Phi$ can be implemented by a network with only two hidden layers: $\Phi = \mathcal{L}_4 \circ (\sigma \circ (\mathcal{L}_3 \circ \mathcal{L}_2)) \circ (\sigma \circ \mathcal{L}_1)$, whose structure is expressed as 
\[
    \mathbf{x}
    \ \Longrightarrow \
    \mathbf{x}_1
    \ \Longrightarrow \
    {\color{persiangreen} \mathbf{x}_2 }
    \ \Longrightarrow \
    \mathbf{x}_3
    \ \Longrightarrow \
    \Phi(\mathbf{x})
\]

\subsection{Computational Lemmata}
\label{s:Compute}

We now compile a sequence of lemmata showing that various key functions can be exactly implemented by ``small'' MLPs; such as the median function and piecewise linear memorizes.  Moreover, when relevant, we show that the MLPs implementing these functions do not alter the regularity of their inputs (explained rigorously below).

\subsubsection{The Median Function}
\label{s:Compute__ss:Median}
This section focuses on the median function; defined in the following lemma.  The implementation of the median function will allow us to construct MLPs with lighted absolute parameter usage (including constants) than in the available literature while matching the available optimal approximation rates.
\begin{lemma}
\label{lem:median_estimators_have_small_errors}
    Let $d\in\mathbb{N}_+$, and $x_1,x_2,\cdots,x_{2d+1}\in\mathbb{R}$. Assume $x_1,\cdots,x_{d+1}\in [y-\varepsilon, y+\varepsilon]$ for some $y\in\mathbb{R}$ and $\varepsilon>0$. Then, $\operatorname{median}(x_1,x_2,\cdots,x_{2d+1})\in [y-\varepsilon, y+\varepsilon]$. 
\end{lemma}

\begin{proof}
Let's relabel $x_1,x_2,\cdots,x_{2d+1}$ as $y_1 \le y_2 \le \cdots \le y_{2d+1}$. Then, at least one of $x_1,x_2,\cdots,x_{d+1}$, denoted $x_n$, is in the set $\{y_1,y_2,\cdots,y_{d+1} \}$, since the complement of this set only has $d$ elements. Similarly, at least one of $x_1,x_2,\cdots,x_{d+1}$, denoted $x_m$, is in the set $\{y_{d+1},y_{d+2},\cdots,y_{2d+1} \}$. Therefore,
$$\operatorname{median}(x_1,x_2,\cdots,x_{2d+1})=y_{d+1}\in [x_n,x_m]\subseteq [y-\varepsilon, y+\varepsilon].$$
\end{proof}

\begin{lemma}
\label{lem:relu_for_computing_median}
    For any $d\in\mathbb{N}_+$, the median function on $2d+1$ non-negative inputs can be implemented by a ReLU MLP with width $6d+3$ and depth $11d+3$.
\end{lemma}

\begin{proof}
We label the nodes as $n_1,n_2,\cdots,n_{2d+1}$, and let 
$$D(x) \eqdef \sum_{i=1}^{2d+1} |x-n_i|$$
be the total distance from $x$ to all the nodes. Then, $D$ attains its minimum exactly at the median since there are odd number of nodes. We will use the following algorithm to compute the median:
\begin{enumerate}
    \item Compute $x_0 \eqdef \min\{n_1,n_2,\cdots,n_{2d+1}\}$
    \item Compute $s_0 \eqdef \min\{D(n_1),D(n_2),\cdots,D(n_{2d+1})\}$
    \item For $i=1,2,\cdots,d$, let $\displaystyle x_i=x_{i-1}+\frac{D(x_{i-1})-s_0}{2d+1-2i}$
    \item Output $x_d$
\end{enumerate}

We first show the correctness of this algorithm. We relabel the nodes as $m_1\le m_2\le\cdots\le m_{2d+1}$ and proceed by induction on $i$. For $i=0$, $x_0=m_1\in [m_1,m_{d+1}]$. For $i=k$, suppose $x_k\in [m_{k+1},m_{d+1}]$ for some $k\le d-1$, we will show that $x_{k+1}\in [m_{k+2},m_{d+1}]$.

If $x_k\in [m_{k+1},m_{k+2}]$, then noting that for $r\le d$, 
$$D(m_r)-D(m_{r+1})=(2(d-r)+1)(m_{r+1}-m_r)$$
and the same holds when $m_{r+1}$ changes to any $x\in [m_{r},m_{r+1}]$, in particular:
$$D(m_{k+1})-D(x_k)=(2(d-r-1)+1)(x_k-m_{k+1})$$
so we have
\begin{align*}
    x_{k+1} &= x_k+\frac{D(x_k)-s_0}{2d-1-2k} = x_k+\frac{D(x_k)-D(m_{d+1})}{2(d-k)-1} \\
    &= x_k+\frac{D(m_{k+1})-(2(d-k-1)+1)(x_k-m_{k+1})-D(m_{d+1})}{2(d-k)-1} \\
    &= x_k-(x_k-m_{k+1})+\frac{D(m_{k+1})-D(m_{d+1})}{2(d-k)-1} \\
    &= m_{k+1}+\frac{1}{2(d-k)-1} \sum_{j=1}^{d-k} (2j-1)(m_{d+2-j}-m_{d+1-j})
\end{align*}
then 
$$x_{k+1} \le m_{k+1} + \sum_{j=1}^{d-k} (m_{d+2-j}-m_{d+1-j}) = m_{d+1}$$
$$x_{k+1} \ge m_{k+1} + \frac{1}{2(d-k)-1} \sum_{j=d-k}^{d-k} (2j-1)(m_{d+2-j}-m_{d+1-j}) = m_{k+2}$$
hence $x_{k+1}\in [m_{k+2},m_{d+1}]$.

Otherwise, $x_k\in [m_{k+2},m_{d+1}]$, then suppose $x_k\in [m_{k'+1},m_{k'+2}]$ for some $k<k'\le d-1$, so we know from above that 
$$x_{k}+\frac{D(x_{k})-D(m_{d+1})}{2d-1-2k'} \le m_{d+1}$$
thus
$$m_{k+2} \le x_k \le x_{k+1} = x_{k}+\frac{D(x_{k})-D(m_{d+1})}{2d-1-2k} \le x_{k}+\frac{D(x_{k})-D(m_{d+1})}{2d-1-2k'} \le m_{d+1}$$
hence $x_{k+1}\in [m_{k+2},m_{d+1}]$, so we completed the induction step.

Therefore, $x_i\in [m_{i+1},m_{d+1}]$ for $i=1,2,\cdots,d$, so $x_d\in [m_{d+1},m_{d+1}]$, $x_d=m_{d+1}$, which is the median.

In the following, we will use this algorithm to construct our ReLU MLP. In the rest of this proof, column vectors represent layers of neurons, and functions apply to them entry-wise, whose outputs are still column vectors.

\textbf{Step 1}: Compute $x_0$ and $s_0$

It is easy to see that
$$x=\sigma(x)-\sigma(-x), \quad |x|=\sigma(x)+\sigma(-x),
\quad \forall x\in\mathbb{R}$$
then 
$$\min\{x,y\} = \frac{x+y}{2} - \frac{|x-y|}{2} = \frac{\sigma(x+y)}{2} - \frac{\sigma(-x-y)}{2} - \frac{\sigma(x-y)}{2} - \frac{\sigma(-x+y)}{2},
\quad \forall x,y\in\mathbb{R}$$
so we can use this formula to construct the following building block for computing $s_0$: denote the column vector $\begin{bmatrix} n_1 & n_2 & \cdots & n_{2d+1} \end{bmatrix}^{\top}$ as $\mathbf{n}$, and let $\mathbf{n}_k$ be the column vector obtained from $\mathbf{n}$ by deleting its $k$th row, then 
$$D(n_k) = \begin{bmatrix} 1 & 1 & \cdots & 1 \end{bmatrix} \cdot [|n_k-\mathbf{n}_k|] 
= \begin{bmatrix} 1 & 1 & \cdots & 1 \end{bmatrix} \cdot \begin{bmatrix} 
\sigma(n_k-\mathbf{n}_k) \\ \sigma(\mathbf{n}_k-n_k) \end{bmatrix}$$
Using these results, we can compute $x_0$ and $d_0$ via the following network:
\begin{align*}
    \begin{bmatrix} \mathbf{n} \\ \min\limits_{1\le i\le k-1}{n_i} \\ \min\limits_{1\le i\le k-1}{D(n_i)} \end{bmatrix} 
    \ &\Longrightarrow \
    \begin{bmatrix} \mathbf{n} \\ \min\limits_{1\le i\le k-1}{n_i} \\ \sigma(n_k-\mathbf{n}_k) \\ \sigma(\mathbf{n}_k-n_k) \\ \min\limits_{1\le i\le k-1}{D(n_i)} \end{bmatrix}
    \ \Longrightarrow \
    \begin{bmatrix} \mathbf{n} \\ \sigma\left(\pm\min\limits_{1\le i\le k-1}{n_i} \pm n_k\right) \\ D(n_k) \\ \min\limits_{1\le i\le k-1}{D(n_i)} \end{bmatrix}
    \\
    \ &\Longrightarrow \
    \begin{bmatrix} \mathbf{n} \\ \min\limits_{1\le i\le k}{n_i} \\ 
    \sigma\Big(
    \pm\min\limits_{1\le i\le k-1}{D(n_i)} \pm D(n_k)
    \Big) 
    \end{bmatrix}
    \ \Longrightarrow \
    \begin{bmatrix} \mathbf{n} \\ \min\limits_{1\le i\le k}{n_i} \\ \min\limits_{1\le i\le k}{D(n_i)} \end{bmatrix}
\end{align*}
where $\sigma(\pm a\pm b)$ abbreviates the four neurons $\sigma(a+b),\sigma(a-b),\sigma(-a+b)$ and $\sigma(-a-b)$. Note that the terms without the $\sigma$ are non-negative and thus are unaffected by the activation function $\sigma$. Therefore, by connecting $2d$ such networks we can output $x_0=\min\limits_{1\le i\le k-1}{n_i}$ and $s_0=\min\limits_{1\le i\le k-1}{D(n_i)}$. Its width is dominated by the second layer in the above, which is $(2d+1)+1+2d+2d+1=6d+3$, and its depth is $4(2d+1)=8d+4$ (not counting the first layer, which is the input layer).

\textbf{Step 2}: Compute $x_i$ for $i=1,2,\cdots,d$

We use the following network:
$$
\begin{bmatrix} \mathbf{n} \\ x_{k-1} \\ s_0 \end{bmatrix}
\ \Longrightarrow \
\begin{bmatrix} \mathbf{n} \\ \sigma(x_{k-1}-\mathbf{n}) \\ \sigma(\mathbf{n}-x_{k-1}) \\ x_{k-1} \\ s_0 \end{bmatrix}
\ \Longrightarrow \
\begin{bmatrix} \mathbf{n} \\ D(x_{k-1}) \\ x_{k-1} \\ s_0 \end{bmatrix}
\ \Longrightarrow \
\begin{bmatrix} \mathbf{n} \\ \displaystyle x_{k-1}+\frac{D(x_{k-1})-s_0}{2d+1-2k}=x_k \\ s_0 \end{bmatrix}
$$
Therefore, by connecting $d$ such networks we can output $x_d$, which is the median as proven above. Its width is dominated by the second layer in the above, which is $(2d+1)+2d+2d+1+1=6d+3$, and its depth is $3d-1$ (not counting the first and the last layer, since the first layer is shared with the network in the previous step, and the last layer is the output).

Connecting the two network in steps 1 and 2 gives the desired ReLU MLP for computing the median of non-negative inputs. In total, this network has width $6d+3$ and depth $(8d+4)+(3d-1)=11d+3$.

\end{proof}

\begin{lemma}[Median function preserves regularity]
\label{lem:median_function_preserves_regularity}
    If $f_1,f_2,\cdots,f_n\in\operatorname{Hol}(\nu,\alpha,X)$ are Lipschitz functions from a metric space $X$ to $\mathbb{R}$, then $f \eqdef \operatorname{median}(f_1,f_2,\cdots,f_n)\in\operatorname{Hol}(\nu,\alpha,X)$.
\end{lemma}

\begin{proof}
Take any $\mathbf{x},\mathbf{y}\in X$, then $|f_i(\mathbf{x})-f_i(\mathbf{y})| \le \nu\|\mathbf{x}-\mathbf{y}\|_X^{\alpha} =:C$ for all $i$. Let $y_i:[0,1]\xrightarrow[]{} \mathbb{R}$ be a linear function such that $y_i(0)=f_i(\mathbf{x}),y_i(1)=f_i(\mathbf{y})$ for all $i$, then $y_i'\le C$ for all $i$. For any $t\in[0,1]$, let $p(t)=(s_1,s_2,\cdots,s_n)$, where $s_1,s_2,\cdots,s_n$ is a permutation of $1,2,\cdots,n$ such that $f_{s_1}(t)\le f_{s_2}(t)\le\cdots\le f_{s_n}(t)$, and $s_i \le s_j$ whenever $f_{s_i}(t)=f_{s_j}(t)$. Consider any $1\le i<j\le n$, there are two cases:
\begin{enumerate}
    \item If $f_i(\mathbf{x})-f_i(\mathbf{y}) = f_j(\mathbf{x})-f_j(\mathbf{y})$, then $y_i' \equiv y_j'$, $y_i(t)-y_j(t)$ is a constant, so the relative position of $i$ and $j$ in $p(t)$ does not change as $t$ goes from $0$ to 1;
    \item If $f_i(\mathbf{x})-f_i(\mathbf{y}) \ne f_j(\mathbf{x})-f_j(\mathbf{y})$, then as $t$ goes from $0$ to 1, $y_i(t)$ and $y_j(t)$ coincide at most once, so the relative position of $i$ and $j$ in $p(t)$ changes at most once.
\end{enumerate}
Therefore, $p(t)$ changes at most $\frac{n(n-1)}{2}$ times as $t$ goes from $0$ to 1, and the points where $p(t)$ changes split $[0,1]$ into finite intervals. Let $y(t)$ be the median of $y_1(t),\cdots,y_n(t)$, then in each of these intervals: 
\begin{enumerate}
    \item If $n$ is odd, then $y(t)=y_i(t)$ for some fixed $i$, so $|y'(t)|=|y_i'(t)| \le C$;
    \item If $n$ is even, then $y(t)=\frac{y_i(t)}{2}+\frac{y_j(t)}{2}$ for some fixed $i$ and $j$, so $|y'(t)| \le \left|\left(\frac{y_i(t)}{2}\right)'\right|+\left|\left(\frac{y_j(t)}{2}\right)'\right| \le \frac{C}{2}+\frac{C}{2} = C$.
\end{enumerate}
Therefore, $|y'(t)| \le C$ in each of these intervals, then $|f(\mathbf{x})-f(\mathbf{y})| = |y(1)-y(0)| \le C = \nu\|\mathbf{x}-\mathbf{y}\|_X^{\alpha}$.
\end{proof}

The next lemma shows that one dimensional continuous piecewise linear approximators have the same regularity as the function they approximate. See Lemma \ref{lem:CPWLApprox_wRegControl} for a generalization to higher dimensions.

\begin{lemma}[Continuous piecewise linear approximators preserve regularity, and $L^{\infty}$ error estimate]
\label{lem:continuous_piecewise_linear_approximators_preserve_regularity}
    Let $f$ be a function from $[a,b]$ to $\mathbb{R}$ with modulus of regularity $\omega$, and let $\Phi$ be a continuous piecewise linear function on $[a,b]$ that passes through the points $\{(x_i,f(x_i))\}_{i=0}^n$ for $a=x_0<x_1<x_2<\cdots<x_n=b$, and is linear on $[x_0,x_1],[x_1,x_2],\cdots,[x_{n-1},x_n]$. Then, we have the error estimate 
    $$\| f-\Phi \|_{L^{\infty}([a,b])} \le \omega \left(
    \max\limits_{i=0,1,\cdots,n-1}{\frac{|x_{i+1}-x_i|}{2}}
    \right)$$
    Moreover, $\omega$ is a modulus of regularity of $\Phi$. 
\end{lemma}

\begin{proof}

We first show that $\Phi$ has at least the same regularity as $f$. Take any $x,y\in[a,b]$, and assume without loss of generality that $x\le y$. If $x,y$ lie in the same interval in $[x_0,x_1],[x_1,x_2],\cdots,[x_{n-1},x_n]$, then $x,y\in[x_m,x_{m+1}]$ for some index $m$, by concavity of $\omega$, 
\begin{align*}
    \omega(y-x) 
    &= 
    \omega\left(
        \left( 1 - \frac{y-x}{x_{m+1}-x_m} \right) \cdot 0
        +
        \frac{y-x}{x_{m+1}-x_m} \cdot (x_{m+1}-x_m)
    \right)
    \\ 
    &\ge 
    \left( 1 - \frac{y-x}{x_{m+1}-x_m} \right) \omega(0)
    +
    \frac{y-x}{x_{m+1}-x_m} \omega(x_{m+1}-x_m)
    \\
    &\ge 
    \frac{y-x}{x_{m+1}-x_m} \omega(x_{m+1}-x_m)
    \\
    &\ge 
    \frac{y-x}{x_{m+1}-x_m} | f(x_{m+1}) - f(x_m) |
    \\
    &= 
    \frac{y-x}{x_{m+1}-x_m} | \Phi(x_{m+1}) - \Phi(x_m) |
    \\
    &=
    | \Phi(y) - \Phi(x) |
\end{align*}
where the last equality follows from the fact that $\Phi$ is linear on $[x_m,x_{m+1}]$. Otherwise, $x,y$ lie in different intervals in $[x_0,x_1],[x_1,x_2],\cdots,[x_{n-1},x_n]$, then $x_p \le x \le x_{p+1} \le x_q \le y \le x_{q+1}$ for some indices $p,q$, and let 
$$A \eqdef \frac{x-x_p}{x_{p+1}-x_p}, \quad
B \eqdef \frac{x_{q+1}-y}{x_{q+1}-x_q}$$
so $0 \le A,B \le 1$. Without loss of generality, assume that $A\le B$. Then, 
$$\Phi(x) = (1-A) \Phi(x_p) + A \Phi(x_{p+1}), \quad
  \Phi(y) = B \Phi(x_q) + (1-B) \Phi(x_{q+1})$$
Since 
\begin{align*}
    y-x
    &=
    \left(\frac{x_{q+1}-y}{x_{q+1}-x_q} - \frac{x-x_p}{x_{p+1}-x_p}\right)
    (x_q-x_p)
    + \left(1-\frac{x_{q+1}-y}{x_{q+1}-x_q}\right)
    (x_{q+1}-x_p)  \\
    &+ \frac{x-x_p}{x_{p+1}-x_p}(x_q - x_{p+1})  \\
    &= (B-A)(x_q-x_p) + (1-B)(x_{q+1}-x_p) + A(x_q - x_{p+1})
\end{align*}
with $B-A\ge 0, 1-B\ge 0, A\ge 0$ and $(B-A)+(1-B)+A=1$, by concavity of $\omega$, we have 
\begin{align*}
    \omega(y-x) 
    &= \omega( (B-A)(x_q-x_p) + (1-B)(x_{q+1}-x_p) + A(x_q - x_{p+1}) )  \\
    &\ge (B-A) \omega(x_q-x_p)
    + (1-B) \omega (x_{q+1}-x_p)
    + A \omega (x_q - x_{p+1})
    \\
    &\ge (B-A) |\Phi(x_q) - \Phi(x_p)|
    + (1-B) |\Phi(x_{q+1}) - \Phi(x_p)|
    + A |\Phi(x_q) - \Phi(x_{p+1})|
    \\
    &\ge (B-A) (\Phi(x_q) - \Phi(x_p))
    + (1-B) (\Phi(x_{q+1}) - \Phi(x_p))
    + A (\Phi(x_q) - \Phi(x_{p+1}))
    \\
    &= (B \Phi(x_q) + (1-B) \Phi(x_{q+1})) 
    - ((1-A) \Phi(x_p) + A \Phi(x_{p+1}))
    \\
    &= \Phi(y) - \Phi(x)
\end{align*}
Similarly, we can show that $\omega(y-x) \ge \Phi(x) - \Phi(y)$, thus 
$$\omega(|y-x|) = \omega(y-x) \ge |\Phi(y) - \Phi(x)|$$
Therefore, we always have $|\Phi(y) - \Phi(x)| \le \omega(|y-x|)$ for any $x,y\in[a,b]$, thus $\omega$ is a modulus of regularity of $\Phi$.

Now we prove the upper bound for the $L^{\infty}$ error. Take any $x\in[a,b]$, then $x\in[x_k,x_{k+1}]$ for some index $k$. For convenience, let $L \eqdef x_{k+1}-x_k$, then we have 
\allowdisplaybreaks
\begin{align*}
    |f(x) - \Phi(x)|
    &= 
    \left| f(x) - 
        \left(
            \frac{x_{k+1}-x}{L} \Phi(x_k)
            + \frac{x-x_k}{L} \Phi(x_{k+1})
        \right) 
    \right|
    \\
    &=\left| f(x) - \left(
        \frac{x_{k+1}-x}{L} f(x_k)
        + \frac{x-x_k}{L} f(x_{k+1})
    \right) \right|
    \\
    &=
    \left|
    \frac{x_{k+1}-x}{L} (f(x) - f(x_k))
    +
    \frac{x-x_k}{L} (f(x) - f(x_{k+1}))
    \right|
    \\
    &\le 
    \frac{x_{k+1}-x}{L} |f(x) - f(x_k)|
    +
    \frac{x-x_k}{L} |f(x) - f(x_{k+1})|
    \\
    &\le
    \frac{x_{k+1}-x}{L} \omega(x - x_k)
    +
    \frac{x-x_k}{L} \omega(x_{k+1} - x)
    \\
    &\le
    \omega \left(
        \frac{x_{k+1}-x}{L} (x - x_k)
        +
        \frac{x-x_k}{L} (x_{k+1} - x)
    \right)
    \\
    &= 
    \omega \left(
        \frac{(L-(x - x_k))(x - x_k)}{L} 
        +
        \frac{(x - x_k)(L-(x - x_k))}{L} 
    \right)
    \\
    &\le
    \omega \left(
        \frac{\left(\frac{L^2}{4}\right)}{L} 
        +
        \frac{\left(\frac{L^2}{4}\right)}{L} 
    \right)
    \\
    &=
    \omega\left( \frac{L}{2} \right)
    \\
    &\le
    \omega \left(
    \max\limits_{i=0,1,\cdots,n-1}{\frac{|x_{i+1}-x_i|}{2}}
    \right)
\end{align*}
Since $x\in[a,b]$ was chosen arbitrarily, we conclude that 
$$\| f-\Phi \|_{L^{\infty}([a,b])} \le \omega \left(
    \max\limits_{i=0,1,\cdots,n-1}{\frac{|x_{i+1}-x_i|}{2}}
\right).
$$
This completes our proof.
\end{proof}

\subsubsection{One-Dimensional Memorizers with Optimal Regularity}
\label{s:Compute__ss:Memorizers_OptimalRegularity}


\begin{proposition}[Efficient Universal Memorization by Two-Hidden-Layer MLPs with Optimal Regularity]
\label{prop:fit_a_network_with_two_hidden_layers}
Let $M,N\in\mathbb{N}_+$. For any set of $MN$ samples $(x_{i},y_{i})_{i=1}^{MN}\subseteq\mathbb{R}^2$ (where $x_1<x_2<\cdots<x_{MN}$), there exists a ReLU MLP $\Phi$ with $\operatorname{widthvec}=[M,4N-2]$ that can memorize this sample set; i.e. 
\[
    \Phi(x_i)=y_i \qquad \mbox{for } i=1,\dots,MN
.
\]
Furthermore, $\Phi$ is linear on the intervals $[x_i,x_{i+1}]$ for $i=1,2,\cdots,MN-1$, and it is constant on each of the segment $(-\infty,x_1]$ and $[x_{MN},\infty)$. The number of nonzero parameters in $\Phi$ is at most $2MN+2M+8N-4$. 
\end{proposition}

\begin{proof}

Let $f:\mathbb{R}\rightarrow\mathbb{R}$ be the function to be implemented, i.e.\ it is the unique function on $\mathbb{R}$ such that:
\begin{enumerate}
    \item $f(x_i)=y_i$ for $i=1,2,\cdots,MN$;
    \item $f$ is linear on $[x_i,x_{i+1}]$ for $i=1,2,\cdots,MN-1$;
    \item $f$ is constant on each of $(-\infty,x_1]$ and $[x_{MN},\infty)$.
\end{enumerate}
Then, it suffices to prove that $f$ can be implemented by a ReLU MLP with $\operatorname{widthvec}=[M,4N-2]$.

Let's relabel the samples as $(x_{ij},y_{ij})_{i,j=1}^{M,N}$ such that $x_{ij}$ increases as $(i,j)$ increases lexicographically, i.e. 
\[
    x_{11} < x_{12} < \cdots < x_{1N}
    < 
    x_{21} < x_{22} < \cdots < x_{2N}
    < \cdots < 
    x_{M1} < x_{M2} < \cdots < x_{MN}
\]

For notational convenience, denote the $M+1$ intervals $(-\infty,x_{11}], [x_{11},x_{21}], [x_{21},x_{31}],\cdots,[x_{M-1,1},x_{M,1}]$ and $[x_{M1}, \infty)$ as $I_0,I_1,\cdots,I_M$, respectively. The notations $x_{ij}$ and $x_{i,j}$ have the same meaning. 

The idea of the proof is matching the jumps in the derivative of $f$ step by step (see Figure \ref{fig:2hidden_layer}): for $r=1,2$ and $j=2,3,\cdots,N$, let $S_r \eqdef \{ I_{r+2k}: k\in\mathbb{N}, 1 \le r+2k \le M \} = \{I_r,I_{r+2},I_{r+4},\cdots\}$, and construct the continuous piecewise linear function $g_j^{(+r)}$ in two steps: 
\begin{enumerate}
    \item \label{cond:derivative_jump_up} \textbf{Matching upward derivative jumps at the $j$th sample point in intervals in $S_r$}: let
    \[
        X_j^{(+r)} 
        \eqdef 
        \{
            x_{r+2k,j}:
            k\in\mathbb{N},
            f'_+(x_{r+2k,j}) - f'_-(x_{r+2k,j}) > 0
        \}
    \]
    be the set of the $j$th sample points in intervals in $S_r$ where $f$ has upward derivative jumps. Suppose the elements in $X_j^{(+r)}$ are $x_{i_1,j} < x_{i_2,j} < \cdots < x_{i_m,j}$, then for each $1\le m'\le m$, let $g_j^{(+r)}$ be linear on $I_{i_{m'}}$ with slope $(-1)^{m'} \big( f'_+(x_{i_{m'},j}) - f'_-(x_{i_{m'},j}) \big)$, and $g_j^{(+r)}(x_{i_{m'}}) = 0$;
    \item \label{cond:smooth_extension} \textbf{Smooth extension to the rest of $\mathbb{R}$}: we have already defined $g_j^{(+r)}$ on some distinct closed intervals in Step \ref{cond:derivative_jump_up} (unless $X_j^{(+r)}$ is empty, in which case simply let $g_j^{(+r)}$ be the zero function on $\mathbb{R}$), which separate $\mathbb{R}$ into some segments. For each of these segments, if it is infinite then we have already defined $g_j^{(+r)}$ at its only endpoint, then let $g_j^{(+r)}$ be constant on this segment matching its value at this endpoint; if it is finite then we have already defined $g_j^{(+r)}$ at both of its endpoints, then let $g_j^{(+r)}$ be linear on this segment such that it matches both of its values at the two endpoints.
\end{enumerate}

\begin{figure}[!htbp]
    \centering
    \includegraphics[width=1.0\linewidth]{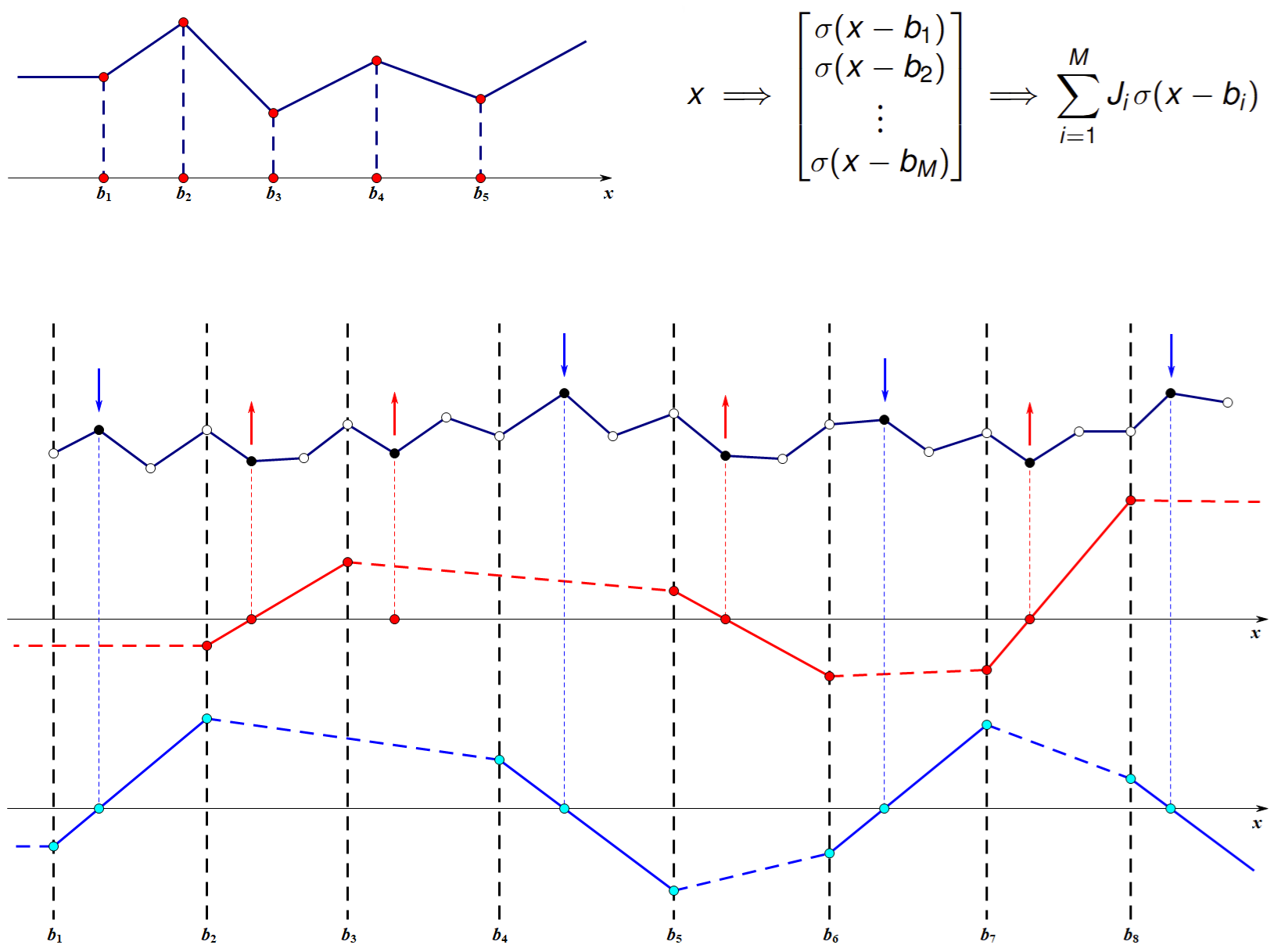}
    \caption{
    The idea of matching derivative jumps using a 2-hidden-layer ReLU MLP. As illustrated in the upper figure, using a single hidden layer, one can express any piecewise linear function whose break points occur in $b_1<\cdots<b_M$ (the biases of the first hidden layer) and which is constant on $(-\infty,b_1]$. Then, in the lower figure, we construct $4N-2$ of these functions to match the derivative jumps of the target function $f$ (the broken line function on the top, where red upward-pointing arrows indicate upward derivative jumps, and blue downward-pointing arrows indicate downward derivative jumps) in the second hidden layer. Here we take $M=8, N=3$ as an example, and two such functions (the red and blue broken lines) are shown, respectively matching half of upward derivative jumps (here we call them $X_1,\cdots,X_p$ for convenience) and half of downward derivative jumps of $f$ at $x_{12},x_{22},\cdots,x_{M2}$. Take the red one for example, it is constructed in two steps: in Step 1, we construct linear pieces (indicated by the solid line segments) in the intervals $[x_{i,1},x_{i+1,1}]$ which respectively contain $X_1,\cdots,X_p$, the absolute values of their slopes are the magnitudes of the respective derivative jumps, with the signs of these slopes alternates along the $x$-axis, and they cross the $x$-axis at $X_1,\cdots,X_p$, respectively; in Step 2, we connect these pieces linearly (via the dashed line segments), and extend horizontally to $\pm\infty$. This red broken line would be the output of a single neuron in the second hidden layer. Upon taking the ReLU, it creates upward derivative jumps exactly at $X_1,\cdots,X_p$ with magnitudes matching that of the target function, without introducing other derivative jumps. Likewise, the blue broken line is constructed similarly to match downward derivative jumps. In the end, we add up the ReLUs of all such broken line functions to match all derivative jumps of $f$, then add another suitable linear function to recover $f$ entirely.
    }
    \label{fig:2hidden_layer}
\end{figure}

We define $X_j^{(+r)}$ and $g_j^{(-r)}$ similar to the above, except that the condition $f'_+(x_{r+2k,j}) - f'_-(x_{r+2k,j}) > 0$ in Step \ref{cond:derivative_jump_up} is changed to $f'_+(x_{r+2k,j}) - f'_-(x_{r+2k,j}) < 0$. It is easy to see that all break points of $g_j^{(\pm r)}$ belong to $\{x_{11},x_{21},\cdots,x_{M1}\}$, and it does not change sign on the intervals $\{I_0,I_1,\cdots,I_M\} \backslash S_r$. Thus, all break points of $\sigma(g_j^{(\pm r)})$ belong to $\{x_{11},x_{21},\cdots,x_{M1}\} \cup \{x_{r,j},x_{r+2,j},x_{r+4,j},\cdots\}$. Moreover, for any $x\in\{x_{r,j},x_{r+2,j},x_{r+4,j},\cdots\}$, if $f'$ has an upward jump at $x$ then $\sigma(g_j^{(+r)})$ also has an upward jump in the derivative with the same magnitude, if $f'$ has an downward jump at $x$ then $\sigma(g_j^{(-r)})$ has an upward jump in the derivative with the same magnitude, thus $\sigma(g_j^{(+r)}) - \sigma(g_j^{(-r)})$ has the same jump in derivatives at $x$ as $f$. Therefore, the function 
\[
    \sum_{\substack{1\le r\le 2 \\ 2\le j\le N}} 
    \left(\sigma(g_j^{(+r)}) - \sigma(g_j^{(-r)})\right)
\]
and $f$ has the same jumps in derivatives at $\{x_{ij}: 1\le i\le M, 2\le j\le N \}$, thus all break points of the function 
\[
    g_0 \eqdef f - 
    \sum_{\substack{1\le r\le 2 \\ 2\le j\le N}} 
    \sigma(g_j^{(+r)})
    +
    \sum_{\substack{1\le r\le 2 \\ 2\le j\le N}} 
    \sigma(g_j^{(-r)})
\]
belong to $\{x_{11},x_{21},\cdots,x_{M1}\}$.

We have 
\begin{align}
\label{decomposition_of_approximant_function}
    f
    &= 
    \sum_{\substack{1\le r\le 2 \\ 2\le j\le N}} 
    \sigma(g_j^{(+r)})
    -
    \sum_{\substack{1\le r\le 2 \\ 2\le j\le N}} 
    \sigma(g_j^{(-r)})
    + g_0  \nonumber  \\
    &= 
    \sum_{\substack{1\le r\le 2 \\ 2\le j\le N}} 
    \sigma(g_j^{(+r)})
    -
    \sum_{\substack{1\le r\le 2 \\ 2\le j\le N}} 
    \sigma(g_j^{(-r)})
    + \sigma(g_0) - \sigma(-g_0)
\end{align}

We now show that $f$ can be represented by a ReLU MLP with $\operatorname{widthvec}=[M,4N-2]$. For notational convenience, we will let $X_i=x_{i1}$ for $i=1,\cdots,M$ in the following.

We know from above that for any function $g\in \{g_{j}^{(\pm r)}: 1\le r\le 2, 2\le j\le N\}\cup\{\pm g_0\}$, the set of break points of $g$ is contained in $\{X_1,X_2,\cdots,X_M\}$, and $g$ is constant on $(-\infty,X_1]$. Then, by comparing the jumps in derivatives as what we did above, we can see that the function
\[
    g(x) - g(X_1) -  \sum_{i=1}^{M} \left(g'_{+}(X_i)-g'_{-}(X_i)\right)\sigma(x-X_i)
\]
has no break points (thus is linear on all of $\mathbb{R}$) and is $0$ for all $x\le X_1$, hence is $0$ on all of $\mathbb{R}$. That is,
\begin{align}
    g(x) &= g(X_1) +  \sum_{i=1}^{M} \left(g'_{+}(X_i)-g'_{-}(X_i)\right)\sigma(x-X_i)  \nonumber  \\
    &= \begin{pmatrix} g'_{+}(X_1)-g'_{-}(X_1) \\ g'_{+}(X_2)-g'_{-}(X_2) \\ \vdots \\ g'_{+}(X_M)-g'_{-}(X_M) \end{pmatrix}
    \cdot \sigma\left(x \begin{pmatrix} 1 \\ 1 \\ \vdots \\ 1 \end{pmatrix} - 
    \begin{pmatrix} X_1 \\ X_2 \\ \vdots \\ X_M \end{pmatrix}\right) + g(X_1),
    \quad
    \forall x \in \mathbb{R} \label{eq:second_layer}
\end{align}
Therefore, $g$ can be represented by a neural network with one hidden layer consisting of $M$ neurons, with a single input neuron and a single output neuron. Moreover, the weight matrix between the input and the hidden layer is always $\begin{bmatrix} 1 & 1 & \cdots & 1 \end{bmatrix}^{\top}$, and the bias vector for the hidden layer is always $\begin{bmatrix} -X_1 & -X_2 & \cdots & -X_M \end{bmatrix}^{\top}$, which is the same for all such $g$. Therefore, all these $g$'s can share their input layers and hidden layers, so we can stack their output neurons ($4N-2$ in total, one for each $g$) in the second hidden layer of the final network such that their outputs are unaffected by each other, then finally $f$ can be implemented by adding one more neuron in the output layer based on Equation \ref{decomposition_of_approximant_function}.

Finally, we tally the nonzero parameters in $\Phi$. 
For the $4N-2$ neurons in the second hidden layer, $4N-4$ of them are $g_j^{(\pm r)}$ for $r=1,2$ and $j=2,3,\cdots,N$. From Step \ref{cond:derivative_jump_up}, we can see that $g_j^{(+r)}$ has at most $2\operatorname{card} \big( X_j^{(+r)} \big)$ break points (since Step \ref{cond:smooth_extension} did not introduce additional ones), all of which belong to $\{X_1,X_2,\cdots,X_M\}$. Thus, at most $2\operatorname{card} \big( X_j^{(+r)} \big)$ weights between the first hidden layer and the neuron that outputs $g_j^{(+r)}$ are nonzero. Therefore, the number of nonzero parameters between the two hidden layers is at most 
\begin{align*}
    &\sum_{\substack{1\le r\le 2 \\ 2\le j\le N}} 
    2\operatorname{card} \big( X_j^{(+r)} \big)
    +
    \sum_{\substack{1\le r\le 2 \\ 2\le j\le N}} 
    2\operatorname{card} \big( X_j^{(-r)} \big)
    +
    2M
    \\
    =
    & \ 2 \operatorname{card}
    \big( \big\{
        x_{ij}:
        1\le i\le M,
        2\le j\le N,
        f'_+(x_{r+2k,j}) - f'_-(x_{r+2k,j}) \ne 0
    \big\} \big)
    +
    2M
    \\
    \le
    & \ 2M(N-1) + 2M
    \\
    =
    & \ 2MN
\end{align*}
From Equation \ref{decomposition_of_approximant_function}, the bias for the output is $0$. Therefore, $\Phi$ has at most $2MN+M+4N-2$ nonzero weights and at most $M+4N-2$ nonzero biases. Altogether, $\Phi$ has at most $2MN+2M+8N-4$ nonzero parameters. 
\end{proof}

In Proposition \ref{prop:fit_a_network_with_two_hidden_layers}, if there are $K$ sample points and take 
\[
    M=2\lceil \sqrt{K} \rceil, 
    \quad
    N=\left\lceil \frac{\lceil \sqrt{K} \rceil}{2} \right\rceil
\]
then, upon noticing that
\[
    MN \ge K,
    \quad
    4N-2 
    = 
    4
    \left\lceil
        \frac{\lceil \sqrt{K} \rceil}{2}
    \right\rceil
    -2
    \le
    4
    \frac{\lceil \sqrt{K} \rceil + 1}{2}
    -2
    =
    2\lceil \sqrt{K} \rceil
\]
we directly deduce the following corollary.

\begin{corollary}
\label{cor:fit_a_network_with_two_hidden_layers}
Let $K\in\mathbb{N}_+$. For any set of $K$ samples $(x_{i},y_{i})_{i=1}^{K}\subseteq\mathbb{R}^2$ (where $x_1<x_2<\cdots<x_{K}$), there exists a ReLU MLP $\Phi$ with width at most $2\lceil \sqrt{K} \rceil$, depth $2$, and with at most $2K+8\lceil \sqrt{K} \rceil$ nonzero parameters that can memorize this sample set; i.e. 
\[
    \Phi(x_i)=y_i \qquad \mbox{for } i=1,\dots,K
.
\]
Furthermore, $\Phi$ is linear on the intervals $[x_i,x_{i+1}]$ for $i=1,2,\cdots,K-1$, and it is constant on each of the segment $(-\infty,x_1]$ and $[x_{K},\infty)$.
\end{corollary}

\begin{theorem}[Efficient Universal Memorization by Deep MLPs with Optimal Regularity]
\label{prop:UnivMemorization_1D}
\hfill\\
Let $K,L\in \mathbb{N}_+$ with $L\ge 3$.  
Given any set of $K$ samples $(x_i,y_i)_{i=1}^{K}\subseteq\mathbb{R}^2$ (with $x_1<x_2<\cdots<x_K$) and any $11< n_1,n_2,\cdots,n_L\in\mathbb{N}_+$ satisfying the constraint:
    \begin{equation}
    \label{eq:num_of_sample_pts}
        K 
        \le 
        \sum_{b=1}^{L-1}
        \left(
            (n_b-11) 
            \left\lfloor \frac{n_{b+1}-9}{4} \right\rfloor
            -
            2
        \right)
    \end{equation}
    there exists a ReLU MLP $\Phi$ with widthvec $[n_1,n_2,\cdots,n_L,8]$ that can memorize this sample set; i.e.\
    \[
        \Phi(x_i)=y_i \qquad \mbox{for } i=1,\dots,K
    .
    \]
    Furthermore, $\Phi$ is linear on the intervals $[x_i,x_{i+1}]$ for $i=1,2,\cdots,K-1$, and it is constant on each of the segment $(-\infty,x_1]$ and $[x_K,\infty)$. The number of nonzero parameters in $\Phi$ is at most
    \[
        2K +  23 \sum_{b=1}^{L} n_b - 121L .
    \]
\end{theorem}

We note that, the factors $23$ and $121$ above can be further improved. For simplicity, here we only prove this weaker bound below, which already shows the significance of this theorem: as the minimum width $n_{\operatorname{min}}$ (i.e. minimum number of neurons in a single hidden layer) goes to infinity, the efficiency of parameter usage
\[
    \operatorname{Efficiency} 
    \eqdef
    \frac
    {\text{degree of freedom of sample set}}
    {\text{number of nonzero parameters}}
    \cdot
    100\%
\]
goes to $100\%$. This is because, assuming equality is achieved in Equation \ref{eq:num_of_sample_pts},
then we have
\begin{align*}
    \operatorname{Efficiency}  
    &\ge
    \frac
    {K\operatorname{dim(\mathbb{R}^2)}}
    {2K +  23 \sum_{b=1}^{L} n_b - 121L}
    \cdot 100\%
    \\
    &\ge
    \frac{2K}{2K + 23 \sum_{b=1}^{L} n_b}
    \cdot 100\%
    \\
    &\ge
    \frac{2K}
    {
        2K + 23 (K+2(L-1)) 
        / 
        \left\lfloor \frac{n_{\operatorname{min}}-9}{4} \right\rfloor
    }
    \cdot 100\%
    \\
    &\ge
    \frac{2}
    {
        2 + 23 (1+2(L-1)) 
        / 
        \left\lfloor \frac{n_{\operatorname{min}}-9}{4} \right\rfloor
    }
    \cdot 100\%
    \\
    &\rightarrow
    100\%
\end{align*}
when the depth $L$ is kept fixed (or $L=o(n_{\operatorname{min}})$) as $n_{\operatorname{min}} \rightarrow \infty$. The same is true for Proposition \ref{prop:fit_a_network_with_two_hidden_layers}.

\begin{proof}[{Proof of Theorem~\ref{prop:UnivMemorization_1D}}]
Let $f:\mathbb{R}\rightarrow\mathbb{R}$ be the function to be implemented, i.e.\ it is the unique function on $\mathbb{R}$ such that:
\begin{enumerate}
    \item $f(x_i)=y_i$ for $i=1,2,\cdots,K$;
    \item $f$ is linear on $[x_i,x_{i+1}]$ for $i=1,2,\cdots,K-1$;
    \item $f$ is constant on each of $(-\infty,x_1]$ and $[x_K,\infty)$.
\end{enumerate}
Then, it suffices to prove that $f$ can be implemented by a ReLU MLP with $\operatorname{widthvec}=[n_1,n_2,\cdots,n_L,8]$.

Clearly $x_1,\cdots,x_K$ contain all possible break points of $f$, so we may assume that all of them are break points of $f$, i.e. the left and right hand derivatives of $f$ at any of the $x_i$'s are different.

The idea of the proof is the same as in the proof of Proposition \ref{prop:fit_a_network_with_two_hidden_layers}: instead of implementing the target function $f$ directly, we match the jumps in its derivative, which is mostly done by adding up $L-1$ functions $f_1,\cdots,f_{L-1}$ (to be defined later), each matching those jumps in different intervals. 

We divide the break points $\{x_i\}_{i=1}^K$ of $f$ into $L-1$ batches: for a batch number $b=1,2,\cdots,L-1$, let $k_b \eqdef n_b - 11$ (and $k_L \eqdef n_L - 11$), and denote the $b$th batch of break points to be 
\[
    x_1^{(b)},x_2^{(b)},\cdots,x_{K_b}^{(b)}, \quad
    K_b \eqdef (n_b-11) 
        \left\lfloor \frac{n_{b+1}-9}{4} \right\rfloor - 2
\]
and let $X_B,Y_B,Z_B\in\mathbb{R}$ for $B=1,2,\cdots,L$ be ``separation points'' between different batches of break points such that
\begin{align*}
    X_1 < Y_1 < Z_1 
    &< x_1^{(1)}<\cdots<x_{K_1}^{(1)}
    < X_2 < Y_2 < Z_2 
    \\
    &< x_2^{(2)}<\cdots<x_{K_2}^{(2)}
    < X_3 < Y_3 < Z_3 
    \\
    &< \cdots  
    \\
    &< x_{L-1}^{(2)}<\cdots<x_{K_{L-1}}^{(2)}
    < X_L < Y_L < Z_L 
\end{align*}
Let $f_b:\mathbb{R} \rightarrow \mathbb{R}$ be a function that satisfies the following conditions:
\begin{enumerate}
    \item \label{cond:fb_has_few_breakpoints} \textbf{Continuous piecewise linear on $\mathbb{R}$ with few ``pieces''}: all break points of $f_b$ belong to 
    \[
        \big\{ x_1^{(b)},x_2^{(b)},\cdots,x_{K_b}^{(b)}, Z_b, X_{b+1} \big\}
    ;
    \]
    \item \label{cond:match_jumps_at_bth_batch} \textbf{Matching derivative jumps at the $b$th batch of break points}: 
    \[
        (f_b)'_+ (x) - (f_b)'_- (x)
        =
        f'_+ (x) - f'_- (x),
        \quad
        \text{for } x = x_1^{(b)},x_2^{(b)},\cdots,x_{K_b}^{(b)}
    \]
    \item \label{cond:vanish_at_other_batches} \textbf{Vanishing at other batches of break points}: 
    \[
        f_b(x) = 0,
        \quad
        \forall x\in (-\infty, Z_b] \cup [X_{b+1}, \infty)
    \]
\end{enumerate}
Then, adding up all the $f_b$'s would recover most informations of $f$. 

We first need to show the existence of these $f_b$'s: let $t_b$ be the linear function on $\mathbb{R}$ such that $t_b(Z_b) = f(Z_b)$ and $t_b(X_{b+1}) = f(X_{b+1})$, then it is easy to verify that the function $f_b$ defined by 
\[
    f_b(x) \eqdef 
    \begin{cases}
        f(x) - t_b(x), \quad &\text{if } x \in [Z_b, X_{b+1}]  \\
        0,             \quad &\text{otherwise}
    \end{cases}
\]
satisfies Conditions \ref{cond:fb_has_few_breakpoints} to \ref{cond:vanish_at_other_batches}.

Now, for each $f_b$, we build a network $\Phi_b$ implementing $f_b$ using the construction in the proof of Proposition \ref{prop:fit_a_network_with_two_hidden_layers}: for $b=1,2,\cdots,L-1$, take $M=k_b,N= \lfloor (k_{b+1}+2)/4 \rfloor$ and sample set $\big\{\big(x_i^{(b)}, f_b(x_i^{(b)})\big)\big\}_{i=1}^{MN-2}$ (namely the $b$th batch of samples) along with the two endpoints $(Z_b,0)$ and $(X_{b+1},0)$ in Proposition \ref{prop:fit_a_network_with_two_hidden_layers}, and then denote the structure of the resulting network as
\[
    x
    \ \Longrightarrow \ 
    \mathbf{h}_1^{(b)}(x) 
    \ \Longrightarrow \ 
    \mathbf{h}_2^{(b)}(x) 
    \ \Longrightarrow \ 
    \Phi_b(x)
\]
where $\mathbf{h}_1^{(b)}(x),\mathbf{h}_2^{(b)}(x)$ are column vectors representing the two hidden layers, and $\mathbf{h}_1^{(b)},\mathbf{h}_2^{(b)}$ are considered as functions of $x\in\mathbb{R}$. By Proposition \ref{prop:fit_a_network_with_two_hidden_layers}, the length of $\mathbf{h}_1^{(b)}(x)$ is $M=k_b$, and the length of $\mathbf{h}_2^{(b)}(x)$ is $4N-2 \le k_{b+1}$, which we treat as exactly $k_{b+1}$ for simplicity. Moreover, we know that 
\[
    \mathbf{h}_1^{(b)}(x) = 
    \begin{bmatrix}
        \sigma\big(x - B_1^{(b)}\big)  \\
        \vdots  \\
        \sigma\big(x - B_{k_b}^{(b)}\big)
    \end{bmatrix},
    \quad \forall x \in \mathbb{R}
\]

Fix a batch number $1\le b\le L-2$. Define $I_b \eqdef [Z_b,X_{b+1}]$, as $\Phi_b$ only matches the jumps in $f'$ inside this interval. Next, we modify $\Phi_b$ so that its second hidden layer $\mathbf{h}_2^{(b)}$ can also serve as the first hidden layer in $\Phi_{b+1}$ on the next interval $I_{b+1}$. There are $k_{b+1}$ functions in $\mathbf{h}_2^{(b)}$, we denote them as $\sigma\big(g_1^{(b)}\big), \sigma\big(g_2^{(b)}\big), \cdots, \sigma\big(g_{k_{b+1}}^{(b)}\big)$, i.e.
\[
    \mathbf{h}_2^{(b)} = 
    \begin{bmatrix}
        \sigma\big(g_1^{(b)}\big)  \\
        \vdots  \\
        \sigma\big(g_{k_{b+1}}^{(b)}\big)
    \end{bmatrix}
\]
Take any $g_j^{(b)}$ with $1\le j\le k_{b+1}$, which is an affine transformation of $\mathbf{h}_1^{(b)}$, and it is linear on $[X_{b+1},\infty)$. Let 
\[
    \tilde{g}_j^{(b)} (x) \eqdef g_j^{(b)} (x) 
    + 
    p_j^{(b)} \sigma\big(x - X_{b+1}\big) 
    + 
    q_j^{(b)} \sigma\big(x - Y_{b+1}\big),
    \quad
    \forall x\in\mathbb{R}
\]
for suitable choices of $p_j^{(b)},q_j^{(b)}\in\mathbb{R}$ such that the following conditions hold (see Figure \ref{fig:1D_batch_tip} for illustration): 
\begin{enumerate}
    \item $\tilde{g}_j^{(b)} \big( Y_{b+1} \big) \ne 0$;
    \item $\tilde{g}_j^{(b)} ( B_j^{(b+1)} ) = 0$;
    \item $\tilde{g}_j^{(b)} (X_{b+1}) \tilde{g}_j^{(b)} (Y_{b+1}) \ge 0$, so that $\sigma\big(\tilde{g}_j^{(b)}\big)$ doesn't have any additional break points between $X_{b+1}$ and $B_j^{(b+1)}$.
\end{enumerate}

\begin{figure}[!htbp]
    \centering
    \includegraphics[width=1.0\linewidth]{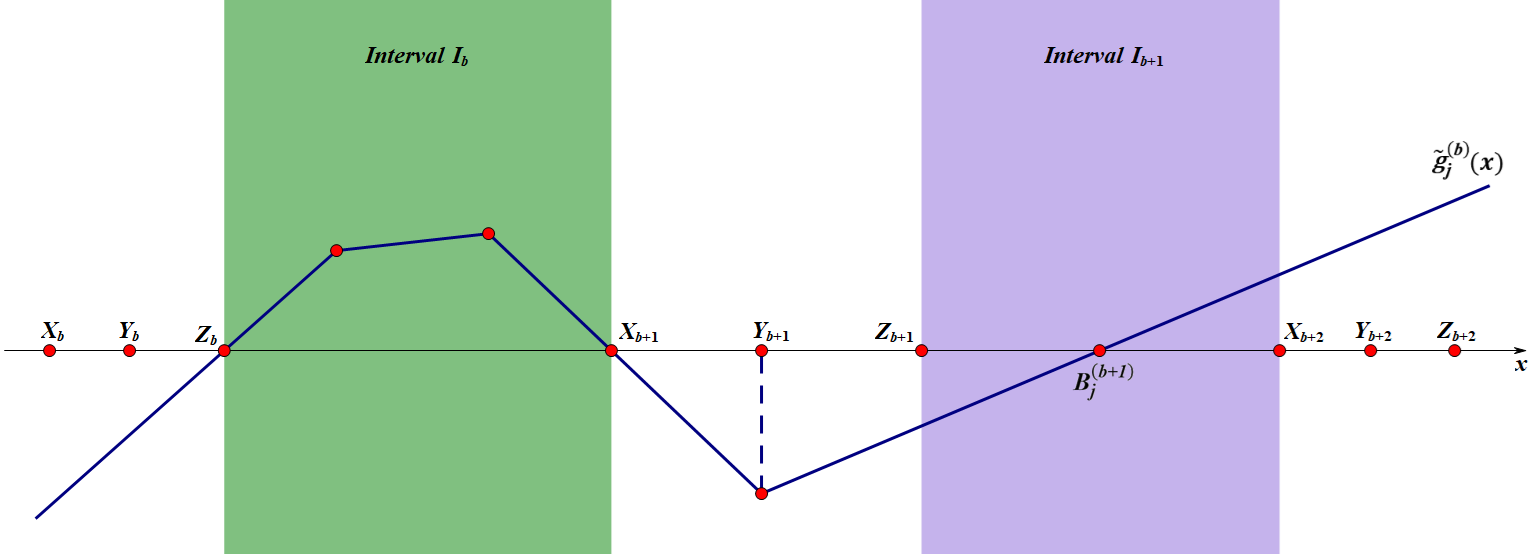}
    \caption{Modification of $g_j^{(b)}$ into $\tilde{g}_j^{(b)}$. This is done by adding $p_j^{(b)} \sigma\big(x - X_{b+1}\big) 
    + 
    q_j^{(b)} \sigma\big(x - Y_{b+1}\big)$ to $g_j^{(b)}(x)$ for suitable choices of the real numbers $p_j^{(b)}$ and $q_j^{(b)}$. In this example, $p_j^{(b)}=0$ and $q_j^{(b)}>0$. The idea is to readjust its slope on the intervals $[X_{b+1},Y_{b+1}]$ and $[Y_{b+1},\infty)$ so that after applying the ReLU, it creates a single nonzero derivative jump at $B_j^{(b+1)}$ on $[X_{b+1},\infty)$, just like the function $\sigma(x-B_j^{(b+1)})$ does, so that it can be used in the same way as the latter, which was used in the first hidden layer of the two-hidden-layer network $\Phi_{b+1}$ for fitting the $(b+1)$-th batch. That is, we ``merged'' the second hidden layer of $\Phi_b$ and the first hidden layer of $\Phi_{b+1}$, which reduces depth and parameter usage.
    }
    \label{fig:1D_batch_tip}
\end{figure}

Since $B_j^{(b+1)} \ge Z_{b+1} > Y_{b+1}$, there exist such choices for $p_j^{(b)},q_j^{(b)}\in\mathbb{R}$. Also, we note that the $\tilde{g}_j^{(b)}$ is a modification of the $g_j^{(b)}$ outside $I_b$, i.e. 
\[
    \tilde{g}_j^{(b)} (x)
    =
    g_j^{(b)} (x),
    \quad \forall x \in I_b
\]
thus we define the following modification layer: 
\begin{equation}
\label{eq:def_of_h2bar}
    \Bar{\mathbf{h}}_2^{(b)} 
    \eqdef 
    \begin{bmatrix}
        \sigma \big( \tilde{g}_1^{(b)} \big)  \\
        \vdots  \\
        \sigma \big( \tilde{g}_{k_{b+1}}^{(b)} \big)
    \end{bmatrix}
\end{equation}
and we have that 
\[
    \Bar{\mathbf{h}}_2^{(b)} (x) 
    =
    \mathbf{h}_2^{(b)} (x) ,
    \quad \forall x \in I_b
\]
Let $\mathcal{L}_{\mathbf{h}_2^{(b)},\Phi_b}$ be the affine transformation that maps the second hidden layer $\mathbf{h}_2^{(b)}$ of $\Phi_b$ to its output. Then, the function
\[
    \tilde{\Phi}_b 
    \eqdef
    \mathcal{L}_{\mathbf{h}_2^{(b)},\Phi_b} 
    (\Bar{\mathbf{h}}_2^{(b)})
\]
``mimics'' the output of $\Phi_b$ on $I_b$, in the sense that
\[
    \tilde{\Phi}_b (x)
    = 
    \mathcal{L}_{\mathbf{h}_2^{(b)},\Phi_b} 
    (\Bar{\mathbf{h}}_2^{(b)} (x))
    =
    \mathcal{L}_{\mathbf{h}_2^{(b)},\Phi_b} 
    (\mathbf{h}_2^{(b)} (x))
    =
    \Phi_b(x),
    \quad \forall x \in I_b
\]
Here, we emphasize that $\tilde{\Phi}_b$ is treated as a function, not a network. 

Now, $\tilde{g}_j^{(b)}$ is linear on $[Z_{b+1},\infty)$ and crosses the $x$-axis at $B_j^{(b+1)}$ with slope $s_j^{(b)}\ne 0$, so we have 
\[
\sigma \big( \tilde{g}_j^{(b)} (x) \big)  
= 
\begin{cases}
    s_j^{(b)} \sigma\big(x - B_j^{(b+1)}\big), \quad
    &\text{if } s_j^{(b)}>0  
    \\
    s_j^{(b)}x - s_j^{(b)} \sigma\big(x - B_j^{(b+1)}\big), \quad
    &\text{if } s_j^{(b)}<0
\end{cases},
\quad \forall x\in [Z_{b+1},\infty)
\]
Correspondingly, we define affine transformations 
\[
\mathcal{L}_j^{(b)} (x,y) = 
\begin{cases}
    y/s_j^{(b)}, \quad
    &\text{if } s_j^{(b)}>0  
    \\
    -y/s_j^{(b)} + x, \quad
    &\text{if } s_j^{(b)}<0
\end{cases},
\quad \forall x,y\in \mathbb{R}
\]
which recover $\sigma(x - B_j)$ on $[Z_{b+1},\infty)$ upon composing with $\sigma \big( \tilde{g}_j^{(b)} \big)$, i.e. 
\[
\mathcal{L}_j^{(b)} \Big(x, \sigma \big( \tilde{g}_j^{(b)}(x) \big) \Big) 
= 
\sigma\big(x - B_j^{(b+1)}\big),
\quad \forall x\in [Z_{b+1},\infty)
\]
We will see that, in the final construction of $\Phi$, we will keep a copy of the input $x$ in every hidden layer. Now, we have the following ``mimic layer'' having exactly the same behavior as $\mathbf{h}_1^{(b+1)}$ on $I_{b+1} \subseteq [Z_{b+1},\infty)$: 
\begin{equation}
\label{eq:def_of_h1bar}
    \tilde{\mathbf{h}}_1^{(b+1)} (x)
    \eqdef 
    \begin{bmatrix}
        \mathcal{L}_1^{(b)} 
        \Big(x, \sigma \big( \tilde{g}_1^{(b)}(x) \big) \Big)  \\
        \vdots  \\
        \mathcal{L}_{k_{b+1}}^{(b)} 
        \Big(x, \sigma \big( \tilde{g}_{k_{b+1}}^{(b)}(x) \big) \Big)  \\
    \end{bmatrix},
    \quad \forall x \in \mathbb{R}
\end{equation}
In other words, 
\[
    \tilde{\mathbf{h}}_1^{(b+1)} (x) 
    =
    \mathbf{h}_1^{(b+1)} (x) ,
    \quad \forall x \in I_{b+1}
\]
Notice that, from Equations \ref{eq:def_of_h1bar} and \ref{eq:def_of_h2bar}, we can see that the function $\tilde{\mathbf{h}}_1^{(b+1)} (x)$ is simply an affine transformation of $\begin{bmatrix} \Bar{\mathbf{h}}_2^{(b)} (x) & x \end{bmatrix}^{\top}$, and thus does not need an extra layer to process. 

So far, we have tweaked the hidden layers $\mathbf{h}_2^{(b)}$ and $\mathbf{h}_1^{(b+1)}$ into $\Bar{\mathbf{h}}_2^{(b)}$ and $\tilde{\mathbf{h}}_1^{(b+1)}$ respectively, so that $\tilde{\mathbf{h}}_1^{(b+1)}$ does not need an extra layer to process. At the same time, we have a collection of the resulting tweaked versions of $\Phi_1,\cdots,\Phi_{L-1}$, namely $\tilde{\Phi}_1,\cdots,\tilde{\Phi}_{L-1}$, with $\tilde{\Phi}_b$ having the same output as $\Phi_b$ on $I_b$, for all batch number $b$. However, these $\tilde{\Phi}_b$'s cannot be used directly, as they have other derivative jumps outside $I_b$, which we now eliminate by post-processing the data $\tilde{\Phi}_b (x)$ together with $x$. For each $b=1,2,\cdots,L-1$, let $w_b>0$ be large enough so that the function 
\[
    q_b(x) \eqdef \tilde{\Phi}_b (x) + w_b x,
    \quad \forall x\in\mathbb{R}
\]
is strictly increasing on $\mathbb{R}$. Since $\tilde{\Phi}_b$ is continuous piecewise linear with finite number of ``pieces'', its Lipschitz constant is finite, then $w_b$ only need to be larger than this constant, so such choice of $w_b$ exists. Define the 1D map projecting $\mathbb{R}$ to $q_b(I_b) = [q_b(Z_b), q_b(X_{b+1})]$: 
\[
    p_b(y)
    \eqdef 
    \sigma(y - q_b(Z_b)) 
    - 
    \sigma(y - q_b(X_{b+1}))
    +
    q_b(Z_b)
    =
    \begin{cases}
        q_b(Z_b),     \quad &\text{if } y < q_b(Z_b)  \\
        y,            \quad &\text{if } q_b(Z_b) \le y \le q_b(X_{b+1})  \\
        q_b(X_{b+1}), \quad &\text{if } y > q_b(X_{b+1})  \\
    \end{cases},
    \quad \forall y \in \mathbb{R}
\]
and then let 
\[
    Q_b(x)
    \eqdef
    p_b ( q_b(x) ) - w_b x,
    \quad \forall x \in \mathbb{R}
\]
We claim that $Q_b(x) = \tilde{\Phi}_b (x)$ for all $x\in I_b$, and $Q_b$ is linear on each of the two segments of $\mathbb{R}\backslash I_b$. Indeed, if $x\in I_b = [Z_b, X_{b+1}]$, then $q_b(Z_b) \le q_b(x) \le q_b(X_{b+1})$ as $q_b$ is strictly increasing, thus $Q_b(x) = p_b (q_b(x)) - w_b x = q_b(x) - w_b x = \tilde{\Phi}_b (x)$. If $x < Z_b$, then $q_b(x) < q_b(Z_b)$, thus $Q_b(x) = p_b (q_b(x)) - w_b x = q_b(Z_b) - w_b x$, which is linear. If $x > X_{b+1}$, then $q_b(x) > q_b(X_{b+1})$, thus $Q_b(x) = p_b (q_b(x)) - w_b x = q_b(X_{b+1}) - w_b x$, which is linear.

Therefore, $Q_b = \tilde{\Phi}_b = \Phi_b$ on $I_b$, which has break points $\big\{x_i^{(b)}\big\}_{i=1}^{MN-2}$ along with $Z_b, X_{b+1}$. Moreover, $Q_b$ has the same derivative jumps at $\big\{x_i^{(b)}\big\}_{i=1}^{MN-2}$ as the target function $f$. Thus, the function $f - \sum_{b=1}^{L-1} Q_b$ only has derivative jumps at $X_1,\cdots,X_{L}$ and $Z_1,\cdots,Z_{L}$. Then, by choosing appropriate real coefficients $u_b$'s and $v_b$'s, the function
\[
    f(x) - 
    \sum_{b=1}^{L-1} Q_b(x)
    -
    \sum_{b=1}^{L} u_b \sigma(x-X_b)
    -
    \sum_{b=1}^{L} v_b \sigma(x-Z_b)
\]
will have no break points on $\mathbb{R}$, which makes it a linear function, say $ux+v$ for some $u,v\in\mathbb{R}$. Then, 
\begin{equation}
\label{eq:decomp_of_target_func}
    f(x)
    =
    \sum_{b=1}^{L-1} Q_b(x)
    +
    \sum_{b=1}^{L} u_b \sigma(x-X_b)
    +
    \sum_{b=1}^{L} v_b \sigma(x-Z_b)
    +
    ux
    +
    v,
    \quad \forall x \in \mathbb{R}
\end{equation}

We can now implement $f$ based on Equation \ref{eq:decomp_of_target_func}. 
We will keep a copy of the original input $x$ via the identity $\sigma(x) - \sigma(-x) = x$ ($\forall x\in\mathbb{R}$) in every hidden layer of the final network $\Phi$, but for notational convenience we will omit the two neurons $\sigma(\pm x)$ in the following representations. Also, we define 
\[
    UV_{A,B}(x) 
    \eqdef 
    \sum_{b=A}^{B} u_b \sigma(x-X_b)
    +
    \sum_{b=A}^{B} v_b \sigma(x-Z_b),
    \quad \forall x \in \mathbb{R}
\]
to simplify notations later.

Now, we start by constructing the first few layers of $\Phi$ computing $\sum_{b=1}^{L-1} Q_b(x)$ as follows: 
\allowdisplaybreaks
\begin{align}
\label{eq:phi_layer_detail}
    x 
    \ &\Longrightarrow \ 
    \begin{bmatrix}
        \mathbf{h}_1^{(1)} (x)  \\  
        \sigma(x - X_{2})  \\  
        \sigma(x - Y_{2})  \\
        \sigma(x - Z_{2})
    \end{bmatrix}
    \ \Longrightarrow \ 
    {\color{persiangreen}
    \begin{bmatrix}
        g_1^{(1)} (x)  \\  
        \vdots  \\  
        g_{k_2}^{(1)} (x)  \\
        \sigma(x - X_{2})  \\  
        \sigma(x - Y_{2})  \\
        \sigma(x - Z_{2})
    \end{bmatrix}
    }
    \ \Longrightarrow \ 
    {\color{persiangreen}
    \begin{bmatrix}
        \tilde{g}_1^{(1)} (x)  \\  
        \vdots  \\  
        \tilde{g}_{k_2}^{(1)} (x)  \\
        UV_{2,2}(x)
    \end{bmatrix}
    }
    \ \Longrightarrow \
    \begin{bmatrix}
        \Bar{\mathbf{h}}_2^{(1)} (x)  \\
        \sigma(x - X_{3})  \\  
        \sigma(x - Y_{3})  \\
        \sigma(x - Z_{3})  \\
        \sigma( \pm UV_{2,2}(x) )
    \end{bmatrix}
    \nonumber
    \\
    \ &\Longrightarrow \ 
    {\color{persiangreen}
    \begin{bmatrix} 
        \tilde{\mathbf{h}}_1^{(2)}  \\
        \tilde{\Phi}_1(x)  \\  
        \sigma(x - X_{3})  \\  
        \sigma(x - Y_{3})  \\
        \sigma(x - Z_{3})  \\
        UV_{2,2}(x)
    \end{bmatrix}
    }
    \ \Longrightarrow \ 
    {\color{persiangreen}
    \begin{bmatrix}
        g_1^{(2)} (x)  \\  
        \vdots  \\  
        g_{k_3}^{(2)} (x)  \\
        q_1(x)  \\  
        \sigma(x - X_{3})  \\  
        \sigma(x - Y_{3})  \\
        \sigma(x - Z_{3})  \\
        UV_{2,2}(x)
    \end{bmatrix}
    }
    \ \Longrightarrow \ 
    {\color{persiangreen}
    \begin{bmatrix}
        \tilde{g}_1^{(2)} (x)  \\  
        \vdots  \\  
        \tilde{g}_{k_3}^{(2)} (x)  \\
        q_1(x) - q_1(Z_1)  \\
        q_1(x) - q_1(X_2)  \\
        UV_{2,3}(x)
    \end{bmatrix}
    }
    \ \Longrightarrow \
    \begin{bmatrix}
        \Bar{\mathbf{h}}_2^{(2)} (x)  \\
        \sigma(x - X_{4})  \\  
        \sigma(x - Y_{4})  \\
        \sigma(x - Z_{4})  \\
        \sigma( q_1(x) - q_1(Z_1) )  \\
        \sigma( q_1(x) - q_1(X_2) )  \\
        \sigma( \pm UV_{2,3}(x) ) 
    \end{bmatrix}
    \nonumber
    \\
    \ &\Longrightarrow \ 
    {\color{persiangreen}
    \begin{bmatrix} 
        \tilde{\mathbf{h}}_1^{(3)}  \\
        \tilde{\Phi}_2(x)  \\  
        \sigma(x - X_{4})  \\  
        \sigma(x - Y_{4})  \\
        \sigma(x - Z_{4})  \\
        Q_1(x)  \\
        UV_{2,3}(x)
    \end{bmatrix}
    }
    \ \Longrightarrow \ 
    {\color{persiangreen}
    \begin{bmatrix}
        g_1^{(3)} (x)  \\  
        \vdots  \\  
        g_{k_4}^{(3)} (x)  \\
        q_2(x)  \\  
        \sigma(x - X_{4})  \\  
        \sigma(x - Y_{4})  \\
        \sigma(x - Z_{4})  \\
        Q_1(x)  \\
        UV_{2,3}(x)
    \end{bmatrix}
    }
    \ \Longrightarrow \ 
    {\color{persiangreen}
    \begin{bmatrix}
        \tilde{g}_1^{(3)} (x)  \\  
        \vdots  \\  
        \tilde{g}_{k_4}^{(3)} (x)  \\
        q_2(x) - q_2(Z_2)  \\
        q_2(x) - q_2(X_3)  \\
        Q_1(x)  \\
        UV_{2,4}(x)
    \end{bmatrix}
    }
    \ \Longrightarrow \
    \begin{bmatrix}
        \Bar{\mathbf{h}}_2^{(3)} (x)  \\
        \sigma(x - X_{5})  \\  
        \sigma(x - Y_{5})  \\
        \sigma(x - Z_{5})  \\
        \sigma( q_2(x) - q_2(Z_2) )  \\
        \sigma( q_2(x) - q_2(X_3) )  \\
        \sigma( \pm Q_1(x) )  \\ 
        \sigma( \pm UV_{2,4}(x) ) 
    \end{bmatrix}
\end{align}
As always, auxiliary layers are colored in green to indicate that we can ignore these layers when estimating the size of the final network, as they are affine transformations that do not use any activation functions, thus each of them can be integrated with the layer after it and does not require an extra layer to process. Removing the auxiliary layers and then extending this pattern, we obtain the final construction of the network $\Phi$: 
\begin{align*}
    x 
    \ &\Longrightarrow \ 
    \begin{bmatrix}
        \mathbf{h}_1^{(1)} (x)  \\  
        \sigma(x - X_{2})  \\  
        \sigma(x - Y_{2})  \\
        \sigma(x - Z_{2})  \\
    \end{bmatrix}
    \ \Longrightarrow \ 
    \begin{bmatrix}
        \Bar{\mathbf{h}}_2^{(1)} (x)  \\
        \sigma(x - X_{3})  \\  
        \sigma(x - Y_{3})  \\
        \sigma(x - Z_{3})  \\
        \sigma ( \pm UV_{2,2}(x) )
    \end{bmatrix}
    \ \Longrightarrow \ 
    \begin{bmatrix}
        \Bar{\mathbf{h}}_2^{(2)} (x)  \\
        \sigma(x - X_{4})  \\  
        \sigma(x - Y_{4})  \\
        \sigma(x - Z_{4})  \\
        \sigma( q_1(x) - q_1(Z_1) )  \\
        \sigma( q_1(x) - q_1(X_2) )  \\
        \sigma( \pm UV_{2,3}(x) ) 
    \end{bmatrix}
    \\
    \ &\Longrightarrow \ 
    \begin{bmatrix}
        \Bar{\mathbf{h}}_2^{(3)} (x)  \\
        \sigma(x - X_{5})  \\  
        \sigma(x - Y_{5})  \\
        \sigma(x - Z_{5})  \\
        \sigma( q_2(x) - q_2(Z_2) )  \\
        \sigma( q_2(x) - q_2(X_3) )  \\
        \sigma( \pm Q_1(x) )  \\ 
        \sigma( \pm UV_{2,4}(x) ) 
    \end{bmatrix}
    \ \Longrightarrow \
    \begin{bmatrix}
        \Bar{\mathbf{h}}_2^{(4)} (x)  \\
        \sigma(x - X_{6})  \\  
        \sigma(x - Y_{6})  \\
        \sigma(x - Z_{6})  \\
        \sigma( q_3(x) - q_3(Z_3) )  \\
        \sigma( q_3(x) - q_3(X_4) )  \\
        \sigma( \pm (Q_1(x)+Q_2(x)) )  \\ 
        \sigma( \pm UV_{2,5}(x) ) 
    \end{bmatrix}
    \ \Longrightarrow \
    \cdots
    \\
    \ &\Longrightarrow \
    \begin{bmatrix}
        \Bar{\mathbf{h}}_2^{(L-2)} (x)  \\
        \sigma(x - X_{L})  \\  
        \sigma(x - Y_{L})  \\
        \sigma(x - Z_{L})  \\
        \sigma( q_{L-3}(x) - q_{L-3}(Z_{L-3}) )  \\
        \sigma( q_{L-3}(x) - q_{L-3}(X_{L-2}) )  \\
        \sigma\big( \pm \sum_{b=1}^{L-4} Q_b(x) \big)  \\ 
        \sigma( \pm UV_{2,L-1}(x) ) 
    \end{bmatrix}
    \ \Longrightarrow \
    \begin{bmatrix}
        \Bar{\mathbf{h}}_2^{(L-1)} (x)  \\
        \sigma(x - X_{1})  \\  
        \sigma(x - Y_{1})  \\
        \sigma(x - Z_{1})  \\
        \sigma( q_{L-2}(x) - q_{L-2}(Z_{L-2}) )  \\
        \sigma( q_{L-2}(x) - q_{L-2}(X_{L-1}) )  \\
        \sigma\big( \pm \sum_{b=1}^{L-3} Q_b(x) \big)  \\
        \sigma ( \pm UV_{2,L}(x) )
    \end{bmatrix}
    \\ 
    \ &\Longrightarrow \
    \begin{bmatrix}
        \sigma( q_{L-1}(x) - q_{L-1}(Z_{L-1}) )  \\
        \sigma( q_{L-1}(x) - q_{L-1}(X_{L}) )  \\
        \sigma\big( \pm \sum_{b=1}^{L-2} Q_b(x) \big)  \\
        \sigma ( \pm UV_{1,L}(x) )
    \end{bmatrix}
    \\
    \ &\Longrightarrow \
    {\color{persiangreen}
    \begin{bmatrix}
        \sum_{b=1}^{L-1} Q_b(x)  \\
        UV_{1,L}(x)  \\
        ux+v
    \end{bmatrix}
    }
    \ \Longrightarrow \
    \sum_{b=1}^{L-1} Q_b(x) 
    + 
    UV_{1,L}(x)
    +
    ux+v
    \ = \ 
    f(x)
\end{align*}
Together with the omitted two neurons $\sigma(\pm x)$, each hidden layer in $\Phi$ has at most $11$ additional neurons other than those of the $\mathbf{h}$'s. Therefore, it has $\operatorname{widthvec}$ at most $[k_1+11, k_2+11, \cdots, k_L+11, 8] = [n_1,n_2,\cdots,n_L,8]$. 

Finally, we tally the nonzero parameters in $\Phi$. From the detailed structure of $\Phi$ shown in Formula \ref{eq:phi_layer_detail}, we can see that most connections (weights) between pairs of hidden layers are the same as their counterparts in the network $\Phi_b$: the neurons in $\mathbf{h}_1^{(1)}$ and $\Bar{\mathbf{h}}_2^{(1)}$ are connected in the same way as the neurons in $\mathbf{h}_1^{(1)}$ and $\mathbf{h}_2^{(1)}$ are connected; for $b=1,2,\cdots,L-1$, the neurons in $\Bar{\mathbf{h}}_2^{(b)}$ and $\Bar{\mathbf{h}}_2^{(b+1)}$ are connected in the same way as the neurons in $\mathbf{h}_1^{(b+1)}$ and $\mathbf{h}_2^{(b+1)}$ are connected. From the proof of Proposition \ref{prop:fit_a_network_with_two_hidden_layers}, we know that there are at most $2 k_b \lfloor (k_{b+1}+2)/4 \rfloor$ nonzero weights between $\mathbf{h}_1^{(b)}$ and $\mathbf{h}_2^{(b)}$. Therefore, assuming the worst case that each hidden layer has $11$ additional neurons other than those in the $\mathbf{h}$'s, we conclude that the number of nonzero weights in $\Phi$ is at most 
\begin{align*}
    W
    &\eqdef
    \sum_{b=1}^{L-1}
    \left(
        2 k_b \left\lfloor \frac{k_{b+1}+2}{4} \right\rfloor
        +
        11 k_b
        +
        11 k_{b+1}
        +
        11 \cdot 11
    \right)
    +
    k_1
    +
    8(k_L+11)
    +
    8
    \\
    &\le
    2K
    +
    22 \sum_{b=1}^{L} k_b
    +
    121(L-1)
    +
    88+8
    \\
    &\le
    2K
    +
    22 \sum_{b=1}^{L} (n_b - 11)
    +
    121L
    -
    25
    \\
    &=
    2K
    +
    22 \sum_{b=1}^{L} n_b 
    -
    121L
    -
    25
\end{align*}
and the number of biases in $\Phi$ is at most 
\[
    B \eqdef 
    \sum_{b=1}^{L} n_b 
    + 8 + 1
\]
Altogether, the number of nonzero parameters in $\Phi$ is at most
\[
    W + B 
    \le 
    2K +  23 \sum_{b=1}^{L} n_b - 121L .
\]
This completes our proof.
\end{proof}

The following example elucidates Proposition~\ref{prop:UnivMemorization_1D}, and frames it in a way which is comparable to the memorization guarantees for MLPs derived in~\cite{vershynin2020memory,vardi2021optimal}.  Though those networks utilize fewer neurons when performing their memorization, they exhibit a much larger Lipschitz constant due to their highly irregular structure; since they are only designed to optimize bit-extraction of~\cite{telgarsky2016benefits,safran2017depth}(see~\cite{baldi2019capacity} for estimates on the maximal bit-extraction ``capacity'' of MLPs).
\begin{example}[A 1D MLP with width $W$ and depth $D$ can memorize $\mathcal{O}(W^2 D)$ data points with optimal Lipschitz constant]
\label{ex:1dMemorization}
Given any set of $K$ samples $(x_i,y_i)_{i=1}^{K}\subseteq\mathbb{R}^2$, where $x_1<x_2<\cdots<x_K$, and
\[
    K = 
    \left(
        (W-11) \left\lfloor \frac{W-9}{4} \right\rfloor - 2 
    \right)
    (D-2)
\]
there exists an MLP with width at most $W$ and depth at most $D$ that can memorize this sample set in such a way that it is linear on $[x_i,x_{i+1}]$ for $i=1,2,\cdots,K-1$, and is constant on each of $(-\infty,x_1]$ and $[x_K,\infty)$. The number of nonzero parameters in $\Phi$ is at most
\[
    2K + (23W-121)D .
\]
\end{example}

\subsection{ReLU Approximations on Polytopes via the Kuhn Triangulation}
\label{s:Compute__ssLKuhn}

In this section, we generalize the concept of continuous piecewise linear approximators and the methods for constructing them to higher dimensions. As we will see, compared to the ones in Theorem \ref{thm:main_SOTA_Benchmark}, these high dimensional continuous piecewise linear approximators require more width, but they can achieve about the same level of global error with much higher regularity (lower Lipschitz constant).

\begin{definition}[Simplex]
    A $d$-dimensional simplex (or $d$-simplex) $S$ is the convex hull of $d+1$ points $\mathbf{v}_1,\cdots,\mathbf{v}_{d+1}$ in $\mathbb{R}^d$ that are affinely independent, i.e. the vectors
    $$\mathbf{v}_{d+1}-\mathbf{v}_1,\mathbf{v}_{d+1}-\mathbf{v}_2,\cdots,\mathbf{v}_{d+1}-\mathbf{v}_d$$
    are linearly independent in $\mathbb{R}^d$, and
    $$S \eqdef 
    \operatorname{conv}(\{\mathbf{v}_1,\cdots,\mathbf{v}_{d+1}\})
    = \left\{
    \sum_{i=1}^{d+1} \lambda_i \mathbf{v}_i :
    \lambda_1,\cdots,\lambda_{d+1} \ge 0, \sum_{i=1}^{d+1} \lambda_i = 1
    \right\}.$$
    Furthermore, $\mathbf{v}_1,\cdots,\mathbf{v}_{d+1}$ are called the vertices of the simplex $S$.
\end{definition}

\begin{definition}[Triangulation of polytopes]
    A triangulation of a polytope $P$ in $\mathbb{R}^d$ is a finite collection of $d$-simplices $\{S_i\}_{i=1}^n$ such that
    \begin{enumerate}
        \item $S_1 \cup S_2 \cup \cdots \cup S_n = P$;
        \item The set of all their vertices is the vertex set of $P$ 
              \footnote{Some literatures allow the existence of additional vertices inside $P$, but here we require that there be no additional vertices other than those of $P$ to simplify later arguments.};
        \item Their interiors are pairwise disjoint.
    \end{enumerate}
\end{definition}

Triangulations are generally not unique. For example, Figure \ref{fig:triangulations_of_hexagon} shows two different triangulations of a hexagon.

\begin{figure}[!htbp]
    \centering
    \begin{subfigure}[b]{0.45\textwidth}
    \centering
        \includegraphics[width=.65\linewidth]{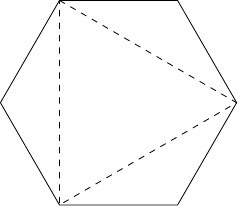}
        \label{fig:triangulations_of_hexagon__a}
    \end{subfigure}
    \hspace{-50pt}
    \begin{subfigure}[b]{0.45\textwidth}
    \centering
        \includegraphics[width=.65\linewidth]{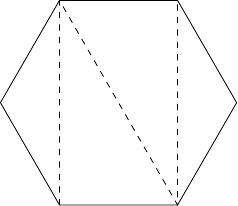}
        \label{fig:triangulations_of_hexagon__b}
    \end{subfigure}
    \caption{Two different triangulations of a hexagon. }
    \label{fig:triangulations_of_hexagon}
\end{figure}

The next lemma shows how to triangulate all unit cubes $\{[i,i+1]: i\in\mathbb{Z}\}^d$ of the entire $\mathbb{R}^{d}$ simultaneously using hyperplanes. The triangulation of each of these unit cubes will be a Kuhn triangulation.

\begin{lemma}
\label{lem:triangulation_of_Rd}
    The set of hyperplanes in $\mathbb{R}^{d}$
    $$H \eqdef \{ x_i \pm x_j = 2k: 1\le i<j \le d, k\in\mathbb{Z} \}
    \cup
    \{ x_i = k: 1\le i \le d, k\in\mathbb{Z} \}$$
    triangulates each of the unit cubes in $\{[i,i+1]: i\in\mathbb{Z}\}^d$.
\end{lemma}

\begin{proof}
Let $\cup H \eqdef \bigcup_{h\in H} h$ be the union of the hyperplanes (when considered as subsets of $\mathbb{R}^d$) in $H$. By the definition of $H$, $\cup H$ has the following two properties:
\begin{enumerate}
    \item \label{property_1} $\cup H$ has period $2$ along each axis:
    $$\mathbf{z} \in \cup H \Longleftrightarrow
      (\mathbf{z} + 2\mathbf{e}_i) \in \cup H, 
      \quad i=1,2,\cdots,d$$
    \item \label{property_2} $\cup H$ is symmetric about each axis:
    $$\mathbf{z} \in \cup H \Longleftrightarrow
      (- \mathbf{z}) \in \cup H,
      \quad i=1,2,\cdots,d$$
\end{enumerate}
By property \ref{property_1}, we only need to show that $\cup H$ triangulates $[-1,1]^d$. By property \ref{property_2}, we only need to show that $\cup H$ triangulates $[0,1]^d$. 

By Definition \ref{defn:Kuhn_triangulation}, the $d!$ simplices
\begin{equation}
    \operatorname{conv}\left(\left\{
    \sum_{i=1}^k \mathbf{e}_{\tau(i)}: 0 \le k \le d
    \right\}\right),
    \quad \tau\in S_d
\end{equation}
form the Kuhn triangulation of $[0,1]^d$. Fix any permutation $\tau\in S_d$. The simplex 
\[
    \operatorname{conv}\left(\left\{
    \sum_{i=1}^k \mathbf{e}_{\tau(i)}: 0 \le k \le d
    \right\}\right)
\]
has $d+1$ faces
$$\operatorname{conv}
\left(
    \left\{
        \sum_{i=1}^k \mathbf{e}_{\tau(i)}: 0 \le k \le d, k \ne j
    \right\}
\right),
\quad j=0,1,\cdots,d.$$
For $j=0$, the corresponding face is contained in the hyperplane $x_{\tau(1)}=1$ in $H$. For $j=d$, the corresponding face is contained in the hyperplane $x_{\tau(d)}=0$ in $H$. For $1\le j\le d-1$, the corresponding face is contained in the hyperplane $x_{\tau(j)} - x_{\tau(j+1)}=0$ in $H$. Therefore, all the faces of all $d!$ simplices are contained in some hyperplanes in $H$. It remains to show that the intersection point of any $d$ hyperplanes in $H$ (if any) is a lattice point.

Let $h_1,h_2,\cdots,h_d$ be any $d$ hyperplanes in $H$ that intersect at a single point $(x_1',\cdots,x_d')$, i.e.
$$h_1 \cap h_2 \cap \cdots \cap h_d = \{ (x_1',\cdots,x_d') \}$$
We will proceed by induction to show that this is a lattice point, and the statement is clearly true for $d=1,2$. Thus, we assume that $d\ge 3$.

If one of these hyperplanes has the form $x_{i_0}=k$, then the remaining $d-1$ hyperplanes, when treated as equations and substituted $x_{i_0}=k$, are $d-1$ equations on $d-1$ variables $\{x_1,\cdots,x_d\}\backslash\{x_{i_0}\}$, which have the form $x_i\pm x_j = 2k$ (if $i\ne i_0\ne j$) or $x_{i_0} \pm x_j = 2k$ and have a unique solution $\{x_1',\cdots,x_d'\}\backslash\{x_{i_0}'\}$ by our assumption above. By induction, $\{x_1',\cdots,x_d'\}\backslash\{x_{i_0}'\}\subset\mathbb{Z}$, and we also have $x_{i_0}'=k\in\mathbb{Z}$, so $(x_1',\cdots,x_d')\in\mathbb{Z}^d$.

If any two of these hyperplanes have the form $x_{i_0}+x_{j_0}=2k$ and $x_{i_0}-x_{j_0}=2l$, then $x_{i_0}'=k+l\in\mathbb{Z}, x_{j_0}'=k-l\in\mathbb{Z}$. Similar to above, the remaining $d-2$ hyperplanes, when treated as equations and substituted $x_{i_0}=k+l$ and $x_{j_0}=k-l$, are $d-2$ equations on $d-2$ variables $\{x_1,\cdots,x_d\}\backslash\{x_{i_0},x_{j_0}\}$, which have the form $x_i\pm x_j = 2k$ (if $i,j\notin\{i_0,j_0\}$) or $x_{i_0} \pm x_j = 2k$ or $x_{j_0} \pm x_j = 2k$ and have a unique solution $\{x_1',\cdots,x_d'\}\backslash\{x_{i_0}',x_{j_0}'\}$ by our assumption above. By induction, $\{x_1',\cdots,x_d'\}\backslash\{x_{i_0}',x_{j_0}'\} \subset\mathbb{Z}$, and we also have $x_{i_0}'=k+l\in\mathbb{Z}, x_{j_0}'=k-l\in\mathbb{Z}$, so $(x_1',\cdots,x_d')\in\mathbb{Z}^d$.

Finally, if $h_1,h_2,\cdots,h_d$ all have the form $x_i\pm x_j=2k$ and they don't contain two hyperplanes of the form $x_i+x_j=2k$ and $x_i-x_j=2k$ for any $1\le i<j\le d$, then consider the (undirected simple) graph 
$G=(\mathcal{V},\mathcal{E})$, where $\mathcal{V}=\{x_1,x_2,\cdots,x_d$\} is the vertex set, and edge set 
$$\mathcal{E}=\{ \{x_i,x_j\}: \text{there is a hyperplane of the form } 
x_i\pm x_j = 2k \}$$
Since $G$ has $d$ vertices and $d$ edges, it has a cycle. Let 
$\{x_{i_1},x_{i_2}\}$, $\{x_{i_2},x_{i_3}\},\cdots,\{x_{i_{m-1}}$, $x_{i_m}\}$, $\{x_{i_m},x_{i_1}\}$ be a shortest cycle, so that $x_{i_1},\cdots,x_{i_m}$ are distinct. Then, $h_1,h_2,\cdots,h_d$ contain the following $m-1$ hyperplanes:
$$x_{i_s} \pm x_{i_{s+1}} = 2k_s, \quad s=1,2,\cdots,m-1$$
By adding or subtracting these equations, we obtain an equation of the form
$$
x_{i_1} \pm x_{i_m} = 2(k_1 \pm k_2 \pm \cdots \pm k_m)
$$
Without loss of generality, assume that the above equation is of the form $x_{i_1} + x_{i_m} = 2k_0$ for some $k_0\in\mathbb{Z}$. If the hyperplane $h_{i_m,i_1}$ corresponding to the edge $\{x_{i_m},x_{i_1}\}$ has the form $x_{i_1} + x_{i_m} = 2l_0$, then $h_1 \cap h_2 \cap \cdots \cap h_d$ is either empty (if $l_0\ne k_0$) or infinite (if $l_0 = k_0$ and the intersection is not empty, since the equation for $h_{i_m,i_1}$ would be redundant and there are essentially only $d-1$ equations for $d$ variables), which contradicts our assumption that $h_1,h_2,\cdots,h_d$ intersect at a single point. Therefore, the hyperplane corresponding to the edge $\{x_{i_m},x_{i_1}\}$ has the form $x_{i_1} - x_{i_m} = 2l_0$, so $x_{i_1}'=k_0+l_0\in\mathbb{Z}$, $x_{i_m}'=k_0-l_0\in\mathbb{Z}$. Proceeding by induction as above, we conclude that $(x_1',\cdots,x_d')\in\mathbb{Z}^d$ is a lattice point.

\end{proof}

\begin{figure}[!htbp]
\centering

\includegraphics[width=0.33\linewidth]{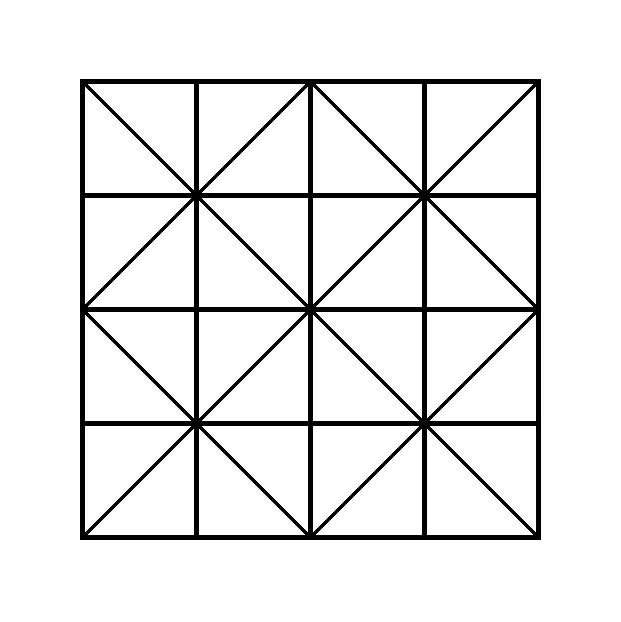}
\hspace{20pt}
\includegraphics[width=0.33\linewidth]{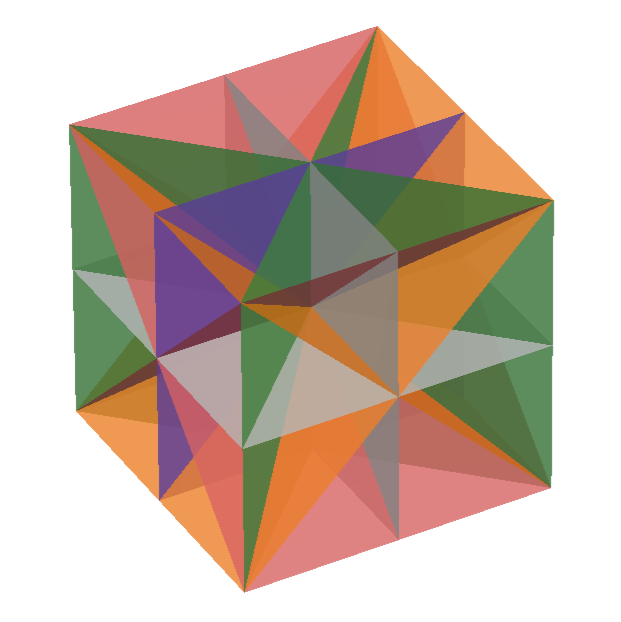}

\caption{Finer Kuhn Triangulation (Definition~\ref{defn:Kuhn_triangulation}) of a $2$-dimensional square (left), and a $3$-dimensional cube (right), than in Figure~\ref{fig:kuhn_local}.}
\label{fig:kuhn_global}
\end{figure}

\begin{definition}[Continuous piecewise linear functions on polytopes]
    Let $f$ be a function from some polytope $P\subseteq\mathbb{R}^{d}$ to $\mathbb{R}$. We say that $f$ is continuous piecewise linear on $P$ if $f$ is continuous on $P$, and there exists a triangulation of $P$ such that $f$ is affine on each of its simplices.
\end{definition}

We first show how to construct continuous piecewise linear approximators on $[0,n]^d$ for some positive integer $n$.

\begin{lemma}
\label{lem:local_affine_implies_global_affine}
    Let $\Omega\subseteq\mathbb{R}^d$ be a connected open set, and $f:\Omega\rightarrow\mathbb{R}$ be a function such that for every point in $\Omega$, it has a neighborhood on which $f$ is affine. Then, $f$ is affine on $\Omega$.
\end{lemma}

\begin{proof}
Fix any $\mathbf{x}\in\Omega$. Then, there exists a neighborhood of $\mathbf{x}$ on which $f\equiv g$, for some affine function $g$ defined on all of $\mathbb{R}^d$.

Assume for contradiction that there is some $\mathbf{y}\in\Omega$ such that $f(\mathbf{y})\ne g(\mathbf{y})$, so $\mathbf{y}\ne \mathbf{x}$. Since $\Omega$ is connected, there exists a path $\gamma$ that connects $\mathbf{x}$ and $\mathbf{y}$ inside $\Omega$, i.e. $\gamma:[0,1]\rightarrow\Omega$ is continuous and $\gamma(0)=\mathbf{x}$, $\gamma(1)=\mathbf{y}$. Let
\[
    T \eqdef \{ t\in[0,1]: f(\gamma(t)) \ne g(\gamma(t)) \},
    \quad
    t_0 \eqdef \inf T
\]
Since $f(\mathbf{y})\ne g(\mathbf{y})$, $f(\gamma(1))\ne g(\gamma(1))$, so $t_0\le 1$ exists. Since $\gamma(0)=\mathbf{x}$ and $f\equiv g$ on some neighborhood of $\mathbf{x}$, $t_0>0$. By assumption, there exists an open neighborhood $U\subseteq\Omega$ of $\gamma(t_0)$ on which $f$ is affine. Since $\gamma(t_0)$ either belongs to $\gamma(T)$ or is a limit point of $\gamma(T)$, and $\gamma(T)\subseteq\mathbb{R}^d\backslash U$ which is closed, $\gamma(t_0)\in\mathbb{R}^d\backslash U$, so $\gamma(t_0)\ne \mathbf{x}$. As $t\rightarrow t_0$, $\gamma(t)\rightarrow\gamma(t_0)$, and since $\gamma(t_0)\ne\mathbf{x}=\gamma(0)$, there are infinite points in $\gamma([0,t_0])$ that are inside $U$. Since both $f$ and $g$ are affine on $U$ and $f\equiv g$ on $\gamma([0,t_0])\cap U$ which is infinite, $f\equiv g$ on $U$. 

However, there exists a sequence $(t_i)_{i=1}^{\infty}$ in $T$ that converges to $t_0$, so $\gamma(t_i)\rightarrow\gamma(t_0)\in U$ as $i\rightarrow\infty$, then there exists some $i_0$ large enough so that $\gamma(t_{i_0})\in U$, but $f(\gamma(t_{i_0})) \ne g(\gamma(t_{i_0}))$, which contradicts our conclusion above.
Therefore, $f\equiv g$ on $\Omega$, thus $f$ is affine on $\Omega$.
\end{proof}

The following lemma provides an explicit formula for the ``hat functions'' for the Kuhn triangulation. Figure~\ref{fig:hat_2d} shows two hat functions for the triangulation in Lemma \ref{lem:triangulation_of_Rd}.

\begin{figure}[!htbp]
    \centering
    \includegraphics[width=0.8\linewidth]{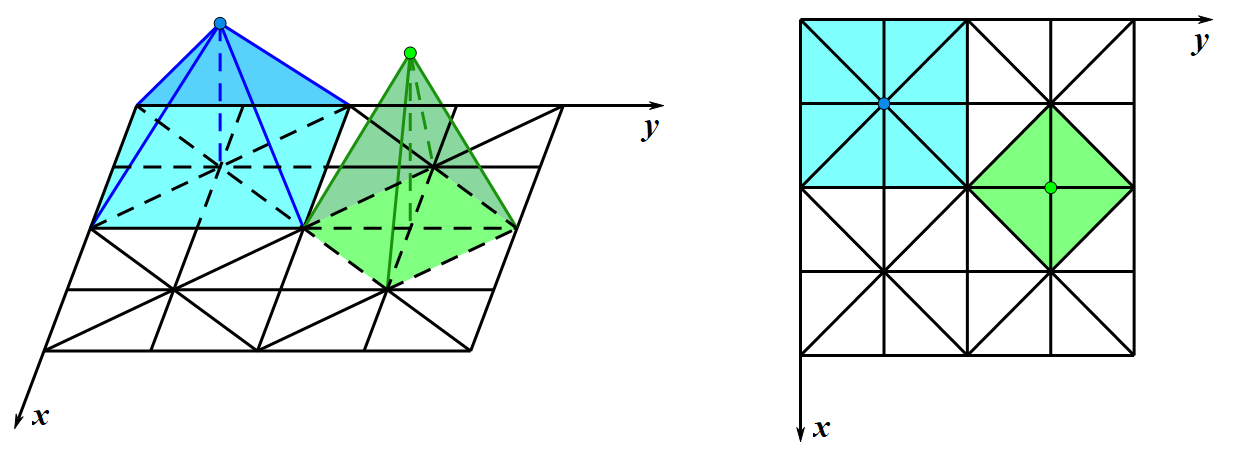}
    \caption{The supports of two hat functions in 3D view (left) and 2d view (right) with respect to the triangulation in Lemma~\ref{lem:triangulation_of_Rd}. Each of them is affine on any of the smallest triangles, so any linear combinations of them is continuous piecewise linear with respect to the triangulation in Lemma~\ref{lem:triangulation_of_Rd}.}
    \label{fig:hat_2d}
\end{figure}

\begin{lemma}[Constructing continuous piecewise linear approximators in $\mathbb{R}^d$]
\label{lem:construct_CPL_approximator}
    Let $f$ be a function from $[0,n]^d$ to $\mathbb{R}$ for some $n\in\mathbb{N}_+$, then the function $\Phi$ defined by
    $$\Phi(\mathbf{x})  \eqdef  
    \sum_{
    \mathbf{y}\in\{0,1,\cdots,n\}^d
    }
    f(\mathbf{y})
    \sigma
    \left(
        1 - 
        \max
        \left\{
            | x_i - y_i | :
            y_i \text{ is even}
        \right\}
        -
        \max
        \left\{
            | x_i - y_i | :
            y_i \text{ is odd}
        \right\}
    \right)
    $$
    for $\mathbf{x} = (x_1,\cdots,x_d) \in [0,n]^d$ is continuous piecewise linear on each of the $n^d$ unit cubes $\{[0,1],[1,2],\cdots,[n-1,n]\}^d$ of $[0,n]^d$ with respect to the triangulation in Lemma \ref{lem:triangulation_of_Rd}, and
    $$\Phi(\mathbf{y}) = f(\mathbf{y}), \quad
    \forall \mathbf{y} \in\{0,1,\cdots,n\}^d$$
\end{lemma}

\begin{proof}
For convenience, we first define
\begin{equation}
\label{eq:hat}
T_{\mathbf{y}}(\mathbf{x})  \eqdef  
\sigma
\left(
    1 - 
    \max
    \left\{
        | x_i - y_i | :
        y_i \text{ is even}
    \right\}
    -
    \max
    \left\{
        | x_i - y_i | :
        y_i \text{ is odd}
    \right\}
\right), \quad
\forall \mathbf{x}\in\mathbb{R}^d
\end{equation}
to be the \emph{hat function}\footnote{This is the same as the hat function (or nodal basis function) discussed in~\cite{goujon2023stable_ACHA}, and here we have provided its formula explicitly when the underlying triangulation is the one in Lemma~\ref{lem:triangulation_of_Rd}.} centred at $\mathbf{y}$. Then, 
$$\Phi(\mathbf{x}) = 
\sum_{
\mathbf{y}\in\{0,1,\cdots,n\}^d
}
f(\mathbf{y})
T_{\mathbf{y}}(\mathbf{x})
$$

For any $\mathbf{y},\mathbf{z}\in\{0,1,\cdots,n\}^d$ with $\mathbf{y} \ne \mathbf{z}$, there is some $i$ such that $|y_i-z_i| \ge 1$, so $T_{\mathbf{z}}(\mathbf{y})=0$, thus
$$\Phi(\mathbf{y}) = 
f(\mathbf{y}) T_{\mathbf{y}}(\mathbf{y})
+
\sum_{\substack{
\mathbf{z} \in \{0,1,\cdots,n\}^d \\
\mathbf{z} \ne \mathbf{y}
}}
f(\mathbf{z})
T_{\mathbf{z}}(\mathbf{y})
=
f(\mathbf{y}) T_{\mathbf{y}}(\mathbf{y})
=
f(\mathbf{y})
$$

Now we show that $\Phi$ is continuous piecewise linear with respect to the triangulations in Lemma \ref{lem:triangulation_of_Rd}. We only need to show that the hat functions $\{T_{\mathbf{y}}\}_{\mathbf{y}\in\{0,1,\cdots,n\}^d}$ are continuous piecewise linear. Since the hat functions are obviously continuous, by Lemma \ref{lem:local_affine_implies_global_affine}, we only need to show that each hat function is locally affine at every point outside the union of those hyperplanes.

Take any $\mathbf{x}\in\mathbb{R}^d$ that is outside the union of these hyperplanes, and fix an arbitrary $\mathbf{y}\in\{0,1,\cdots,n\}^d$. We need to find a neighborhood of $\mathbf{x}$ on which $T_{\mathbf{y}}$ is affine. For notational convenience, let
$$
S_{\mathbf{y}}(\mathbf{z})  \eqdef  
1 - 
\max
\left\{
    | z_i - y_i | :
    y_i \text{ is even}
\right\}
-
\max
\left\{
    | z_i - y_i | :
    y_i \text{ is odd}
\right\}, \quad
\forall \mathbf{z}\in\mathbb{R}^d
$$
If $S_{\mathbf{y}}(\mathbf{x})=0$, then there are some $i,j$ with $i\ne j$ such that
$$1 - |x_i-y_i| - |x_j-y_j| = 0, \quad 
y_i \text{ is even}, y_j \text{ is odd}$$
removing the absolute values results in an equation of the form
$$\pm x_i \pm x_j = \pm y_i \pm y_j + 1$$
where $\pm y_i \pm y_j + 1$ is even, contradicting the assumption that $\mathbf{x}$ is outside those hyperplanes. If $S_{\mathbf{y}}(\mathbf{x}) < 0$, then $T_{\mathbf{y}}\equiv\sigma(S_{\mathbf{y}})\equiv 0$ on some neighborhood of $\mathbf{x}$, and we are done. Therefore, in the following we will assume that $S_{\mathbf{y}}(\mathbf{x})>0$, so there is some neighborhood $U_1$ of $\mathbf{x}$ on which $S_{\mathbf{y}}>0$.

If $| x_i - y_i | = 0$ for some $i$, then $x_i=y_i\in\mathbb{Z}$, contradicting the assumption that $\mathbf{x}$ is outside those hyperplanes. Then, $| x_i - y_i | \ne 0$ for all $i$.

If there are some $i,j$ with $i\ne j$ such that $y_i$ and $y_j$ are both even, and $|x_i-y_i| = |x_j-y_j|$, then removing the absolute values results in an equation of the form $\pm x_i \pm x_j = \pm y_i \pm y_j$ where $\pm y_i \pm y_j$ is even, contradicting the assumption that $\mathbf{x}$ is outside those hyperplanes. Then, there is a unique $i_0$ such that 
$$\max
\left\{
    | x_i - y_i | :
    y_i \text{ is even}
\right\}
= 
| x_{i_0} - y_{i_0} |, \quad
y_{i_0} \text{ is even}$$
thus $|x_i-y_i| < | x_{i_0} - y_{i_0} |$ for all $i\ne i_0$ such that $y_i$ is even. Therefore, there exists a neighborhood $U_2$ of $\mathbf{x}$ such that for any $\mathbf{z}\in U_2$, $z_{i_0}-y_{i_0}$ does not change sign and 
$\max
\left\{
    | z_i - y_i | :
    y_i \text{ is even}
\right\}
= 
| z_{i_0} - y_{i_0} |$
, so the function 
$\max
\left\{
    | z_i - y_i | :
    y_i \text{ is even}
\right\}$ 
of $\mathbf{z}$ is affine on $U_2$. Similarly, there exists a neighborhood $U_3$ of $\mathbf{x}$ such that the function 
$\max
\left\{
    | z_i - y_i | :
    y_i \text{ is odd}
\right\}$ 
of $\mathbf{z}$ is affine on $U_3$.

Now take $U = U_1 \cap U_2 \cap U_3$, so $U$ is also a neighborhood of $\mathbf{x}$. Moreover, $S_{\mathbf{y}}>0$ on $U$, so $T_{\mathbf{y}}\equiv S_{\mathbf{y}}$ on $U$. Since $U\subseteq U_2$ and $U\subseteq U_3$, the functions 
$\max
\left\{
    | z_i - y_i | :
    y_i \text{ is even}
\right\}$ 
and
$\max
\left\{
    | z_i - y_i | :
    y_i \text{ is odd}
\right\}$ 
of $\mathbf{z}$ are affine on $U$, so $S_{\mathbf{y}}$ is affine on $U$, hence $T_{\mathbf{y}}$ is also affine on $U$.
\end{proof}

\begin{lemma}[Continuous piecewise linear approximators with respect to the Kuhn triangulation preserve regularity under $\ell^1$ norm]
\label{lem:CPWLApprox_wRegControl}
    Let $f$ be a function from $[0,n]^d$ to $\mathbb{R}$ for some $n\in\mathbb{N}_+$. Let $\omega:[0,nd]\rightarrow\mathbb{R}$ be a modulus of regularity of $f$ under $\ell^1$ norm of $\mathbb{R}^d$. Let $\Phi:[0,n]^d\rightarrow\mathbb{R}$ be continuous piecewise linear on each of the $n^d$ unit cubes $\{[0,1],[1,2],\cdots,[n-1,n]\}^d$ with respect to the triangulation in Lemma \ref{lem:triangulation_of_Rd}, and
    $$\Phi(\mathbf{y}) = f(\mathbf{y}), \quad
    \forall \mathbf{y} \in\{0,1,\cdots,n\}^d$$
    Then, $\Phi$ satisfies the approximation guarantee
    $$\| f-\Phi \|_{L^{\infty}([0,n]^d)} \le 
    \omega \left( \frac{d}{2} \right)$$
    Furthermore, $\omega$ is a modulus of regularity of $\Phi$ under $\ell^1$ norm of $\mathbb{R}^d$.
\end{lemma}

During the proof of Lemma~\ref{lem:CPWLApprox_wRegControl} we will make the following notational and terminological conventions. 
For any vector $\mathbf{z}\in\mathbb{R}^d$, we denote its $i$-th coordinate by $(\mathbf{z})_i$. The notation $\mathbf{z}_i$ without brackets will have other contextual meanings specified when appropriate. We say that a function $g:\mathbb{R}^d \rightarrow \mathbb{R}$ has $\omega$-regularity at a pair of points $\mathbf{u},\mathbf{v}\in\mathbb{R}^d$ if 
$| g(\mathbf{u}) - g(\mathbf{v}) | \le \omega (\| \mathbf{u} - \mathbf{v} \|_1)$.

\begin{proof}
We turn the problem into a path-finding problem: observe that, for any pair of $\mathbf{x},\mathbf{y}\in [0,n]^d$, if we can find a pair of corresponding finite linear paths $\gamma_{\mathbf{x}},\gamma_{\mathbf{y}}: [0,T] \rightarrow [0,n]^d$ and a time $t_0\in[0,T]$ such that the following conditions hold:
\begin{enumerate}
    \item \label{cond:path_1} $\gamma_{\mathbf{x}}(0),\gamma_{\mathbf{x}}(1)$ are contained in the same simplex;
    \item \label{cond:path_2} $\gamma_{\mathbf{y}}(0),\gamma_{\mathbf{y}}(1)$ are contained in the same simplex;
    \item \label{cond:path_3} $\mathbf{x} = \gamma_{\mathbf{x}}(t_0)$;
    \item \label{cond:path_4} $\mathbf{y} = \gamma_{\mathbf{y}}(t_0)$;
    \item \label{cond:path_5} As $t$ changes linearly from $0$ to $1$, 
    $\| \gamma_{\mathbf{x}}(t) - \gamma_{\mathbf{y}}(t) \| _ 1$
    changes linearly from $\|\gamma_{\mathbf{x}}(0) - \gamma_{\mathbf{y}}(0)\|_1$ to $\|\gamma_{\mathbf{x}}(1) - \gamma_{\mathbf{y}}(1)\|_1$;
    \item \label{cond:path_6} $\Phi$ has $\omega$-regularity at $\gamma_{\mathbf{x}}(0),\gamma_{\mathbf{y}}(0)$ and $\gamma_{\mathbf{x}}(1),\gamma_{\mathbf{y}}(1)$.
\end{enumerate}
then we can conclude that 
$\Phi$ also has $\omega$-regularity at $\mathbf{x},\mathbf{y}$; here, by $\gamma_{\mathbf{x}},\gamma_{\mathbf{y}}$ being linear we mean that each of their coordinates are linear functions of $t$.

This is because, by conditions \ref{cond:path_1} and \ref{cond:path_2}, $\Phi$ is linear on the line segment connecting $\gamma_{\mathbf{x}}(0),\gamma_{\mathbf{x}}(1)$ as well as the line segment connecting $\gamma_{\mathbf{y}}(0),\gamma_{\mathbf{y}}(1)$, so
$$\Phi(\mathbf{x}) = 
\frac{t_0}{T} \Phi(\gamma_{\mathbf{x}}(1)) + 
\frac{T-t_0}{T} \Phi(\gamma_{\mathbf{x}}(0)), 
\quad
\Phi(\mathbf{y}) = 
\frac{t_0}{T} \Phi(\gamma_{\mathbf{y}}(1)) + 
\frac{T-t_0}{T} \Phi(\gamma_{\mathbf{y}}(0)) $$
by conditions \ref{cond:path_3},\ref{cond:path_4} and \ref{cond:path_5}, we have that 
$$\| \mathbf{x} - \mathbf{y} \|_1 = 
\frac{t_0}{T} \|\gamma_{\mathbf{x}}(1) - \gamma_{\mathbf{y}}(1)\|_1
+ 
\frac{T-t_0}{T} \|\gamma_{\mathbf{x}}(0) - \gamma_{\mathbf{y}}(0)\|_1$$
thus
\begin{align*}
    |\Phi(\mathbf{x}) - \Phi(\mathbf{y})| 
    &= \left| \frac{t_0}{T} 
              ( \Phi(\gamma_{\mathbf{x}}(1)) - 
              \Phi(\gamma_{\mathbf{y}}(1)) ) 
              +
              \frac{T-t_0}{T} 
              ( \Phi(\gamma_{\mathbf{x}}(0)) - \Phi(\gamma_{\mathbf{y}}(0)) ) 
        \right|  \\
    &\le \frac{t_0}{T} 
         | \Phi(\gamma_{\mathbf{x}}(1)) - \Phi(\gamma_{\mathbf{y}}(1)) |
         +
         \frac{T-t_0}{T} 
         | \Phi(\gamma_{\mathbf{x}}(0)) - \Phi(\gamma_{\mathbf{y}}(0)) |  \\
    &\le \frac{t_0}{T} 
         \omega(\|\gamma_{\mathbf{x}}(1) - \gamma_{\mathbf{y}}(1)\|_1)
         +
         \frac{T-t_0}{T} 
         \omega(\|\gamma_{\mathbf{x}}(0) - \gamma_{\mathbf{y}}(0)\|_1)  \tag{by condition \ref{cond:path_6}}  \\
    &\le \omega
         \left( \frac{t_0}{T} 
                \|\gamma_{\mathbf{x}}(1) - \gamma_{\mathbf{y}}(1)\|_1
                + 
                \frac{T-t_0}{T} 
                \|\gamma_{\mathbf{x}}(0) - \gamma_{\mathbf{y}}(0)\|_1 
         \right)  
         \tag{since $\omega$ is concave}  \\
    &= \omega(\|\mathbf{x} - \mathbf{y}\|_1)
\end{align*}
In other words, for any pair of $\mathbf{x},\mathbf{y}\in [0,n]^d$, if we can find two finite linear paths such that:
\begin{enumerate}[label=\arabic*$'$]
    \item \label{cond:path_1'} Each path is entirely contained in some simplex;
    \item \label{cond:path_2'} $\omega$-regularity of $\Phi$ is assured at the pair of starting points and the pair of ending points;
    \item \label{cond:path_3'} As two points travel along the two paths, the $\ell^1$ distance between them changes linearly;
    \item \label{cond:path_4'} As two points travel along the two paths, they reach $\mathbf{x},\mathbf{y}$ respectively at the same time,
\end{enumerate}
then the concavity of $\omega$ would guarantee the $\omega$-regularity of $\Phi$ at any intermediate pair of points, including $\mathbf{x},\mathbf{y}$.

Now, we fix any pair of $\mathbf{x},\mathbf{y}\in [0,n]^d$. 
For any $\mathbf{z}\in [0,n]^d$, suppose it is contained in some simplex $P_{\mathbf{z}}$ with vertices $\mathbf{u}_1,\cdots,\mathbf{u}_{d+1}\in \{0,1,\cdots,n\}^d$, then there exists $\lambda_1,\cdots,\lambda_{d+1}\in[0,1]$ such that
$$\mathbf{z} = \lambda_1 \mathbf{u}_1 + \cdots + \lambda_{d+1} \mathbf{u}_{d+1}, \quad
\lambda_1 + \cdots + \lambda_{d+1} = 1$$
We say that $\mathbf{z}$ is contained in the interior of a $k$-dimensional face if exactly $k+1$ of $\lambda_1,\cdots,\lambda_{d+1}$ are nonzero, and define
\[
        d_{\mathbf{z}}  
    \eqdef  
        \min 
        \{ 
            k\in\mathbb{N}_+ : 
            \mathbf{z} \text{ is contained in the interior of a $k$-dimensional face}
        \}
\]
Using the idea above, we will later find some linear paths that ``push'' $\mathbf{x},\mathbf{y}$ to 
some boundary points $\mathbf{x}',\mathbf{y}'$ of some $(d_{\mathbf{x}}-1),(d_{\mathbf{y}}-1)$-dimensional faces which contain $\mathbf{x},\mathbf{y}$ respectively, where we can use induction on $d_{\mathbf{x}} + d_{\mathbf{y}}$
to conclude $\omega$-regularity at $\mathbf{x},\mathbf{y}$.

Without loss of generality, assume $\mathbf{x}\in [0,1]^d$. We first assume that $\mathbf{y}\in [0,1]^d$ and demonstrate how to find such paths for $\mathbf{x}$ and $\mathbf{y}$, then we generalize the result. 
As argued above, let $F_{\mathbf{x}},F_{\mathbf{y}}\subseteq[0,1]^d$ be $(d_{\mathbf{x}}-1),(d_{\mathbf{y}}-1)$-dimensional faces which contain $\mathbf{x},\mathbf{y}$, and $P_{\mathbf{x}},P_{\mathbf{y}}$ be some simplices in $[0,1]^d$ which contain $F_{\mathbf{x}},F_{\mathbf{y}}$, respectively. 

First, we define two infinite linear paths $\gamma_{\mathbf{x}},\gamma_{\mathbf{y}}: [0,\infty) \rightarrow \mathbb{R}^d$ by:
$$\gamma_{\mathbf{x}}(t)  \eqdef  t\mathbf{x}, \gamma_{\mathbf{y}}(t)  \eqdef  t\mathbf{y}, \quad \forall t\in[0,\infty)$$
Note that, $\|\gamma_{\mathbf{x}}(t) - \gamma_{\mathbf{y}}(t)\|_1 = t\|\mathbf{x} - \mathbf{y}\|_1$ changes linearly with time, and both $\mathbf{x}, \mathbf{y}$ are reached at the same time $t=1$, thus conditions \ref{cond:path_3'} and \ref{cond:path_4'} are satisfied. Next, we modify $\gamma_{\mathbf{x}}$ and $\gamma_{\mathbf{y}}$ so that conditions \ref{cond:path_1'} and \ref{cond:path_2'} are satisfied, while ensuring that both conditions \ref{cond:path_3'} and \ref{cond:path_4'} remain satisfied.

From Equation \ref{eq:kuhn_triang_d!_simplices}, let $\tau_{\mathbf{x}},\tau_{\mathbf{y}}\in S_d$ be the permutations corresponding to the simplices $P_{\mathbf{x}},P_{\mathbf{y}}$. 
Let
$\{ \mathbf{v}_1,\cdots,\mathbf{v}_{d_{\mathbf{x}}} \},
\{ \mathbf{w}_1,\cdots,\mathbf{w}_{d_{\mathbf{y}}} \}
\subseteq \{0,1\}^d$
be the vertices of $F_{\mathbf{x}},F_{\mathbf{y}}$. 
Then, $\mathbf{v}_i$'s and $\mathbf{w}_i$'s have the form
$$\{ \mathbf{v}_1,\cdots,\mathbf{v}_{d_{\mathbf{x}}} \} = 
\left\{ 
    \sum_{i=1}^{a_1} \mathbf{e}_{\tau_{\mathbf{x}}(i)},
    \cdots,
    \sum_{i=1}^{a_{d_{\mathbf{x}}}} \mathbf{e}_{\tau_{\mathbf{x}}(i)}
\right\},\quad
0\le a_1 < a_2 < \cdots < a_{d_{\mathbf{x}}} \le d$$
$$\{ \mathbf{w}_1,\cdots,\mathbf{w}_{d_{\mathbf{y}}} \} = 
\left\{ 
    \sum_{i=1}^{b_1} \mathbf{e}_{\tau_{\mathbf{y}}(i)},
    \cdots,
    \sum_{i=1}^{b_{d_{\mathbf{y}}}} \mathbf{e}_{\tau_{\mathbf{y}}(i)}
\right\},\quad
0\le b_1 < b_2 < \cdots < b_{d_{\mathbf{y}}} \le d$$
Define new paths $\gamma_{\mathbf{x}}^{\ast}$ and $\gamma_{\mathbf{y}}^{\ast}$ as follows: let $(\gamma_{\mathbf{x}}^{\ast})_i=(\gamma_{\mathbf{x}})_i$ for all coordinate components for which $i\ne\tau_{\mathbf{x}}(1),\cdots,\tau_{\mathbf{x}}(a_1)$, otherwise let $(\gamma_{\mathbf{x}}^{\ast})_i \equiv 1$; similarly, let $(\gamma_{\mathbf{y}}^{\ast})_i=(\gamma_{\mathbf{y}})_i$ for all coordinate components for which $i\ne\tau_{\mathbf{y}}(1),\cdots,\tau_{\mathbf{y}}(b_1)$, otherwise let $(\gamma_{\mathbf{y}}^{\ast})_i \equiv 1$. Finally, let
$$T  \eqdef  \max \{ t\in[0,\infty): 
               \gamma_{\mathbf{x}}^{\ast}([0,t])\subseteq F_{\mathbf{x}}, 
               \gamma_{\mathbf{y}}^{\ast}([0,t])\subseteq F_{\mathbf{y}} \}$$
For any $i=\tau_{\mathbf{x}}(1),\cdots,\tau_{\mathbf{x}}(a_1)$, since $(\mathbf{v}_1)_i,\cdots,(\mathbf{v}_{d_{\mathbf{x}}})_i = 1$ and $\mathbf{x}$ is contained in the convex hull of $\{ \mathbf{v}_1,\cdots,\mathbf{v}_{d_{\mathbf{x}}} \}$, $(\mathbf{x})_i = 1$, thus $\gamma_{\mathbf{x}}^{\ast}(1) = \gamma_{\mathbf{x}}(1) = \mathbf{x} \in F_{\mathbf{x}}$, and since
$$\gamma_{\mathbf{x}}^{\ast}(0) 
= \sum_{i=1}^{a_1} \mathbf{e}_{\tau_{\mathbf{x}}(i)} 
\in \{ \mathbf{v}_1,\cdots,\mathbf{v}_{d_{\mathbf{x}}} \}
\subseteq F_{\mathbf{x}}, \quad
$$
where $F_{\mathbf{x}}$ is convex, $\gamma_{\mathbf{x}}^{\ast}([0,1]) \subseteq F_{\mathbf{x}}$. Similarly, $\gamma_{\mathbf{y}}^{\ast}([0,1]) \subseteq F_{\mathbf{y}}$. Thus, $T\ge 1$. 
If $T$ is infinite, then both paths $\gamma_{\mathbf{x}}^{\ast},\gamma_{\mathbf{y}}^{\ast}$ are constant, in this case both $\mathbf{x},\mathbf{y}$ must be lattice points, at which $\Phi$ has $\omega$-regularity since the values of $\Phi$ and $f$ are the same at $\mathbf{x},\mathbf{y}$. Therefore, we may assume that $T\ge 1$ is finite.

We now show that the finite linear paths $\gamma_{\mathbf{x}}^{\ast}(t),\gamma_{\mathbf{y}}^{\ast}(t)$ ($t\in [0,T]$) satisfy conditions \ref{cond:path_1'} to \ref{cond:path_4'}. Condition \ref{cond:path_1'} is automatically satisfied by the definition of $T$. We have proved above that $\gamma_{\mathbf{x}}^{\ast}(1) = \mathbf{x}$, and $\gamma_{\mathbf{y}}^{\ast}(1) = \mathbf{y}$ can be proved similarly, also $T\ge 1$, thus condition \ref{cond:path_4'} is satisfied. 

For condition \ref{cond:path_2'}, since
$$\gamma_{\mathbf{x}}^{\ast}(0) 
= \sum_{i=1}^{a_1} \mathbf{e}_{\tau_{\mathbf{x}}(i)}
\in \{0,1\}^d, \quad
\gamma_{\mathbf{y}}^{\ast}(0) 
= \sum_{i=1}^{b_1} \mathbf{e}_{\tau_{\mathbf{y}}(i)}
\in \{0,1\}^d
\quad$$
are lattice points on which the values of $\Phi$ and $f$ are the same, we have $\omega$-regularity of $\Phi$ at $\gamma_{\mathbf{x}}^{\ast}(0)$ and $\gamma_{\mathbf{y}}^{\ast}(0)$.

Since the faces $F_{\mathbf{x}},F_{\mathbf{y}}$ of the convex polytopes $S_{\mathbf{x}},S_{\mathbf{y}}$ are themselves convex polytopes, they are the convex hull of their vertices:
$$F_{\mathbf{x}} = \{
    p_1' \mathbf{v}_1 + 
    \cdots + 
    p_{d_{\mathbf{x}}}' \mathbf{v}_{d_{\mathbf{x}}} :
    p_1', \cdots, p_{d_{\mathbf{x}}}' \ge 0,
    p_1' + \cdots + p_{d_{\mathbf{x}}}' = 1
\}$$
$$F_{\mathbf{y}} = \{
    q_1' \mathbf{w}_1 + 
    \cdots + 
    q_{d_{\mathbf{y}}}' \mathbf{w}_{d_{\mathbf{y}}} :
    q_1', \cdots, q_{d_{\mathbf{y}}}' \ge 0,
    q_1' + \cdots + q_{d_{\mathbf{y}}}' = 1
\}$$
By the definition of $T$, at least one of the ending points $\gamma_{\mathbf{x}}^{\ast}(T),\gamma_{\mathbf{y}}^{\ast}(T)$ lies on the boundary of its face. Thus, without loss of generality, we may assume that $\gamma_{\mathbf{x}}^{\ast}(T)$ lies on the boundary of $F_{\mathbf{x}}$. Then, when $\gamma_{\mathbf{x}}^{\ast}(T)$ is expressed in the form 
$p_1' \mathbf{v}_1 + 
\cdots + 
p_{d_{\mathbf{x}}}' \mathbf{v}_{d_{\mathbf{x}}}$, 
some of the coefficients $p_i'$ will be $0$ (unless $d_{\gamma_{\mathbf{x}}^{\ast}(T)}=1$, which is the base case), so $\gamma_{\mathbf{x}}^{\ast}(T)$ lies in the interior of some $d_{\gamma_{\mathbf{x}}^{\ast}(T)}$-dimensional face with $d_{\gamma_{\mathbf{x}}^{\ast}(T)} < d_{\mathbf{x}}$, while $\gamma_{\mathbf{y}}^{\ast}(T)$ is still contained in the $d_{\mathbf{y}}$-dimensional face $F_{\mathbf{y}}$, so $d_{\gamma_{\mathbf{y}}^{\ast}(T)}\le d_{\mathbf{y}}$. By the induction hypothesis, since $d_{\gamma_{\mathbf{x}}^{\ast}(T)}+d_{\gamma_{\mathbf{y}}^{\ast}(T)} < d_{\mathbf{x}}+d_{\mathbf{y}}$, $\Phi$ has $\omega$-regularity at $\gamma_{\mathbf{x}}^{\ast}(T),\gamma_{\mathbf{y}}^{\ast}(T)$. In the base case where $d_{\gamma_{\mathbf{x}}^{\ast}(T)}=d_{\gamma_{\mathbf{y}}^{\ast}(T)}=1$, $\mathbf{x}$ and $\mathbf{y}$ are lattice points on which the values of $\Phi$ and $f$ are the same, so we have $\omega$-regularity of $\Phi$ at $\mathbf{x}$ and $\mathbf{y}$. Therefore, condition \ref{cond:path_2'} is satisfied. 

Finally, for condition \ref{cond:path_3'}, let $i\in\{1,2,\cdots,d\}$ be arbitrary. As $t$ changes linearly from $0$ to $T$, $(\gamma_{\mathbf{x}}(t))_i - (\gamma_{\mathbf{y}}(t))_i = t((\mathbf{x})_i - (\mathbf{y})_i)$ changes linearly and does not change sign. By definitions of $\gamma_{\mathbf{x}}^{\ast}$ and $\gamma_{\mathbf{y}}^{\ast}$, compared to $\gamma_{\mathbf{x}}$ and $\gamma_{\mathbf{y}}$, $(\gamma_{\mathbf{x}}^{\ast})_i$ and $(\gamma_{\mathbf{y}}^{\ast})_i$ are either the same or always equal to $1$. Then, there are three possibilities:
\begin{enumerate}
    \item If both $(\gamma_{\mathbf{x}}^{\ast}(t))_i \equiv (\gamma_{\mathbf{x}}(t))_i$ and $(\gamma_{\mathbf{y}}^{\ast}(t))_i \equiv (\gamma_{\mathbf{y}}(t))_i$ for all $t\in [0,T]$, then as $t$ changes linearly from $0$ to $T$, $(\gamma_{\mathbf{x}}^{\ast}(t))_i - (\gamma_{\mathbf{y}}^{\ast}(t))_i = (\gamma_{\mathbf{x}}(t))_i - (\gamma_{\mathbf{y}}(t))_i = t((\mathbf{x})_i - (\mathbf{y})_i)$ changes linearly and does not change sign;
    \item If exactly one of $(\gamma_{\mathbf{x}}^{\ast}(t))_i$ and $(\gamma_{\mathbf{y}}^{\ast}(t))_i$ always equal to $1$, without loss of generality assume $(\gamma_{\mathbf{x}}^{\ast}(t))_i \equiv 1$, then $(\gamma_{\mathbf{x}}^{\ast}(t))_i - (\gamma_{\mathbf{y}}^{\ast}(t))_i = 1 - (\gamma_{\mathbf{y}}(t))_i \ge 1 - 1 = 0$ changes linearly with $t$ and does not change sign;
    \item If both of $(\gamma_{\mathbf{x}}^{\ast}(t))_i$ and $(\gamma_{\mathbf{y}}^{\ast}(t))_i$ always equal to $1$, then $(\gamma_{\mathbf{x}}^{\ast}(t))_i - (\gamma_{\mathbf{y}}^{\ast}(t))_i \equiv 0$ changes linearly with $t$ and does not change sign.
\end{enumerate}
Therefore, $(\gamma_{\mathbf{x}}^{\ast}(t))_i - (\gamma_{\mathbf{y}}^{\ast}(t))_i$ always changes linearly with $t$ and does not change sign, thus
$$\| \gamma_{\mathbf{x}}^{\ast}(t) - \gamma_{\mathbf{y}}^{\ast}(t) \|_1 = 
\sum_{i=1}^d |(\gamma_{\mathbf{x}}^{\ast}(t))_i - (\gamma_{\mathbf{y}}^{\ast}(t))_i|$$
changes linearly with $t$, hence condition \ref{cond:path_3'} is satisfied.

Now, we have found suitable paths for $\mathbf{y}\in[0,1]^d$, next we show how to find such paths for an arbitrary $\mathbf{y}'\in[0,n]^d$. 
Suppose $\mathbf{y}'\in Q_{\mathbf{y}'}  \eqdef  [m_1,m_1+1] \times [m_2,m_2+1] \times \cdots \times [m_d,m_d+1]$ for some integers $m_1,\cdots,m_d$. Define the function $T_{\mathbf{y}'} : [0,1]^d \rightarrow Q_{\mathbf{y}'}$ by: 
$$(T_{\mathbf{y}'}(\mathbf{z}))_i  \eqdef  
\begin{cases}
    (\mathbf{z})_i + 2k_i,  \quad 
    \text{if } m_i = 2k_i \text{ for some } k_i \in \mathbb{N}  \\
    2k_i - (\mathbf{z})_i,  \quad 
    \text{if } m_i = 2k_i - 1 \text{ for some } k_i \in \mathbb{N}  \\
\end{cases},
\quad i=1,2,\cdots,d$$
Then $T_{\mathbf{y}'}$ is a composition of reflections and translations, thus it is a linear bijection between $[0,1]^d$ and $Q_{\mathbf{y}'}$. 

Let $\cup H$ be the union of the hyperplanes (which are considered as subsets of $\mathbb{R}^d$) in $H$, which is a subset of $\mathbb{R}^d$ that triangulates every unit cube in $\{[i,i+1]: i\in\mathbb{Z}\}^d$. By the definition of $H$, $\cup H$ has the following two properties:
\begin{enumerate}
    \item Translation invariance:
    $$\mathbf{z} \in \cup H \Longleftrightarrow
      (\mathbf{z} + 2\mathbf{e}_i) \in \cup H, 
      \quad i=1,2,\cdots,d$$
    \item Reflection invariance:
    $$\mathbf{z} \in \cup H \Longleftrightarrow
      (2\mathbf{e}_i - \mathbf{z}) \in \cup H,
      \quad i=1,2,\cdots,d$$
\end{enumerate}
Then, by the definition of $T_{\mathbf{y}'}$, the part of $\cup H$ inside $Q_{\mathbf{y}'}$ is exactly the image of the part of $\cup H$ inside $[0,1]^d$ under $T_{\mathbf{y}'}$, and vice versa: the part of $\cup H$ inside $[0,1]^d$ is exactly the image of the part of $\cup H$ inside $Q_{\mathbf{y}'}$ under $T_{\mathbf{y}'}^{-1}$. Therefore, the simplices in $[0,1]^d$ and the simplices in $Q_{\mathbf{y}'}$ are in one-to-one correspondence via $T_{\mathbf{y}'}$.

Let $\mathbf{y}  \eqdef  T_{\mathbf{y}'}^{-1}(\mathbf{y}')$. Since $\mathbf{x},\mathbf{y} \in [0,1]^d$, by the arguments above, there exist finite linear paths $\gamma_{\mathbf{x}}^{\ast}(t),\gamma_{\mathbf{y}}^{\ast}(t)$ ($t\in[0,T]$) that satisfy conditions \ref{cond:path_1'} to \ref{cond:path_4'}. Using the same notations as above, let $F_{\mathbf{x}},F_{\mathbf{y}}\subseteq[0,1]^d$ be $(d_{\mathbf{x}}-1),(d_{\mathbf{y}}-1)$-dimensional faces which contain $\mathbf{x},\mathbf{y}$, and $P_{\mathbf{x}},P_{\mathbf{y}}$ be some simplices in $[0,1]^d$ which contain $F_{\mathbf{x}},F_{\mathbf{y}}$, respectively.

Define the path for $\mathbf{y}'$ by:
$$\gamma_{\mathbf{y}'}^{\ast}(t)  \eqdef  
T_{\mathbf{y}'} ( \gamma_{\mathbf{y}}^{\ast}(t) ),
\quad t\in[0,T]$$
We now show that the finite linear paths $\gamma_{\mathbf{x}}^{\ast}(t),\gamma_{\mathbf{y}'}^{\ast}(t)$ ($t\in[0,T]$) satisfy condition \ref{cond:path_1'} to \ref{cond:path_4'}. Since $\gamma_{\mathbf{y}}^{\ast}([0,T]) \subseteq P_{\mathbf{y}}$, $\gamma_{\mathbf{y}'}^{\ast}([0,T]) \subseteq T_{\mathbf{y}'}(P_{\mathbf{y}})$ and $T_{\mathbf{y}'}(P_{\mathbf{y}})$ is a simplex, so condition \ref{cond:path_1'} is satisfied. Similar as above, since $\gamma_{\mathbf{x}}^{\ast}(0),\gamma_{\mathbf{y}'}^{\ast}(0) = T_{\mathbf{y}'} ( \gamma_{\mathbf{y}}^{\ast}(0) )$ are lattice points and $\gamma_{\mathbf{x}}^{\ast}(T),\gamma_{\mathbf{y}'}^{\ast}(T) = T_{\mathbf{y}'} ( \gamma_{\mathbf{y}}^{\ast}(T) )$ are in the boundary of their faces $F_{\mathbf{x}},T_{\mathbf{y}'}(F_{\mathbf{y}})$, by doing induction on $d_{\mathbf{x}}+d_{\mathbf{y}'}$, we conclude that $\Phi$ has $\omega$-regularity at $\gamma_{\mathbf{x}}^{\ast}(0),\gamma_{\mathbf{y}'}^{\ast}(0)$ and $\gamma_{\mathbf{x}}^{\ast}(T),\gamma_{\mathbf{y}'}^{\ast}(T)$, so condition \ref{cond:path_2'} is satisfied. Since $\gamma_{\mathbf{x}}^{\ast}(1) = \mathbf{x}$ and $\gamma_{\mathbf{y}'}^{\ast}(1) = T_{\mathbf{y}'} ( \gamma_{\mathbf{y}}^{\ast}(1) ) = T_{\mathbf{y}'}(\mathbf{y}) = \mathbf{y}'$, condition \ref{cond:path_4'} is satisfied.

Finally, for condition \ref{cond:path_3'}, if $(\mathbf{y})_i \in [0,1]$, then $m_i = 0$, $k_i = 0$, $(\gamma_{\mathbf{y}'}^{\ast})_i = 2k_i + (\gamma_{\mathbf{y}}^{\ast})_i = (\gamma_{\mathbf{y}}^{\ast})_i$, and since $(\gamma_{\mathbf{x}}^{\ast}(t))_i - (\gamma_{\mathbf{y}}^{\ast}(t))_i$ always changes linearly with $t$ and does not change sign, so does $(\gamma_{\mathbf{x}}^{\ast}(t))_i - (\gamma_{\mathbf{y}'}^{\ast}(t))_i$. Otherwise, if $(\mathbf{y})_i > 1$, then $m_i \ge 1$, $k_i \ge 1$, $(\gamma_{\mathbf{y}'}^{\ast})_i \ge 2k_i - (\gamma_{\mathbf{y}}^{\ast})_i \ge 2 - 1 = 1 \ge (\gamma_{\mathbf{x}}^{\ast})_i$, so $(\gamma_{\mathbf{x}}^{\ast})_i - (\gamma_{\mathbf{y}'}^{\ast})_i$ does not change sign, and obviously it changes linearly with $t$. Thus, 
$$\| \gamma_{\mathbf{x}}^{\ast}(t) - \gamma_{\mathbf{y}'}^{\ast}(t) \|_1 = 
\sum_{i=1}^d |(\gamma_{\mathbf{x}}^{\ast}(t))_i - (\gamma_{\mathbf{y}'}^{\ast}(t))_i|$$
changes linearly with $t$, hence condition \ref{cond:path_3'} is satisfied.

Therefore, the paths $\gamma_{\mathbf{x}}^{\ast},\gamma_{\mathbf{y}'}^{\ast}$ for $\mathbf{x},\mathbf{y}'$ indeed satisfy conditions \ref{cond:path_1'} to \ref{cond:path_4'}.
By the argument above, $\Phi$ has $\omega$-regularity at $\mathbf{x},\mathbf{y}'$. Since $\mathbf{x},\mathbf{y}'$ were chosen arbitrarily, we conclude that $\omega$ is also a modulus of regularity of $\Phi$ under $\ell^1$ norm of $\mathbb{R}^d$.

For the error estimate, fix any $\mathbf{x}\in[0,n]^d$. Suppose $\mathbf{x}\in Q_{\mathbf{x}}  \eqdef  [n_1,n_1+1] \times [n_2,n_2+1] \times \cdots \times [n_d,n_d+1]$ for some integers $n_1,\cdots,n_d$. Let $\mathbf{u}_1,\cdots,\mathbf{u}_{d+1}$ be the vertices of a simplex $P_{\mathbf{x}}\subseteq Q_{\mathbf{x}}$ which contains $\mathbf{x}$, then there exist $\lambda_1,\cdots,\lambda_{d+1} \ge 0$ such that 
$$\mathbf{x} = \lambda_1 \mathbf{u}_1 + \cdots + \lambda_{d+1} \mathbf{u}_{d+1}, \quad
\lambda_1 + \cdots + \lambda_{d+1} = 1$$
Since $\Phi$ is affine on this simplex,
\[
    \Phi(\mathbf{x}) 
    = 
    \sum_{i=1}^{d+1} \lambda_i \Phi(\mathbf{u}_i)
    =
    \sum_{i=1}^{d+1} \lambda_i f(\mathbf{u}_i)
\]
thus
\allowdisplaybreaks
\begin{align}
\nonumber
    | \Phi(\mathbf{x}) - f(\mathbf{x}) | 
    &= | \lambda_1 (f(\mathbf{u}_1) - f(\mathbf{x}))
       + \cdots
       + \lambda_{d+1} (f(\mathbf{u}_{d+1}) - f(\mathbf{x})) |  \\
\nonumber
    &\le \lambda_1 | f(\mathbf{u}_1) - f(\mathbf{x}) |
         + \cdots
         + \lambda_{d+1} | f(\mathbf{u}_{d+1}) - f(\mathbf{x}) |  \\
\nonumber
    &\le \lambda_1 \omega( \| \mathbf{u}_1 - \mathbf{x} \|_1 )
         + \cdots
         + \lambda_{d+1} \omega( \| \mathbf{u}_{d+1} - \mathbf{x} \|_1 )  \\
\label{eq:cancave}
    &
    \le \omega( \lambda_1 \| \mathbf{u}_1 - \mathbf{x} \|_1
                 + \cdots
                 + \lambda_{d+1} \| \mathbf{u}_{d+1} - \mathbf{x} \|_1 )  
\end{align}
where~\eqref{eq:cancave} follows by the concavity of the modulus $\omega$,
and 
\begin{align*}
    \sum_{j=1}^{d+1}
    \lambda_j \| \mathbf{u}_j - \mathbf{x} \|_1
    &= \sum_{j=1}^{d+1} 
       \lambda_j \sum_{i=1}^d | (\mathbf{u}_j)_i - (\mathbf{x})_i |  \\
    &= \sum_{i=1}^d 
       \sum_{j=1}^{d+1}  \lambda_j | (\mathbf{u}_j)_i - (\mathbf{x})_i |  \\
    &= \sum_{i=1}^d \left(
       \sum_{\substack{j=1 \\ (\mathbf{u}_j)_i = n_i}} ^ {d+1}  
       \lambda_j | (\mathbf{u}_j)_i - (\mathbf{x})_i |
       +
       \sum_{\substack{j=1 \\ (\mathbf{u}_j)_i = n_i + 1}} ^ {d+1}  
       \lambda_j | (\mathbf{u}_j)_i - (\mathbf{x})_i |
       \right)  \\
    &= \sum_{i=1}^d \left(
       \sum_{\substack{j=1 \\ (\mathbf{u}_j)_i = n_i}} ^ {d+1}  
       \lambda_j ((\mathbf{x})_i - n_i)
       +
       \sum_{\substack{j=1 \\ (\mathbf{u}_j)_i = n_i + 1}} ^ {d+1}  
       \lambda_j (n_i + 1 - (\mathbf{x})_i)
       \right)  \\
    &= \sum_{i=1}^d \left(
       ((\mathbf{x})_i - n_i)
       \sum_{\substack{j=1 \\ (\mathbf{u}_j)_i = n_i}} ^ {d+1}  
       \lambda_j 
       +
       (1 - ((\mathbf{x})_i - n_i))
       \sum_{\substack{j=1 \\ (\mathbf{u}_j)_i = n_i + 1}} ^ {d+1}  
       \lambda_j 
       \right)  \\
\end{align*}
for which we have 
\begin{align*}
    (\mathbf{x})_i - n_i 
    &= \left( \sum_{j=1}^{d+1}  \lambda_j (\mathbf{u}_j)_i \right) - n_i  \\
    &= \left( 
       \sum_{\substack{j=1 \\ (\mathbf{u}_j)_i = n_i}} ^ {d+1}  
       \lambda_j n_i 
       +
       \sum_{\substack{j=1 \\ (\mathbf{u}_j)_i = n_i + 1}} ^ {d+1}  
       \lambda_j (n_i + 1)
       \right) - n_i  \\
    &= \sum_{\substack{j=1 \\ (\mathbf{u}_j)_i = n_i + 1}} ^ {d+1}  
       \lambda_j  \\
    &= 1 - \sum_{\substack{j=1 \\ (\mathbf{u}_j)_i = n_i}} ^ {d+1}  
       \lambda_j
\end{align*}
thus
$$((\mathbf{x})_i - n_i)
\sum_{\substack{j=1 \\ (\mathbf{u}_j)_i = n_i}} ^ {d+1}  
\lambda_j
=
\left(
1 - \sum_{\substack{j=1 \\ (\mathbf{u}_j)_i = n_i}} ^ {d+1}  
\lambda_j
\right)
\sum_{\substack{j=1 \\ (\mathbf{u}_j)_i = n_i}} ^ {d+1}  
\lambda_j
\le 
\frac{1}{4}$$
$$(1 - ((\mathbf{x})_i - n_i))
\sum_{\substack{j=1 \\ (\mathbf{u}_j)_i = n_i + 1}} ^ {d+1}  
\lambda_j
=
\left(
1 - \sum_{\substack{j=1 \\ (\mathbf{u}_j)_i = n_i + 1}} ^ {d+1}  
\lambda_j
\right)
\sum_{\substack{j=1 \\ (\mathbf{u}_j)_i = n_i + 1}} ^ {d+1}  
\lambda_j
\le 
\frac{1}{4}$$
Combining the results above, we have that
\begin{align*}
    \sum_{j=1}^{d+1}
    \lambda_j \| \mathbf{u}_j - \mathbf{x} \|_1 
    &= \sum_{i=1}^d \left(
       ((\mathbf{x})_i - n_i)
       \sum_{\substack{j=1 \\ (\mathbf{u}_j)_i = n_i}} ^ {d+1}  
       \lambda_j 
       +
       (1 - ((\mathbf{x})_i - n_i))
       \sum_{\substack{j=1 \\ (\mathbf{u}_j)_i = n_i + 1}} ^ {d+1}  
       \lambda_j 
       \right)  \\
    &\le \sum_{i=1}^d \left(
         \frac{1}{4} 
         +
         \frac{1}{4}
         \right)  \\
    &= \frac{d}{2}
\end{align*}
Therefore,
$$|\Phi(\mathbf{x}) - f(\mathbf{x}) |
\le \omega 
    \left( 
    \sum_{j=1}^{d+1}
    \lambda_j \| \mathbf{u}_j - \mathbf{x} \|_1  
    \right) 
\le \omega\left( \frac{d}{2} \right)$$
Since $\mathbf{x}\in[0,n]^d$ was chosen arbitrarily, we conclude that
\[
        \| f-\Phi \|_{L^{\infty}([0,n]^d)} 
    \le 
        \omega\left( \frac{d}{2} \right)
.
\]
This concludes our proof.
\end{proof}
Indeed, we may deduce an even stronger conclusion.  Namely, we only used information of $f$ at the lattice points, so the modulus of regularity $\omega_{\Phi}$ of $\Phi$ can be the minimum concave function that satisfies
$$\omega_{\Phi}(i) = 
\max\limits_
    {\substack{
        \mathbf{x},\mathbf{y}\in\{0,1,\cdots,n\}^d  \\
        \|\mathbf{x} - \mathbf{y}\|_1 = i  
        }
    }
|f(\mathbf{x}) - f(\mathbf{y})|,\quad \text{for } i=0,1,\cdots,nd$$
which is a polygonal function. Therefore, we have the following corollary.

\begin{corollary}
\label{cor:CPWLApprox_wRegControl_cor}
    Let $f$ be a function from $[0,n]^d$ to $\mathbb{R}$ for some $n\in\mathbb{N}_+$. Let $\omega_f$ be a modulus of regularity of $f$ under $\ell^1$ norm of $\mathbb{R}^d$. Let $\Phi:[0,n]^d\rightarrow\mathbb{R}$ be continuous piecewise linear on each of the $n^d$ unit cubes $\{[0,1],[1,2],\cdots,[n-1,n]\}^d$ with respect to the triangulation in Lemma \ref{lem:triangulation_of_Rd}, and
    $$\Phi(\mathbf{y}) = f(\mathbf{y}), \quad
    \forall \mathbf{y} \in\{0,1,\cdots,n\}^d$$
    Then, any monotone increasing and concave function $\omega_{\Phi}$ that satisfies the following condition is a modulus of regularity of $\Phi$ under $\ell^1$ norm of $\mathbb{R}^d$:
    $$\omega_{\Phi}(i) = 
    \max\limits_
    {\substack{
        \mathbf{x},\mathbf{y}\in\{0,1,\cdots,n\}^d  \\
        \|\mathbf{x} - \mathbf{y}\|_1 = i  
        }
    }
    |f(\mathbf{x}) - f(\mathbf{y})|,\quad 
    \text{for } i=0,1,\cdots,nd$$
    Moreover, 
    $$\| f-\Phi \|_{L^{\infty}([0,n]^d)} \le 
    \omega_f \left( \frac{d}{2} \right)$$
\end{corollary}

\subsection{{Proof of the Main Result (Theorem~\ref{thm:MainTheorem_StableUniversalApproximation})}}
\label{s:Compute__ss:MainResult}
Using the results above, we are able to derive our first main theorem, namely Theorem~\ref{thm:MainTheorem_StableUniversalApproximation}.  

\begin{proof}[{Proof of Theorem~\ref{thm:MainTheorem_StableUniversalApproximation}}]
To apply the previous results, we first scale $f$ into a function $\Tilde{f}$ defined on $[0,n]^d$: 
\begin{equation}
\label{eq:rescalingdomain}
\Tilde{f}(\mathbf{x})  \eqdef  f\left(\frac{\mathbf{x}}{n}\right), \quad 
\forall \mathbf{x}\in[0,n]^d
\end{equation}
Since $\omega$ is a modulus of regularity of $f$, for all $\mathbf{x},\mathbf{y}\in[0,1]^d$, we have 
$$|\Tilde{f}(\mathbf{x}) - \Tilde{f}(\mathbf{y})| 
= \left|
      f\left(\frac{\mathbf{x}}{n}\right) 
      - 
      f\left(\frac{\mathbf{y}}{n}\right)
  \right|
\le \omega\left( 
        \left\| 
            \frac{\mathbf{x}}{n} - \frac{\mathbf{y}}{n} 
        \right\|_1 
    \right)
=   \omega\left( 
        \frac{\|\mathbf{x} - \mathbf{y}\|_1}{n}
    \right)$$
then $\Tilde{\omega}$ defined below is a modulus of regularity of $\Tilde{f}$: 
$$\Tilde{\omega}(x)  \eqdef  \omega\left(\frac{x}{n}\right), \quad 
\forall x\in[0,nd]$$
Let
$$\Tilde{\Phi}(\mathbf{x})  \eqdef  
    \sum_{
    \mathbf{y}\in\{0,1,\cdots,n\}^d
    }
    \Tilde{f}(\mathbf{y})
    T_{\mathbf{y}}(\mathbf{x})
$$
where $T_{\mathbf{y}}$ be the hat function 
$$T_{\mathbf{y}}(\mathbf{x}) \eqdef 
\sigma
    \left(
        1 - 
        \max
        \left\{
            | x_i - y_i | :
            y_i \text{ is even}
        \right\}
        -
        \max
        \left\{
            | x_i - y_i | :
            y_i \text{ is odd}
        \right\}
    \right)$$
for $\mathbf{x} = (x_1,\cdots,x_d) \in [0,n]^d$. By Lemma \ref{lem:construct_CPL_approximator}, $\Tilde{\Phi}$ is continuous piecewise linear on each of the $n^d$ unit cubes $\{[0,1],[1,2],\cdots,[n-1,n]\}^d$ of $[0,n]^d$ with respect to the triangulation in Lemma \ref{lem:triangulation_of_Rd}, and
\begin{equation}
\label{eq:interpolation}
\Tilde{\Phi}(\mathbf{y}) = \Tilde{f}(\mathbf{y}), \quad
\forall \mathbf{y} \in\{0,1,\cdots,n\}^d
\end{equation}
Then, by Lemma \ref{lem:CPWLApprox_wRegControl}, $\Tilde{\omega}$ is also a modulus of regularity of $\Tilde{\Phi}$ under $\ell^1$ norm of $\mathbb{R}^d$. Moreover, 
$$\| \Tilde{f}-\Tilde{\Phi} \|_{L^{\infty}([0,n]^d)} \le 
\Tilde{\omega} \left( \frac{d}{2} \right)$$
Now let 
\begin{equation}
\label{eq:rescalingPhi}
\Phi(\mathbf{x})  \eqdef  \Tilde{\Phi}(n\mathbf{x}), \quad 
\forall \mathbf{x}\in[0,1]^d
\end{equation}
Then, we deduce can $\Phi$ uniformly approximates $f$ on $[0,1]^d$ since
$$\| f-\Phi \|_{L^{\infty}([0,1]^d)}
= \| \Tilde{f}-\Tilde{\Phi} \|_{L^{\infty}([0,n]^d)} 
\le \Tilde{\omega} \left( \frac{d}{2} \right)
= \omega\left(\frac{d}{2n}\right)$$

Define the number of samples point $N(n)\eqdef (1+n)^d$ and the grid $\{\mathbf{x}_i\}_{i=1}^{N(n)}\eqdef n^{-1}\cdot \mathbb{X}_n$.  Note that, together~\eqref{eq:rescalingdomain} and~\eqref{eq:interpolation} along with the definition of $\Phi$ given in~\eqref{eq:rescalingPhi} implies that
\[
    \Phi(\mathbf{x}_i) = f(\mathbf{x}_i)
    \qquad
    \forall i=1,\dots,N(n)
.
\]
In other words, $\Phi$ interpolates $f$ on the grid $\{\mathbf{x}_i\}_{i=1}^{N(n)}$.  Thus, $\sum_{i=1}^{N(n)}\, |\Phi(\mathbf{x}_i)-f(\mathbf{x}_i)|=0$.

Since $\Tilde{\omega}$ is a modulus of regularity of $\Tilde{\Phi}$, for all $\mathbf{x},\mathbf{y}\in[0,1]^d$, 
$$| \Phi(\mathbf{x}) - \Phi(\mathbf{y}) |
= | \Tilde{\Phi}(n\mathbf{x}) - \Tilde{\Phi}(n\mathbf{y}) |
\le \Tilde{\omega} ( \|n\mathbf{x} - n\mathbf{y}\|_1 )
= \omega( \|\mathbf{x} - \mathbf{y}\|_1 )$$
Therefore, $\omega$ is a modulus of regularity of $\Phi$. Now we remain to show that $\Phi$ can be implemented by a ReLU MLP with width at most $8d(n+1)^d$ and depth at most $\lceil \log_2{d} \rceil + 4$.

It is easy to see that
$$x=\sigma(x)-\sigma(-x), \quad |x|=\sigma(x)+\sigma(-x),
\quad \forall x\in\mathbb{R}$$
then 
$$\max\{x,y\} = \frac{x+y}{2} + \frac{|x-y|}{2} = \frac{\sigma(x+y)}{2} - \frac{\sigma(-x-y)}{2} + \frac{\sigma(x-y)}{2} + \frac{\sigma(-x+y)}{2},
\quad \forall x,y\in\mathbb{R}$$
Using this formula, for any $k\in\mathbb{N}_+$, we may compute the maximum value function on $2^k$ arguments via the following network structure:
\begin{align}
\label{eq:max_network_structure}
    \begin{bmatrix} x_1^{(1)} \\ \vspace{-5pt} \\ x_2^{(1)} \\ \vspace{-5pt} \\ x_3^{(1)} \\ \vspace{-5pt} \\ x_4^{(1)} \\ \vspace{-5pt} \\ \vdots \\ \vspace{-5pt} \\ x_{2^k-1}^{(1)} \\ \vspace{-5pt} \\ x_{2^k}^{(1)} \end{bmatrix}
    \ &\Longrightarrow \
    \begin{bmatrix} 
        \displaystyle
        \sigma\left(\frac{\pm x_1^{(1)} \pm x_2^{(1)}}{2}\right) \\ 
        \vspace{-3pt} \\ 
        \displaystyle
        \sigma\left(\frac{\pm x_3^{(1)} \pm x_4^{(1)}}{2}\right) \\ 
        \vspace{-3pt} \\ \vdots \\ \vspace{-3pt} \\ 
        \displaystyle
        \sigma\left(\frac{\pm x_{2^k-1}^{(1)} \pm x_{2^k}^{(1)}}{2}\right)
    \end{bmatrix}
    \ \Longrightarrow \
    {\color{persiangreen}
    \begin{bmatrix} \max\{x_1^{(1)},x_2^{(1)}\} \eqdef x_1^{(2)} \\ \vspace{-3pt} \\ \max\{x_3^{(1)},x_4^{(1)}\} \eqdef x_2^{(2)} \\ \vspace{-3pt} \\ \vdots \\ \vspace{-3pt} \\ \max\{x_{2^k-1}^{(1)},x_{2^k}^{(1)}\} \eqdef x_{2^{k-1}}^{(2)} \end{bmatrix}
    }
    \nonumber
    \\
    \ &\Longrightarrow \
    \begin{bmatrix} 
        \displaystyle
        \sigma\left(\frac{\pm x_1^{(2)} \pm x_2^{(2)}}{2}\right) \\ 
        \vspace{-3pt} \\ 
        \displaystyle
        \sigma\left(\frac{\pm x_3^{(2)} \pm x_4^{(2)}}{2}\right) \\ 
        \vspace{-3pt} \\ \vdots \\ \vspace{-3pt} \\ 
        \displaystyle
        \sigma\left(\frac{\pm x_{2^{k-1}-1}^{(2)} \pm x_{2^{k-1}}^{(2)}}{2}\right)
    \end{bmatrix}
    \ \Longrightarrow \
    {\color{persiangreen}
    \begin{bmatrix} \max\{x_1^{(2)},x_2^{(2)}\} \eqdef x_1^{(3)} \\ \vspace{-3pt} \\ \max\{x_3^{(2)},x_4^{(2)}\} \eqdef x_2^{(3)} \\ \vspace{-3pt} \\ \vdots \\ \vspace{-3pt} \\ \max\{x_{2^{k-1}-1}^{(2)},x_{2^{k-1}}^{(2)}\} \eqdef x_{2^{k-2}}^{(3)} \end{bmatrix}
    }
    \nonumber
    \\
    \vspace{-3pt}
    \nonumber
    \\
    \ &\Longrightarrow \ \cdots \ \Longrightarrow \
    {\color{persiangreen}
    \begin{bmatrix} \max\{x_1^{(k-1)},x_2^{(k-1)}\} \eqdef x_1^{(k)} \\ \vspace{-3pt} \\ \max\{x_3^{(k-1)},x_4^{(k-1)}\} \eqdef x_2^{(k)} \end{bmatrix}
    }
    \nonumber
    \\
    \vspace{-3pt}
    \nonumber
    \\
    \ &\Longrightarrow \
    \begin{bmatrix} 
        \displaystyle
        \sigma\left(\frac{\pm x_1^{(k)} \pm x_2^{(k)}}{2} \right) 
    \end{bmatrix}
    \ \Longrightarrow \
    {\color{persiangreen}
    \begin{bmatrix} \max\{x_1^{(k)},x_2^{(k)}\} \end{bmatrix}
    }
    = 
    {\color{persiangreen}
    \begin{bmatrix} 
        \max\{x_1^{(1)},x_2^{(1)},\cdots,x_{2^k}^{(1)}\} 
    \end{bmatrix}
    }
\end{align}
where $\sigma(\pm a\pm b)$ abbreviates the four neurons $\sigma(a+b),\sigma(a-b),\sigma(-a+b)$ and $\sigma(-a-b)$, and the column vectors in green represent the ``auxiliary'' hidden layers which are simply affine transformations that don't use any activation functions, thus each of them can be integrated with the layer after it, and does not require an extra hidden layer to process. Therefore, we will ignore these layers. We denote this network as $\mathbf{M}_k$, thus $\mathbf{M}_k: \mathbb{R}^{2^k} \rightarrow \mathbb{R}$ outputs the maximum value of its $2^k$ input arguments, and it has width $2^{k+1}$ and depth $k+1$. We note that, after deleting an appropriate set of neurons in $\mathbf{M}_k$ (which we will still call $\mathbf{M}_k$), it can be applied to $k'$ arguments for any $1\le k'\le 2^k$ and output their maximum. 

Let 
$$d_0 \eqdef \lceil \log_2{d} \rceil$$
For any lattice point $\mathbf{y}=(y_1,y_2,\cdots,y_d)\in\{0,1,\cdots,n\}^d$, suppose $y_{i_1},y_{i_2},\cdots,y_{i_p}$ are even and $y_{j_1},y_{j_2},\cdots,y_{j_q}$ are odd, for some $0 \le p,q \le d$, then we can implement the hat function $T_{\mathbf{y}}$ using the network $\mathbf{M}_{d_0}$ as follows: 
\begin{align}
\label{eq:hat_network_structure}
    \begin{bmatrix} x_1 \\ x_2 \\ \vdots \\ x_d \end{bmatrix}
    \ &\Longrightarrow \
    \begin{bmatrix} \sigma(\pm(x_1-y_1)) \\ \sigma(\pm(x_2-y_2)) \\ \vdots \\ \sigma(\pm(x_d-y_d)) \end{bmatrix}
    \ \Longrightarrow \
    {\color{persiangreen}
    \begin{bmatrix} 
    \sigma(x_1 - y_1) + \sigma(-(x_1 - y_1)) = |x_1 - y_1|  \\ 
    \sigma(x_2 - y_2) + \sigma(-(x_2 - y_2)) = |x_2 - y_2|  \\ 
    \vdots  \\
    \sigma(x_d - y_d) + \sigma(-(x_d - y_d)) = |x_d - y_d|  \\
    \end{bmatrix}
    }
    \nonumber
    \\
    \ &\Longrightarrow \
    \begin{bmatrix} 
    \mathbf{M}_{d_0}\left(
        \begin{bmatrix}
            |x_{i_1} - y_{i_1}| & |x_{i_2} - y_{i_2}| & \cdots & |x_{i_p} - y_{i_p}|
        \end{bmatrix}^{\top}
    \right)  
    \nonumber
    \\ 
    \vspace{-5pt}  
    \nonumber
    \\
    \mathbf{M}_{d_0}\left(
        \begin{bmatrix}
            |x_{j_1} - y_{j_1}| & |x_{j_2} - y_{j_2}| & \cdots & |x_{j_q} - y_{j_q}|
        \end{bmatrix}^{\top}
    \right)
    \end{bmatrix}  \\
    \ &= \ 
    \begin{bmatrix} 
    \max
    \left\{
        | x_i - y_i | :
        y_i \text{ is even}
    \right\}  \\ 
    \max
    \left\{
        | x_i - y_i | :
        y_i \text{ is odd}
    \right\}  \\ 
    \end{bmatrix} 
    \nonumber
    \\
    \ &\Longrightarrow \
    \begin{bmatrix}
        \sigma
        \left(
            1 - 
            \max
            \left\{
                | x_i - y_i | :
                y_i \text{ is even}
            \right\}
            -
            \max
            \left\{
                | x_i - y_i | :
                y_i \text{ is odd}
            \right\}
        \right)
    \end{bmatrix}  
    \nonumber
    \\
    \ &= \ 
    T_{\mathbf{y}}\left(
        \begin{bmatrix}
            x_1 & x_2 & \cdots & x_d
        \end{bmatrix}^{\top}
    \right)
\end{align}
As before, $\sigma(\pm a)$ abbreviates the two neurons $\sigma(a),\sigma(-a)$, and the column vectors in green represent the ``auxiliary'' hidden layer which we can ignore. We call this network $\mathbf{T}_{\mathbf{y}}$, which has width at most 
\begin{equation}
\label{eq:hat_width_bound}
    \operatorname{width}(\mathbf{T}_{\mathbf{y}})
    = 
    \operatorname{width}(\mathbf{M}_{d_0}) + 
    \operatorname{width}(\mathbf{M}_{d_0})
    \le 
    2^{d_0+1} + 2^{d_0+1} = 2^{d_0+2} \le 2^{(\log_2{d}+1)+2} = 8d
\end{equation}
and depth at most
\begin{equation}
\label{eq:hat_depth_bound}
    \operatorname{depth}(\mathbf{T}_{\mathbf{y}})
    = 1 + \operatorname{depth}(\mathbf{M}_{d_0}) + 1
    \le d_0 + 3
\end{equation}
Let $\mathbf{y}_1,\mathbf{y}_2,\cdots,\mathbf{y}_{(n+1)^d}$ be the $(n+1)^d$ lattice points inside the cube $[0,n]^d$. Since 
$$\Phi(\mathbf{x}) = \Tilde{\Phi}(n\mathbf{x}) 
= \sum_{
\mathbf{y}\in\{0,1,\cdots,n\}^d
}
\Tilde{f}(\mathbf{y})
T_{\mathbf{y}}(n\mathbf{x}),
\quad \forall \mathbf{x}\in[0,n]^d$$
then $\Phi$ can be implemented as follows: 
\begin{align}
\label{eq:phi_construction}
    \mathbf{x}
    \ \Longrightarrow \
    {\color{persiangreen}
    n\mathbf{x}
    }
    \ \Longrightarrow \
    \begin{bmatrix} 
    \mathbf{T}_{\mathbf{y}_1}(n\mathbf{x})
    \\ 
    \mathbf{T}_{\mathbf{y}_2}(n\mathbf{x})
    \\
    \vdots
    \\
    \mathbf{T}_{\mathbf{y}_{(n+1)^d}}(n\mathbf{x})
    \end{bmatrix}
    \ \Longrightarrow \
    \begin{bmatrix} 
    \displaystyle
    \sum_{i=1}^{(n+1)^d}
    \Tilde{f}(\mathbf{y}_i)
    T_{\mathbf{y}_i}(n\mathbf{x})
    \end{bmatrix}
    \ = \
    \Phi(\mathbf{x})
\end{align}
This network has width
$$\operatorname{width}(\Phi)
= \sum_{i=1}^{(n+1)^d} \operatorname{width}(\mathbf{T}_{\mathbf{y}_i})
\le \sum_{i=1}^{(n+1)^d} 8d 
= 8d(n+1)^d$$
and depth
$$\operatorname{depth}(\Phi)
= \max\limits_{1\le i\le (n+1)^d} 
  \operatorname{depth}(\mathbf{T}_{\mathbf{y}_i}) + 1
\le (d_0 + 3) + 1
= \lceil \log_2{d} \rceil + 4.$$

Finally, we tally the nonzero parameters in $\Phi$. From \eqref{eq:max_network_structure}, the network $\mathbf{M}_k$ has no nonzero biases, and has nonzero weights at most 
\[
    M 
    \eqdef
    {\color{blue}4} \cdot 2^{k-1}
    +
    8 \cdot 
    \sum_{i=1}^{k-2} 2^i
    \le
    2 \cdot 2^k + 8 \cdot 2^{k-1}
    =
    6 \cdot 2^k
\]
where the factor ${\color{blue}4}$ colored in blue is the number of nonzero connections between the second layer in \eqref{eq:max_network_structure} and the second layer in \eqref{eq:hat_network_structure}, after ``integrating'' the green auxiliary layer in \eqref{eq:hat_network_structure} with the layer after it. In other words, this is because in \eqref{eq:hat_network_structure}, each of the neurons in the input layer of $\mathbf{M}_{d_0}$ is the sum of two neurons in the second layer, so since in \eqref{eq:max_network_structure} each neuron in the second layer has $2$ nonzero connection with the first layer, in \eqref{eq:hat_network_structure} this number will have to be doubled, becoming $4$. 

From \eqref{eq:hat_network_structure}, the network $\mathbf{T}_{\mathbf{y}}$ has at most $T^{(b)} \eqdef 2d+1$ nonzero biases, and has nonzero weights at most 
\[
    T^{(w)} 
    \eqdef 
    2d + M + 2
    =
    2d + 6 \cdot 2^{d_0} + 2
    \le
    2d + 6 (2d-1) + 2
    \le
    14d - 4
\]
In total, the network $\mathbf{T}_{\mathbf{y}}$ has nonzero parameters at most
\begin{equation}
\label{eq:hat_par_bound}
    T \eqdef T^{(w)} + T^{(b)} = 16d - 3
\end{equation}

From \eqref{eq:phi_construction}, $\Phi$ has no additional nonzero biases, thus it has nonzero parameters at most
\begin{equation}
\label{eq:total_parameters}
    (n+1)^d T + (n+1)^d \le 16d (n+1)^d.
\end{equation}
\end{proof}

\begin{remark}[{Proof of Proposition~\ref{proposition:ParamterEstimates__RefinedVersion} is given in the Proof of Theorem~\ref{thm:MainTheorem_StableUniversalApproximation}}]
\label{remark:ParamterEstimates__RefinedVersion}
The proof of Proposition~\ref{proposition:ParamterEstimates__RefinedVersion} directly follows from the proof of Theorem~\ref{thm:MainTheorem_StableUniversalApproximation}, upon noting that all parameters except for the ones between the input layer and the first hidden layer and the ones between the output layer and the last hidden layer come from $\{0,\pm 1/2\}$, and all parameters except for the ones between the output layer and the last hidden layer are independent of the ``sample values'' $\{f(\mathbf{x}_i)\}_{i=1}^{N(n)}$; i.e.\ the value of the encoder $\mathcal{E}_n(f)$. 
\end{remark}

\subsection{{Proof of Secondary Approximation Results}}
\label{s:Proofs__ss:Secondary}

\begin{proof}[Proof of Proposition~\ref{prop:MainTheorem_AdjustableSize}]
In the proof of Theorem~\ref{thm:MainTheorem_StableUniversalApproximation}, we constructed final network in \eqref{eq:phi_construction}. Here, we only need to distribute the $(n+1)^d$ sub-networks $\mathbf{T}_{\mathbf{y}_1}, \cdots, \mathbf{T}_{\mathbf{y}_{(n+1)^d}}$ to $L$ different layers.

We relabel the $(n+1)^d$ sub-network and function value pairs $\big(\mathbf{T}_{\mathbf{y}_j}, f(\mathbf{y}_j)\big)$ for $j=1,\cdots,(n+1)^d$ as $\big(\mathbf{T}_1^{(i)}, f_1^{(i)}\big), \cdots, \big(\mathbf{T}_{m_i}^{(i)}, f_{m_i}^{(i)}\big)$ for $i=1,\cdots,L$. To simplify notations, for $k=1,\cdots,L$, let
\[
    T_k (\mathbf{x})
    \eqdef
    \sum_{i=1}^k
    \sum_{j=1}^{m_i}
    f_j^{(i)} \mathbf{T}_j^{(i)}(n\mathbf{x}),
    \quad
    \forall \mathbf{x} \in [0,1]^d,
\]
Then, $\Phi(\mathbf{x}) = T_L(\mathbf{x})$. Now, we can construct $\Phi$ as follows:
\begin{align*}
    \mathbf{x}
    \ &\Longrightarrow \
    \begin{bmatrix} 
    \mathbf{T}_1^{(1)}(n\mathbf{x})
    \\ 
    \vdots
    \\
    \mathbf{T}_{m_1}^{(1)}(n\mathbf{x})
    \\
    n\mathbf{x}
    \end{bmatrix}
    \ \Longrightarrow \
    \begin{bmatrix} 
    \mathbf{T}_1^{(2)}(n\mathbf{x})
    \\ 
    \vdots
    \\
    \mathbf{T}_{m_2}^{(2)}(n\mathbf{x})
    \\
    \sigma(\pm T_1 (\mathbf{x}))
    \\
    n\mathbf{x}
    \end{bmatrix}
    \ \Longrightarrow \
    \begin{bmatrix} 
    \mathbf{T}_1^{(3)}(n\mathbf{x})
    \\ 
    \vdots
    \\
    \mathbf{T}_{m_3}^{(3)}(n\mathbf{x})
    \\
    \sigma(\pm T_2 (\mathbf{x}))
    \\
    n\mathbf{x}
    \end{bmatrix}
    \ \Longrightarrow \
    \cdots
    \\
    \ &\Longrightarrow \
    \begin{bmatrix} 
    \mathbf{T}_1^{(L)}(n\mathbf{x})
    \\ 
    \vdots
    \\
    \mathbf{T}_{m_L}^{(L)}(n\mathbf{x})
    \\
    \sigma(\pm T_{L-1} (\mathbf{x}))
    \\
    n\mathbf{x}
    \end{bmatrix}
    \ \Longrightarrow \
    {\color{persiangreen}
    T_L(\mathbf{x})
    }
    \ = \ 
    \Phi(\mathbf{x})
\end{align*}
where we used the fact that $\sigma(n\mathbf{x}) = n\mathbf{x}$ for any $\mathbf{x}\in[0,1]^d$. We know from \eqref{eq:hat_width_bound} and \eqref{eq:hat_depth_bound} that the sub-networks $\mathbf{T}_j^{(i)}$ has width $W \le 8d$, depth $D \le \lceil \log_2{d} \rceil + 3$, and nonzero parameters $T \le 16d-3$. Then, $\Phi$ has width at most
\[
    \operatorname{width}(\Phi)
    \le
    W \max\{m_1,\cdots,m_L\} + 2 + d
    =
    8d \max\{m_1,\cdots,m_L\} + d + 2
\]
depth at most
\[
    \operatorname{depth}(\Phi)
    \le
    L (D+1)
    =
    L (\lceil \log_2{d} \rceil + 4)
\]
and nonzero parameters at most
\begin{align*}
    \operatorname{par}(\Phi)
    &\le
    m_1 T + d
    +
    \sum_{i=2}^L
    (m_i T + 2m_{i-1} + 2 + d)
    +
    m_L + 2
    \\
    &\le
    (T+2) \sum_{i=1}^L m_i + L(d+2)
    \\
    &\le
    16d(n+1)^d + L(d+2)
    .
\end{align*}

\end{proof}

\begin{proof}[{Proof of Corollary~\ref{cor:main_StableUniversalApproximation___with_OptimalLipschitzExtrapolation}}]
Since the support of $\Phi$ in Theorem \ref{thm:MainTheorem_StableUniversalApproximation} is contained in $[-1/n,1+1/n]^d$ which is just slightly larger than $[0,1]^d$, the original construction of $\Phi$ could lead to a steep descent near the boundary of $[0,1]^d$. To resolve this issue, we avoid taking values of $\Phi$ outside $[0,1]^d$ by first projecting $\mathbb{R}^d$ to $[0,1]^d$ by applying the following 1D projection coordinate-wise: 
\begin{equation}
\label{eq:projection_l1}
    p(x) \eqdef \sigma(x) - \sigma(x-1)
    = 
    \begin{cases}
        0, & \text{if } x < 0  \\
        x, & \text{if } x \in [0,1]  \\
        1, & \text{if } x > 1  \\
    \end{cases},
    \quad \forall x \in \mathbb{R}
\end{equation}
For any $\mathbf{x},\mathbf{y}\in\mathbb{R}^d$, $p(\mathbf{x}) , p(\mathbf{y}) \in [0,1]^d$, $\|p(\mathbf{x}) - p(\mathbf{y})\|_1 \le \|\mathbf{x} - \mathbf{y}\|_1$, thus
\[
    | \Phi(p(\mathbf{x})) - \Phi(p(\mathbf{y})) |
    \le
    \omega( \| p(\mathbf{x}) - p(\mathbf{y}) \|_1 )
    =
    \bar{\omega}( \| p(\mathbf{x}) - p(\mathbf{y}) \|_1 )
    \le
    \bar{\omega}( \| \mathbf{x} - \mathbf{y} \|_1 )
\]
Therefore, $\bar{\omega}$ is a modulus of regularity of $\Phi\circ p$, whose restriction on $[0,1]^d$ is the same as $\Phi$. Thus, the function $\Phi\circ p$ satisfies our requirement, which is implemented by the following network: 
\begin{equation}
\label{eq:construction_of_Phi_global}
    \mathbf{x}
    \ \Longrightarrow \ 
    \begin{bmatrix} \sigma(\mathbf{x}) \\ \sigma(\mathbf{x}-1) \end{bmatrix}
    \ \Longrightarrow \ 
    \begin{bmatrix} 
        \Phi ( \sigma(\mathbf{x}) - \sigma(\mathbf{x}-1) )
    \end{bmatrix}
    \ = \
    \Phi ( p(\mathbf{x}) )
\end{equation}
which has width 
\[
    \operatorname{width}(\Phi\circ p)
    =
    \max\{ 2d, \operatorname{width}(\Phi) \}
    \le
    \max\{ 2d, 8d(n+1)^d \}
    =
    8d(n+1)^d
\]
and depth
\[
    \operatorname{depth}(\Phi\circ p)
    =
    \operatorname{depth}(\Phi) + 1
    \le
    \lceil \log_2{d} \rceil + 5
    .
\]
Finally, we tally the nonzero parameters in this network. From the proof of Theorem \ref{thm:MainTheorem_StableUniversalApproximation} (see the left-hand side of \eqref{eq:total_parameters}), we see that the network $\Phi$ has at most $(16d-2)(n+1)^d$ nonzero parameters (we omitted the $(-2)$ in the result to simplify the expression). Now, in the original construction of $\Phi$ (see \eqref{eq:hat_network_structure} and \eqref{eq:phi_construction}), there are $2d(n+1)^d$ neurons in its first hidden layer, each connecting to $m=1$ neuron in its input layer. Here, in the network \eqref{eq:construction_of_Phi_global} implementing $\Phi\circ p$, this number $m$ is doubled, resulting in an additional
\[
    W_1 
    \eqdef 
    (2m-m) \cdot 2d(n+1)^d 
    =
    2d(n+1)^d
\]
nonzero weights. For the biases, it is clear that only $B \eqdef d$ additional biases (namely $-1$) were introduced in the first hidden layer. Lastly, the only additional nonzero parameters in \eqref{eq:construction_of_Phi_global} are the weights between the input layer and the first hidden layer, which are $W_2 \eqdef 2d$ in total. Therefore, the number of nonzero parameters in \eqref{eq:construction_of_Phi_global} is at most 
\begin{align*}
    (16d-2)(n+1)^d + W_1 + B + W_2
    &=
    18d(n+1)^d - 2(n+1)^d + 3d
    \\
    &\le
    18d(n+1)^d - 2 \cdot 2^d + 3d
    \\
    &\le
    18d(n+1)^d.
\end{align*}
\end{proof}

\begin{proof}[{Proof of Corollary~\ref{cor:Extrapolation}}]
Fix $n\in \mathbb{N}_+$.
First note that the 
metric projection onto the unit cube in $\mathbb{R}^d$, given for each $\mathbf{x}\in \mathbb{R}^d$ by $
    \Pi(\mathbf{x}) \in \operatorname{argmin}_{\mathbf{z}\in 
    [0,1]^d
    }\, \|\mathbf{z}-\mathbf{x}\|_1
$ is well-defined since $[0,1]^d$ is closed and convex.  One can readily verify that $\Pi=p$, where $p$ is defined in~\eqref{eq:projection_l1}.

Now, by Corollary~\ref{cor:main_StableUniversalApproximation___with_OptimalLipschitzExtrapolation}, there exists a ReLU MLP $\Phi:\mathbb{R}^d\to \mathbb{R}$ which is uniformly continuous with modulus of regularity $\bar{\omega}$, width at most $8d(n+1)^d$, and depth at most $\lceil \log_2{d} \rceil + 5$, satisfying the uniform estimate
\begin{equation}
\label{eq:Proof_SOAMLP}
        \|f-\Phi\|_{L^{\infty}([0,1]^d)} 
    \le 
        \bar{\omega}\left(\frac{d}{2n}\right)
.
\end{equation}   
Consequentially,~\eqref{eq:Proof_SOAMLP} implies that: for each $\mathbf{x}\in \mathbb{R}^d$ the following holds
\allowdisplaybreaks
\begin{align}
\nonumber
        |f(\mathbf{x})-\Phi(\mathbf{x})|
    & \le 
        |f(p(\mathbf{x}))-\Phi(p(\mathbf{x}))|
        +
        |f(p(\mathbf{x}))-f(\mathbf{x})|
        +
        |\Phi(p(\mathbf{x}))-\Phi(\mathbf{x})|
\\
\label{eq:projection_Phi_definition}
& =
        |f(p(\mathbf{x}))-\Phi(p(\mathbf{x}))|
        +
        |f(p(\mathbf{x}))-f(\mathbf{x})|
        +
        |\Phi(p(\mathbf{x}))-\Phi(p(\mathbf{x}))|
\\
\nonumber
& = 
        |f(p(\mathbf{x}))-\Phi(p(\mathbf{x}))|
        +
        |f(p(\mathbf{x}))-f(\mathbf{x})|
\\
\nonumber
& \le 
        |f(p(\mathbf{x}))-\Phi(p(\mathbf{x}))|
        +
        \bar{\omega}(\|p(\mathbf{x})-\mathbf{x})\|_1)
\\
\nonumber
&= 
        |f(p(\mathbf{x}))-\Phi(p(\mathbf{x}))|
        +
        \bar{\omega}(\|p(\mathbf{x})-\mathbf{x})\|_1)
\\
\nonumber
& \le 
        \sup_{\mathbf{u}\in [0,1]^d}\,
        |f(\mathbf{u})-\Phi(\mathbf{u})|
        +
        \bar{\omega}(\|p(\mathbf{x})-\mathbf{x})\|_1)
\\
\label{eq:express_as_distance__from_projection}
& \le 
        \bar{\omega}\left(\frac{d}{2n}\right)
        +
        \bar{\omega}(\|p(\mathbf{x})-\mathbf{x})\|_1)
,
\end{align}
where~\eqref{eq:projection_Phi_definition} followed by the proof of by Corollary~\ref{cor:main_StableUniversalApproximation___with_OptimalLipschitzExtrapolation}, specifically in~\eqref{eq:construction_of_Phi_global}, we see that the first layer of $\Phi$ is given by pre-composition with $p$.

Since $[0,1]^d$ is compact then the minimal distance to any given point is realized. Also, since $p$ is the metric projection of $\mathbb{R}^d$ onto $[0,1]^d$, then 
\[
    d_{[0,1]^d}(\mathbf{x}) = \|p(\mathbf{x})-\mathbf{x}\|_1
\]
for every $\mathbf{x}\in \mathbb{R}^d$.  Consequentially, the right-hand side of~\eqref{eq:express_as_distance__from_projection} can be re-expressed as
\[
        \bar{\omega}\left(\frac{d}{2n}\right)
        +
        \bar{\omega}(\|p(\mathbf{x})-\mathbf{x})\|_1)
    =
        \bar{\omega}\left(\frac{d}{2n}\right)
        +
        \bar{\omega}(\operatorname{dist}_{[0,1]^d}(\mathbf{x}))
,
\]
which concludes our proof.
\end{proof}

\subsubsection{{Proof of Statistical Result - Theorem~\ref{thm:RademacherBound}}}
\label{s:learning_theory_proofs}
Let $\mathcal{F}$ be a non-empty set of functions from $\mathbb{R}^d$ to $\mathbb{R}$, $Z_1,\dots,Z_N$ be random vectors in $\mathbb{R}^d$, and let $\mathcal{Z}\eqdef (Z_n)_{n=1}^N$.
Its \textit{empirical Rademacher complexity}, denoted by $\mathfrak{R}_{\mathcal{Z}}(\mathcal{F})$, is the (random) quantity
\begin{equation}
\label{eq:RademacherComplexityDefinition}
        \mathfrak{R}_{\mathcal{Z}}(\mathcal{F})
    \eqdef 
        \frac1{N}
        \,
            \mathbb{E}_{\sigma}\Biggl[
                \sup_{g\in \mathcal{F}}\,
                \sum_{n=1}^N\,
                    \sigma_n\,g(Z_n)
            \Biggr]
\end{equation}
where $\mathbf{\sigma}=(\sigma_n)_{n=1}^N$ and $\sigma_1,\dots,\sigma_N$ are i.i.d.\ Rademacher random variables; i.e.\ $\mathbb{P}(\sigma_1=1)=\mathbb{P}(\sigma_1=-1)=1/2$.
\begin{proof}[{Proof of Theorem~\ref{thm:RademacherBound}}]
Let $\delta_1,\delta_2\in [0,1)$, to be fixed retroactively.  Let $\mathcal{Z}\eqdef \{(X_n,Y_n)\}_{n=1}^N$.
By~\citep[Theorems 8 and 12]{bartlett2002rademacher}, the following holds with probability at-least $1-\delta_1$
\begin{equation}
\label{eq:empirical_process_round_A}
    \sup_{\Phi\in \mathcal{NN}_{\Delta,W}^{L,n}} 
        \big|\mathcal{R}(\Phi)-\mathcal{R}_N(\Phi)\big|
    \le 
        2\,L_{\ell}\,\mathfrak{R}_{\mathcal{X}}(\mathcal{NN}_{\Delta,W}^{L,n})
        +
        \frac{
            \sqrt{8\log(2/\delta_1)}
        }{
            \sqrt{N}
        }
.
\end{equation}
Since $\mathcal{NN}_{\Delta,W}^{L,n}\subset \mathcal{NN}_{\Delta,W} \cap \operatorname{Lip}(\mathbb{R}^d,[0,1],L)$ then, by the definition of the empirical Rademacher complexity of a class in~\eqref{eq:RademacherComplexityDefinition}, we have that
\begin{equation}
\label{eq:EmpRadComplexity_MinBound}
    \mathfrak{R}_{\mathcal{Z}}(\mathcal{NN}_{\Delta,W}^{L,n})
\le 
    \min\biggl\{ 
        \underbrace{
            \mathfrak{R}_{\mathcal{Z}}\big(\mathcal{NN}_{\Delta,W}\big)
        }_{\term{t:Rad_ParameterSpace}}
    ,
        \underbrace{
            \mathfrak{R}_{\mathcal{Z}}\big(\operatorname{Lip}(\mathbb{R}^d,[0,1],L)\big)
        }_{\term{t:Rad_FunctionSpace}}
    \biggr\}
.
\end{equation}
\textbf{Step 1 - Bounding Term~\eqref{t:Rad_ParameterSpace}:}\hfill\\
Let $\mathbf{X}$ be the $N\times d$ random matrix with rows given by $X_1,\dots,X_N$.
Term~\eqref{t:Rad_ParameterSpace} can be bounded using the empirical Rademacher complexity bound derived in~\citep[Theorem 3.3]{bartlett2017spectrally} which implies that
\begin{equation}
\label{eq:bound__t:Rad_ParameterSpace}
    \eqref{t:Rad_ParameterSpace} 
\le 
    \frac{4}{N^{3/2}}
    +
    \frac{26\log(N)\log(2W)}{N}
    \Biggl(
        \|\mathbf{X}\|_F \prod_{l=1}^{\Delta}\, s_l \, \Big(\sum_{l=1}^{\Delta}\, \big(\frac{b_l}{s_l}\big)^{2/3}\Big)^{3/2}
    \Biggr)
\end{equation}
where $\|\cdot\|_F$ and $\|\cdot\|_2$ respectively denote the Fr\"{o}benius and the spectral matrix norms, and $s_l$ denotes the maximum spectral norm of the $l^{th}$ matrix $\mathbf{W}^{(l)}$ for any ReLU MLP $\Phi$ with representation~\eqref{eq:MLPRepresentation}, and $b_l$ denotes the maximum $\|\cdot\|_{1,2}$ matrix norm thereof (defined by the sum of the Euclidean norm of its columns\footnote{The authors take a transpose in their main result; since this matrix norm is typically defined by summing over the Euclidean norms of the rows of a matrix.}).  

Recall that, if $\Phi\in \mathcal{NN}_{\Delta,W}^{L,n}$ then by Proposition~\ref{proposition:ParamterEstimates__RefinedVersion} the entries of all its weight matrices have absolute value at-most $1/2$ if $l\in \{2,\dots,\Delta-1\}$ and $\max\{n,1\}$ otherwise.  These observations, together with the elementary matrix norm inequalities (all found on~\citep[page 56]{GolubVanLoan_MatrixComputationsBook_1989}), and the fact that each $\Phi\in\mathcal{NN}_{\Delta,W}$ has width at-most $W$, and that by enlarging the class a bit, we may assume that 
$s_1=s_{\Delta}=n\,W^{3/2}$ and that $b_l=s_l=\frac{1}{2}\,W^{3/2}$ for $l=1,\dots,\Delta$.  Thus, the right-hand side of~\eqref{eq:bound__t:Rad_ParameterSpace} simplifies to
\begin{equation}
\label{eq:bound__t:Rad_ParameterSpaceII_a}
    \eqref{t:Rad_ParameterSpace} 
\le 
    \frac{4}{N^{3/2}}
    +
    \frac{26\log(N)\log(2W)}{
        2^{\max\{0,\Delta-2\}}
    N
    }
    \|\mathbf{X}\|_F
    \,W^{3\Delta/2}
.
\end{equation} 
Using the standard matrix-norm bounds $\|\cdot\|_F\le \sqrt{N}\|\cdot\|_{\operatorname{op}}$ (where $\|\cdot\|_{\operatorname{op}}$ denotes the operator norm) and the fact that the operator norm of a matrix $\mathbf{A}$ equals to its largest singular value $\sigma_{max}(\mathbf{A})$, we reduce the right-hand side of~\eqref{eq:bound__t:Rad_ParameterSpaceII_a} to
\begin{equation}
\label{eq:bound__t:Rad_ParameterSpaceII_b}
\begin{aligned}
    \eqref{t:Rad_ParameterSpace} 
& \le 
    \frac{4}{N^{3/2}}
    +
    \frac{26\log(N)\log(2W)}{
        2^{\max(\{0,\Delta-2\}}
    N
    }
    \sqrt{N}\|\mathbf{X}\|_{\operatorname{op}}
    \,W^{3\Delta/2}
\\
& =
    \frac{4}{N^{3/2}}
    +
    \frac{26\log(N)\log(2W)}{
        2^{\max(\{0,\Delta-2\}}
        \sqrt{N}
    }
    \sigma_{max}(\mathbf{X})
    \,W^{3\Delta/2}
.
\end{aligned}
\end{equation} 
Define $\bar{\mathbf{X}}\eqdef \sqrt{N}\cdot \mathbf{X}$.
By the min-max characterization of singular values, see e.g.~\citep[Theorem 4.2.11]{HornJohnstone_2013_matrixanalysisBook}, we have that $\sigma_{max}(\bar{\mathbf{X}})
=
\sqrt{N}\sigma_{max}(\mathbf{X})
$.  Observe also that $\bar{\mathbf{X}}$ is isotropic since
\[
\mathbb{E}\big[\bar{\mathbf{X}}\bar{\mathbf{X}}^{\top}\big] 
= 
\mathbb{E}\big[(\sqrt{N})\mathbf{X}(\sqrt{N}\mathbf{X})^{\top}\big] 
=
N
\mathbb{E}\big[\mathbf{X}\mathbf{X}^{\top}]
=
N
\frac1{N}\mathbf{I}_d
=
\mathbf{I}_d
.
\]
Therefore, a consequence to a version of Gordon's Theorem given in~\citep[Theorem 4.6.1]{Vershynin_HighDimensionalProbBook_2018} applies to the random matrix $\bar{\mathbf{X}}$ since it has independent, sub-Gaussian, and isotropic rows; thus: with probability at-least $1-\delta_2$ we have that 
\begin{equation*}
\sigma_{max}(\bar{\mathbf{X}})\le \sqrt{d} + C\, (\sqrt{N}+ \sqrt{\ln(2/\delta_2)})
.
\end{equation*}
Therefore, with probability at-least $1-\delta_2$, the maximal singular value of $\mathbf{X}$ is bounded above by
\footnote{Remark that: without the correct scaling of $\frac{1}{N}\,I_d$ the covariance of each $X_1$ the third term on the right-hand side of~\eqref{eq:Gordon_HPBound} would not tend to $0$ as $N$ tends to infinity.}
\begin{equation}
\label{eq:Gordon_HPBound}
    \sigma_{max}(\mathbf{X})
\le 
    \frac1{\sqrt{N}}\big(
        \sqrt{d} + C\, (\sqrt{N}+ \sqrt{\ln(2/\delta_2)})
    \big)
=
    \frac{\sqrt{d}}{\sqrt{N}}
    +
    C\frac{
        \sqrt{\ln(2/\delta_2)}
    }{\sqrt{N}}
    +
    C
.
\end{equation}
Incorporating~\eqref{eq:Gordon_HPBound} into the right-hand side of~\eqref{eq:bound__t:Rad_ParameterSpaceII_b}, implies that: the following holds with probability at least $1-\delta_2$
\begin{equation}
\label{eq:bound__t:Rad_ParameterSpaceII}
\begin{aligned}
    \eqref{t:Rad_ParameterSpace} 
\le &
    \frac{4}{N^{3/2}}
    +
    \frac{26\log(N)\log(2W)}{
        2^{\max(\{0,\Delta-2\}}
        \sqrt{N}
    }
    \sigma_{max}(\mathbf{X})
    \,W^{3\Delta/2}
\\
\le &
    \frac{4}{
    N^{3/2}}
    +
    \frac{
        26\log(N)\log(2W)
        \,
        W^{3\Delta/2}
    }{
        2^{\max(\{0,\Delta-2\}}
    N
    }
    \,
    \big(
        \sqrt{d}
        +
        C\,\sqrt{N}
        +
        C\,\sqrt{\ln(2/\delta_2)}
    \big)
.
\end{aligned}
\end{equation} 
Fix $\delta \in (0,1)$ and set $\delta_1\eqdef \delta_2\eqdef \delta/2$.  By a union bound together with~\eqref{eq:bound__t:Rad_ParameterSpaceII} and~\eqref{eq:EmpRadComplexity_MinBound} we deduce that: the following holds with probability at-least $1-\delta$
\begin{equation}
\label{eq:Rademacherbound_NearCompletion}
\begin{aligned}
    \sup_{\Phi\in \mathcal{NN}_{\Delta,W}^{L,n}} 
        \big|\mathcal{R}(\Phi)-\mathcal{R}_N(\Phi)\big|
    &
    \le 
        2\,L_{\ell}\,\mathfrak{R}_{\mathcal{Z}}(\mathcal{NN}_{\Delta,W}^{L,n})
        +
        \frac{
            \sqrt{8\log(4/\delta)}
        }{
            \sqrt{N}
        }
\\
     \mathfrak{R}_{\mathcal{Z}}(\mathcal{NN}_{\Delta,W}^{L,n})
& 
\le 
    \min\biggl\{
            \frac{4}{N^{3/2}}
            +
            \frac{
                26\log(N)\log(2W)
                \,
                W^{3\Delta/2}
            }{
                2^{\max(\{0,\Delta-2\}}
                N
            }
            \,
            \big(
                \sqrt{d}
                +
                C\,\sqrt{N}
                +
                C\,\sqrt{\ln(4/\delta)}
            \big)
    ,
    \\
    &
    \qquad \quad \,
        \underbrace{
            \mathfrak{R}_{\mathcal{Z}}\big(\operatorname{Lip}(\mathbb{R}^d,[0,1],L)\big)
        }_{\term{t:Rad_FunctionSpace_v2}}
    \biggr\}
.
\end{aligned}
\end{equation}

\noindent \textbf{Step 2 - Bounding Term~\eqref{t:Rad_FunctionSpace}:}\hfill\\
By~\cite[Lemma 25]{HouKratsios_JMLR_2023}, we have that
\begin{equation}
\label{eq:Rademacherbound_II}
\eqref{t:Rad_FunctionSpace_v2}
\le 
    C_d\,\frac{L^{\frac{d}{d+3}}}{N^{\frac{1}{d+3}}}
= 
 \left(
    \big(8(d+1)^2 16^d\big)^{\frac{1}{d+3}} + 4\sqrt{2} \frac{16^{\frac{d}{d+3}}}{(18(d+1))^{\frac{d+1}{d+3}}}
\right)
\,\frac{L^{\frac{d}{d+3}}}{N^{\frac{1}{d+3}}}
\end{equation}
where $C_d\eqdef \big(8(d+1)^2 16^d\big)^{\frac{1}{d+3}} + 4\sqrt{2} \frac{16^{\frac{d}{d+3}}}{(18(d+1))^{\frac{d+1}{d+3}}}$. 
Incorporating~\eqref{eq:Rademacherbound_II} into~\eqref{eq:Rademacherbound_NearCompletion} yields: the following holds with probability at-least $1-\delta$
\begin{equation*}
\begin{aligned}
&
    \sup_{\Phi\in \mathcal{NN}_{\Delta,W}^{L,n}} 
        \big|\mathcal{R}(\Phi)-\mathcal{R}_N(\Phi)\big|
    \le 
        2\,L_{\ell}\,\mathfrak{R}_{\mathcal{Z}}(\mathcal{NN}_{\Delta,W}^{L,n})
        +
        \frac{
            \sqrt{8\log(4/\delta)}
        }{
            \sqrt{N}
        }
\\
&
     \mathfrak{R}_{\mathcal{Z}}(\mathcal{NN}_{\Delta,W}^{L,n})
\le 
    \min\biggl\{
            \frac{4}{N^{3/2}}
            +
            \frac{
                26\log(N)\log(2W)
                \,
                W^{3\Delta/2}
            }{
                2^{\max(\{0,\Delta-2\}}
                N
            }
            \,
            \big(
                \sqrt{d}
                +
                C\,\sqrt{N}
                +
                C\,\sqrt{\ln(4/\delta)}
            \big)
    ,
            C_d\,\frac{L^{\frac{d}{d+3}}}{N^{\frac{1}{d+3}}}
    \biggr\}
.
\end{aligned}
\end{equation*}
\end{proof}

\subsection{Proof of Neural McShane Theorem}
\begin{proof}[{Proof of Corollary~\ref{cor:McShane}}]
If $A$ is empty, there is nothing to show.  We thus assume that $A$ is non-empty.
By the version of McShane's extension theorem in~\cite[Theorem 3]{beer2020mcshane} there exists an $\omega$-uniformly continuous $F:[0,1]^d\to \mathbb{R}$ extending $f$; i.e.\ $F(\mathbf{x})=f(\mathbf{x})$ for each $\mathbf{x}\in A$ and $F$ is $\omega$-uniformly continuous.  
Since $A$ is finite (and non-empty), and $2^{-n}\mathbb{Z}^d\subseteq 2^{-n-1}\mathbb{Z}^d$ then, there exists some $n\in \mathbb{N}_+$ such that $A\subseteq 
\{i/n:i=0,\dots,n\}^d$.  Applying Theorem~\ref{thm:MainTheorem_StableUniversalApproximation} to $F$ we deduce that
\[
    F(\mathbf{x})=\Phi(\mathbf{x}) 
        \,\,
        \mbox{ for all } \mathbf{x}\in
        \left\{ \frac{0}{n}, \frac{1}{n}, \cdots, \frac{n}{n} \right\}^d
.
\]
Since $F(\mathbf{x})=f(\mathbf{x})$ for all $\mathbf{x}\in A$ and $A\subseteq \{i/n:i=0,\dots,n\}^d$ we obtain our conclusion.
\end{proof}

\subsection{Proof of Results in the Discussion Section}
\label{s:Proofs__ss:Discussion}

\begin{proof}[{Proof of Theorem~\ref{thm:KuhnUnique}}]
We will prove the following stronger statement: in the context of Lemma \ref{lem:CPWLApprox_wRegControl}, for $n=1$, the Kuhn triangulation is the only triangulation (up to reflections) such that for all Lipschitz functions $f$ with modulus of regularity $\omega$ which is a linear function, $\omega$ is also a modulus of regularity of $\Phi$. By a reflection, we mean a transformation of the form $x_i \mapsto 1-x_i$, for some $i=1,\cdots,d$. In other words, we will only consider Lipschitz functions and linear moduli of regularity.

Note that, all reflections are involutions and they are commutative with each other, so we will refer to a transformation as a reflection of the axes $x_{i_1},\cdots,x_{i_m}$ if it is a composition of the reflections $x_j \mapsto 1-x_j$ for $j=i_1,\cdots,i_m$, in any order; in particular, we say that it reflects the $x_i$-axis if $i\in\{i_1,\cdots,i_m\}$.

For $d=1,2$, the Kuhn triangulation is the only triangulation of $[0,1]^d$ up to reflections, so we will assume that $d\ge 3$ throughout the rest of the proof. 

Let $S \eqdef \{S_k\}_{k=1}^m$ be any triangulation of $[0,1]^d$ that makes Lemma \ref{lem:CPWLApprox_wRegControl} true. We will prove that $S$ is some reflected version of the Kuhn triangulation. 

For $i=1,\cdots,d$ and $\delta=0,1$, let 
$$F_{i,\delta} \eqdef [0,1]^d \cap \{(x_1,\cdots,x_d)\in\mathbb{R}^d : x_i = \delta\}$$ 
be some face of the unit cube $[0,1]^d$, and 
$$S_{i,\delta}  \eqdef  \{S_k \cap F_{i,\delta} : 
S_k \cap F_{i,\delta} \text{ is a } (d-1) \text{-dimensional simplex},
k=1,\cdots,m\}$$
be the ``restricted triangulation'' of $S$ on $F_{i,\delta}$. For any such $(d-1)$-dimensional face $F_{i,\delta}$, we show that the finite collection $S_{i,\delta}$ of $(d-1)$-dimensional simplices is a triangulation of $F_{i,\delta}$:
\begin{enumerate}
    \item $\cup S_{i,\delta} = F_{i,\delta}$: assume for contradiction that $\cup S_{i,\delta}$ is a proper subset of $F_{i,\delta}$. Denote $F_{i,\delta}^o$ as the interior of $F_{i,\delta}$, then $F_{i,\delta}^o \backslash (\cup S_{i,\delta})$ is open and non-empty, since otherwise $(\cup S_{i,\delta})\supseteq F_{i,\delta}^o$, which implies that $(\cup S_{i,\delta})\supseteq F_{i,\delta}$ as $\cup S_{i,\delta}$ is a closed set. Consider the Lebesgue measure on the $(d-1)$-dimensional hyperplane $x_i=\delta$ that contains $F_{i,\delta}$. The measure of $F_{i,\delta}^o \backslash (\cup S_{i,\delta})$ is positive as it is non-empty and open. If $S_k \cap F_{i,\delta}$ is not $(d-1)$-dimensional, then it is at most $(d-2)$-dimensional, which has measure $0$. There are only a finite number of these $(d-2)$-dimensional faces with $0$ measure, so their total measure is $0$, which cannot fill up $F_{i,\delta}^o \backslash (\cup S_{i,\delta})$ as it has positive measure, so we get a contradiction as desired;
    \item The set of vertices of the simplices in $S_{i,\delta}$ is $\{0,1\}^d \cap F_{i,\delta}$, which is the set of vertices of $F_{i,\delta}$;
    \item The interiors of the simplices in $S_{i,\delta}$ are pairwise disjoint: without loss of generality, consider the case $i=1$ and $\delta=0$. Assume for contradiction that $\mathbf{z}$ is contained in the interiors of $S_p$ and $S_q$, for two different simplices in $S_{1,0}$. Assume $S_p=S_{p'}\cap F_{1,0}$ for some $S_{p'}\in S$, then compared to $S_p$, $S_{p'}$ has an additional vertex in the ``upper face'' $F_{1,1}$ (otherwise $S_{p'}$ would not be a simplex in $\mathbb{R}^d$ as all its vertices lie in the $(d-1)$-dimensional hyperplane $x_1=0$), so $\mathbf{z} + t \mathbf{e}_1$ is contained in the interior of $S_{p'}$, for $t>0$ small enough. Similarly, assume $S_q=S_{q'}\cap F_{1,0}$ for some $S_{q'}\in S$, then $\mathbf{z} + t \mathbf{e}_1$ is also contained in the interior of $S_{q'}$ for $t>0$ small enough, which is impossible as $S$ is a triangulation and its simplices has pairwise disjoint interiors, so we get a contradiction as desired.
\end{enumerate}
Therefore, $S_{i,\delta}$ is indeed a triangulation of $F_{i,\delta}$, which is a translated and reflected $(d-1)$-dimensional unit cube $[0,1]^{d-1}$. 

Assume for contradiction that $S_{i,\delta}$ is not a reflected version of the Kuhn triangulation. Without loss of generality, consider the case $i=1$ and $\delta=0$. By induction, there exists a function $f:F_{1,0}\rightarrow\mathbb{R}$ and a linear modulus of regularity $\omega_f$ of $f$ which is not a modulus of regularity of the continuous piecewise linear approximator $\Phi:F_{1,0}\rightarrow\mathbb{R}$ of $f$ with respect to the triangulation $S_{1,0}$. Consider the extension $f^{\ast}$ of $f$ to $[0,1]^d$ defined by:
$$f^{\ast}(x_1,x_2,x_3,\cdots,x_d)  \eqdef  f(0,x_2,x_3,\cdots,x_d), \quad
\forall (x_1,x_2,x_3,\cdots,x_d)\in [0,1]^d$$
Clearly $\omega_f$ is also a modulus of regularity of $f^{\ast}$ as it is non-decreasing. Let $\Phi^{\ast}$ be the continuous piecewise linear approximator of $f$ on $[0,1]^d$ with respect to the triangulation $S$, then the restriction of $\Phi^{\ast}$ on $F_{1,0}$ is $\Phi$, but this implies that $\omega_f$ is not a modulus of regularity of $\Phi^{\ast}$, as this is not so even just on the subset $F_{1,0}$ of $[0,1]^d$. Thus, we get a contradiction as desired.

Therefore, $S_{i,\delta}$ is indeed some reflected version of the Kuhn triangulation, for $i=1,\cdots,d$ and $\delta=0,1$. Since $F_{d,0}$ is the unit cube $[0,1]^{d-1}$ in the $(d-1)$-dimensional hyperplane $x_d=0$, without loss of generality, we may assume that $S_{d,0}$ is the ``original'' Kuhn triangulation of $F_{d,0}$ without any reflections. Then, by definition of the Kuhn triangulation, any simplex in $S_{d,0}$ has an edge connecting $\mathbf{u}_1  \eqdef  (0,0,\cdots,0,0)$ and $\mathbf{u}_2  \eqdef  (1,1,\cdots,1,0)$. Since $S_{d,0}$ is the ``restricted'' triangulation of $S$ on $F_{d,0}$, there is some corresponding simplex in $S$ that has an edge connecting $\mathbf{u}_1$ and $\mathbf{u}_2$. 

Next, we identify the ``separating hyperplanes'' (i.e. each simplex lies on the same side of it) in $S$ and show that they are the same as the ones given in Lemma \ref{lem:triangulation_of_Rd}. We first show that the hyperplanes $x_i-x_j=0$ for all $1 \le i<j \le d-1$ are separating hyperplanes of $S$. 

Assume for contradiction that there is some $1 \le i<j \le d-1$ such that the hyperplane $x_i-x_j=0$ is not a separating hyperplane. Without loss of generality, assume that $i=1,j=2$. Then, there is some simplex in $S$ that has two vertices $\mathbf{v}_1,\mathbf{v}_2$ on different sides of the hyperplane $x_1-x_2=0$, so it has an edge connecting $\mathbf{v}_1,\mathbf{v}_2$ as it is a simplex. Since $\mathbf{v}_1,\mathbf{v}_2\in\{0,1\}^d$ and they lie on different sides of the hyperplane $x_1-x_2=0$, $\mathbf{v}_1 = (0,1,\delta_3^{(1)},\cdots,\delta_d^{(1)})$, $\mathbf{v}_2 = (1,0,\delta_3^{(2)},\cdots,\delta_d^{(2)})$, where the $\delta_a^{(b)}$'s are either $0$ or $1$.

Now, consider the following ``counterexample'' function:
$$f(x_1,x_2,\cdots,x_d)  \eqdef  \sigma(x_1 - x_2) 
+ \sum_{i=3}^{d} x_i, \quad
\forall (x_1,x_2,\cdots,x_d)\in [0,1]^d$$
Clearly $f$ is Lipschitz continuous with Lipschitz constant $1$ (under $\ell^1$ norm of $\mathbb{R}^d$), so $\omega(x) \eqdef x$($\forall x\ge 0)$ is a modulus of regularity of $f$. However, for the continuous piecewise linear approximator $\Phi$ of $f$ on $[0,1]^d$ with respect to the triangulation $S$: consider the pair of points 
$$\mathbf{u}  \eqdef  \frac{\mathbf{u}_1 + \mathbf{u}_2}{2} = 
\left( \frac{1}{2},\frac{1}{2},0,0,\cdots,0,0 \right) $$
$$\mathbf{v}  \eqdef  \frac{\mathbf{v}_1 + \mathbf{v}_2}{2} = 
\left( 
    \frac{1}{2},\frac{1}{2},
    \frac{\delta_3^{(1)} + \delta_3^{(2)}}{2},
    \cdots,
    \frac{\delta_{d-1}^{(1)} + \delta_{d-1}^{(2)}} {2}, 
    \frac{\delta_{d}^{(1)}   + \delta_{d}^{(2)}}   {2}
\right)$$
Since there is some simplex in $S$ that has an edge connecting $\mathbf{u}_1$ and $\mathbf{u}_2$, 
$$\Phi ( \mathbf{u} )  
= \Phi \left( \frac{\mathbf{u}_1 + \mathbf{u}_2}{2} \right)  
= \frac{ \Phi(\mathbf{u}_1) + \Phi(\mathbf{u}_2) } {2}   
= \frac{ f(\mathbf{u}_1) + f(\mathbf{u}_2) } {2} 
= 0$$
Since there is some simplex in $S$ that has an edge connecting $\mathbf{v}_1$ and $\mathbf{v}_2$, 
$$\Phi ( \mathbf{v} )  
= \Phi \left( \frac{\mathbf{v}_1 + \mathbf{v}_2}{2} \right)  
= \frac{ \Phi(\mathbf{v}_1) + \Phi(\mathbf{v}_2) } {2}   
= \frac{ f(\mathbf{v}_1) + f(\mathbf{v}_2) } {2} 
= \frac{1}{2} + \sum_{i=3}^d 
  \frac{\delta_i^{(1)}  + \delta_i^{(2)}} {2}$$
Then, $\omega$ is not a modulus of regularity of $\Phi$ because 
$$|\Phi ( \mathbf{u} ) - \Phi ( \mathbf{v} )| 
= \frac{1}{2} + \sum_{i=3}^d 
  \frac{\delta_i^{(1)}  + \delta_i^{(2)}} {2}
> \sum_{i=3}^d 
  \frac{\delta_i^{(1)}  + \delta_i^{(2)}} {2}
= \omega( \| \mathbf{u} - \mathbf{v} \|_1 )$$
This contradicts our assumption that the continuous piecewise linear approximators of any Lipschitz continuous functions on $[0,1]^d$ with respect to the triangulation $S$ always preserve regularity under $\ell^1$ norm of $\mathbb{R}^d$. Thus, we get a contradiction as desired.

Therefore, the hyperplanes $x_i-x_j=0$ for all $1 \le i<j \le d-1$ are indeed separating hyperplanes of $S$.

By similar arguments as above, switching to other ``pivot axis'' than $x_d$, we can conclude that for any $1\le i\le j\le d$, either $x_i-x_j=0$ or $x_i+x_j=1$ is a separating hyperplane of $S$, taking possible reflections into account. We have shown above that for all $1\le i\le j<d$, the separating hyperplane is of the form $x_i-x_j=0$. Thus, we remain to show that all remaining hyperplanes either all have the form $x_i-x_d=0$ or all have the form $x_i+x_d=1$ for $1\le i\le d-1$, and then the result follows from Lemma \ref{lem:triangulation_of_Rd}.

Assume for contradiction that there are some distinct $i,j$ with $1\le i,j\le d-1$ such that both $x_i-x_d=0$ and $x_j+x_d=1$ are separating hyperplanes of $S$. Without loss of generality, assume that $i=1,j=2$. Then, consider the following $d$ separating hyperplanes of $S$:
$$x_1 - x_d = 0, \quad 
x_2 + x_d = 1, \quad
x_1 - x_2 = 0, \quad
x_k = 0 \text{ for } 3\le k\le d-1 $$
They intersect at a single point 
$\left( \frac{1}{2}, \frac{1}{2}, 0,0,\cdots,0, \frac{1}{2} \right)$, which must be a vertex of some simplex in $S$, but this is impossible since it is not a lattice point. Therefore, we get a contradiction as desired. 
\end{proof}

\begin{proof}[{Proof of Theorem~\ref{thm:main_SOTA_Benchmark}}]
The proof is a mild modification of the one in~\cite{ShenYangZhang_JMPA_OptApprx_ReLU}, 
which approximates the target function by assigning constant values (sample values) on some pairwise disjoint cubes inside $[0,1]^d$, and the rest of $[0,1]^d$ is called ``trifling regions'' where errors are controlled by making these regions extremely small and taking the median of $3^d$ such MLPs with different trifling regions. The problem with this approach is that the Lipschitz constant explodes in these extremely small trifling regions. Here, we partially resolve this issue by taking the median of only $2d+1$ of these MLPs, thus enabling these trifling regions to become as large as possible, which leads to a lower Lipschitz constant. 

For the case $d=1$: Let $x_i=\frac{i}{n^2}$ for $i=0,1,\cdots,n^2$. Since
\[
    \left\lfloor
        \frac{2n+2}{4}
    \right\rfloor
    =
    \left\lfloor
        \frac{n+1}{2}
    \right\rfloor
    \ge
    \frac{n+1}{2} - \frac{1}{2}
    =
    \frac{n}{2}
\]
then by Proposition \ref{prop:fit_a_network_with_two_hidden_layers} (taking $M=2n+1,N=\left\lfloor \frac{2n+2}{4} \right\rfloor$), given the set of no more than $MN$ samples $\{(x_i,f(x_i)): i=0,1,\cdots,n^2 \}$, there exists $\Phi\in\mathcal{NN}(\#\text{input}=1;\text{widthvec}=[2n+1,2n])$ such that $\Phi(x_i)=f(x_i)$ for $i=0,1,\cdots,n^2$, and $\Phi$ is linear on $[x_i,x_{i+1}]$ for $i=0,1,\cdots,n^2-1$. Then, the result follows from Lemma \ref{lem:continuous_piecewise_linear_approximators_preserve_regularity}.

Therefore, in the rest of this proof, we will assume that $d\ge 2$. We will also assume that $f$ is non-negative, since otherwise we can first construct $\Phi$ for the non-negative function $f-\min{f}$ and then add back $\min{f}$ (which exists since $f$ is a continuous function on a compact domain) in the output layer. 

For $k=1,2,\cdots,2d+1$ we define some modifications of these ``trifling region'' in~\cite{ShenYangZhang_JMPA_OptApprx_ReLU}:
\begin{equation}
\label{eq:defintion_of_trifling_region}
    \Omega_k \eqdef \bigcup_{i=1}^d \left\{ 
    \mathbf{x}=(x_1,\cdots,x_d)\in[0,1]^d: x_i\in\bigcup_{j=0}^{n-1} \left(
    \frac{1}{n} \left( \frac{k-1}{2d+1}+j \right) , 
    \frac{1}{n} \left( \frac{k}{2d+1}+j \right) 
    \right) 
    \right\}
\end{equation}
$\Omega_k$ is the union of $nd$ ``thick hyperplanes'', where for each of the $d$ axes, there are $n$ of these parallel and equispaced ``thick hyperplanes'' that are perpendicular to it. Then, $\Omega_k$ separates $[0,1]^d$ into $n^d$ (if $k=1$ or $2d+1$) or $(n+1)^d$ (if $2\le k\le 2d$) non-overlaping equispaced cuboids:
\begin{equation}
\label{eq:defintion_of_NON_trifling_region}
Q_k \eqdef [0,1]^d\backslash\Omega_k=\left[ 
[0,1] 
\left\backslash 
\bigcup_{j=0}^{n-1} \left(
\frac{1}{n} \left( \frac{k-1}{2d+1}+j \right) , 
\frac{1}{n} \left( \frac{k}{2d+1}+j \right) 
\right)
\right.
\right]^d
\end{equation}
Most of these cuboids are $d$-dimensional cubes, except for the ones that touches some face of $[0,1]^d$ for $2\le k\le 2d$. See Figure \ref{fig:Omega_12345} for an illustration of $\Omega_1,\cdots,\Omega_5$ and $Q_1,\cdots,Q_5$ for the case $n=d=2$.

\begin{figure}[!htbp]
\centering

\includegraphics[width=0.25\linewidth]{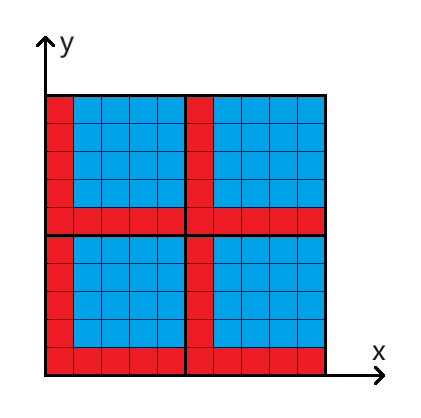}
\includegraphics[width=0.25\linewidth]{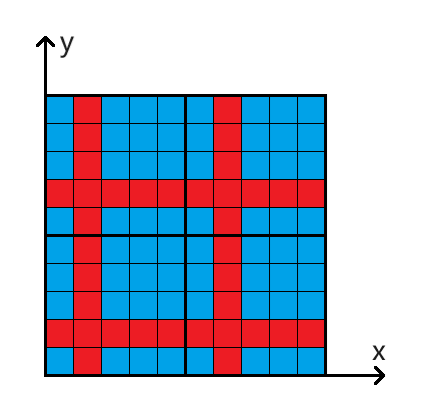}
\includegraphics[width=0.25\linewidth]{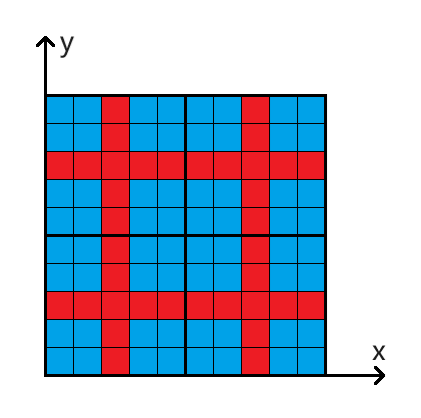}

\hfill\\

\includegraphics[width=0.25\linewidth]{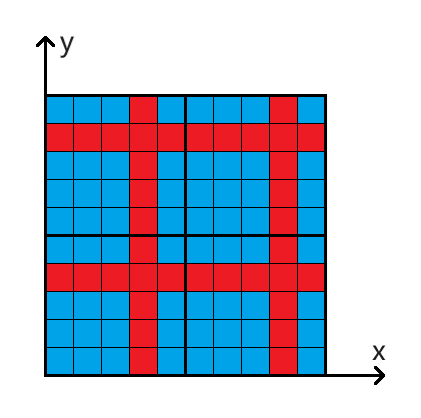}
\includegraphics[width=0.25\linewidth]{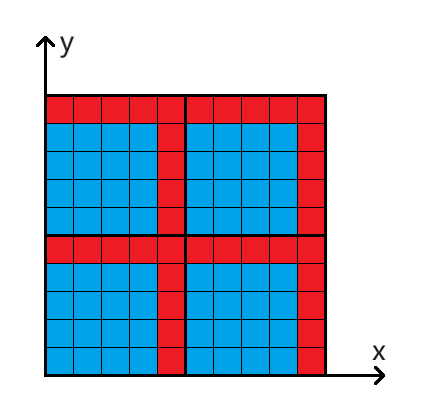}

\caption{An illustration of the trifling regions $\Omega_1,\cdots,\Omega_5$ (colored in red) and $Q_1,\cdots,Q_5$ (colored in blue) for the case $n=d=2$. Each of the smallest cubes appears in blue regions (where approximation errors are controlled) for at least $3$ times. By Lemma \ref{lem:median_estimators_have_small_errors}, the errors of their medians are also controlled. }
\label{fig:Omega_12345}
\end{figure}

\textbf{Step 1}: \textbf{
Construct the projection maps $\pi_k: [0,1]^d \xrightarrow[]{} \mathbb{R}$ for each $Q_k$
}

Fix $2\le k\le 2d$. The cases $k=1$ or $2d+1$ can be treated similarly and require less neurons, but we will assume that they require the same number of neurons as other cases for convenience. We will define $\pi_k$ in such a way that it is constant on each cuboid of $Q_k$, and different cuboids correspond to different constants. Notice that
$$[0,1] \left\backslash 
\bigcup_{j=0}^{n-1} \left(
\frac{1}{n} \left( \frac{k-1}{2d+1}+j \right) , 
\frac{1}{n} \left( \frac{k}{2d+1}+j \right) 
\right)
\right.$$
is the union of $n+1$ equi-spaced closed intervals, so we define a map $p_k:[0,1] \xrightarrow[]{} \mathbb{R}$ which maps the first of them (counting from the left) to $0$, the second of them to $1$,\dots, the last one of them to $n$, and let $p_k$ be linear on each of the remaining intervals such that it is continuous. Then, $p_k$ has $2n+2$ break points (including $0$ and $1$), so by \citep[Lemma 2.1]{shen2019nonlinear}, $p_k$ can be implemented by a ReLU MLP with width $2n+1$ and depth $1$. Then, for $\mathbf{x}=(x_1,\cdots,x_d)\in\mathbb{R}^d$, we define
$$\pi_k(\mathbf{x})=\sum_{i=1}^d (n+1)^{i-1}p_k(x_i)$$
which can be seen as the number $(\overline{p_k(x_d)p_k(x_{d-1})\cdots p_k(x_1)})$ in base $n+1$, so it is easy to see that $\pi_k$ indeed maps each of the $(n+1)^d$ cuboids to different constants, namely $0,1,2,\cdots,(n+1)^d-1$. Then, we implement $\pi_k$ as follows:
$$\begin{bmatrix} x_1 \\ x_2 \\ \vdots \\ x_d \end{bmatrix}
\ \Longrightarrow \
\begin{bmatrix} p_k(x_1) \\ p_k(x_2) \\ \vdots \\ p_k(x_d) \end{bmatrix}
\ \Longrightarrow \
\sum_{i=1}^d (n+1)^{i-1}p_k(x_i)$$
so $\pi_k$ has width $(2n+1)d$ and depth $3$ (counting the output layer, which is still a hidden layer of the final network).

\textbf{Step 2}: \textbf{
Construct a $1$-dimensional memorizer $\phi_k:\mathbb{R} \rightarrow \mathbb{R}$ for each $Q_k$
}

Fix $2\le k\le 2d$. As in Step 1, the cases $k=1$ or $2d+1$ can be treated similarly and require less neurons, but we will assume that they require the same number of neurons as other cases for convenience.

For $j=0,1,2,\cdots,(n+1)^d-1$, let $y_j$ be the value of $f$ at the center of the cuboid that is mapped to $j$ under $\pi_k$. Note that the range of $\pi_k$ is $[0,(n+1)^d-1]$, so we don't need to care about the behavior of our memorizer outside this interval.

Let $\{(j,y_j):j=0,1,2,\cdots,(n+1)^d-1 \}$ be a sample set of size $(n+1)^d$ for our memorizer. Since  
\[
    \left\lfloor 
        \frac{2 \lceil (n+1)^{d/2} \rceil + 2}{4} 
    \right\rfloor
    =
    \left\lfloor 
        \frac{ \lceil (n+1)^{d/2} \rceil + 1}{2} 
    \right\rfloor
    \ge
    \frac{ \lceil (n+1)^{d/2} \rceil + 1}{2} - \frac{1}{2}
    \ge
    \frac{ (n+1)^{d/2}}{2}
\]
then by Proposition \ref{prop:fit_a_network_with_two_hidden_layers} (taking $M=2 \lceil (n+1)^{d/2} \rceil$, 
$N= \big\lfloor 
        \frac{2 \lceil (n+1)^{d/2} \rceil + 2}{4} 
    \big\rfloor$),
there is a ReLU MLP $\phi_k$ with width $2\lceil (n+1)^{d/2} \rceil$ and depth 3 (counting the output layer, which is still a hidden layer of the final network) that memorizes this sample set, and $\phi_k$ is linear on $[j,j+1]$ for $j=0,1,\cdots,n^d-2$.

\textbf{Step 3}: \textbf{
Combine Step 1, Step 2 and Lemma \ref{lem:relu_for_computing_median} to construct the final network $\Phi$
}

We denote the ReLU MLP implementing the median function on $2d+1$ non-negative arguments as $\mathbf{M}_{2d+1}$, which has width $6d+3$ and depth $11d+3$ by Lemma \ref{lem:relu_for_computing_median}, then we construct $\Phi$ as follows:
\begin{align*}
\mathbf{x}
\ &\Longrightarrow \
\begin{bmatrix} \mathbf{x} \\ \pi_1(\mathbf{x}) \end{bmatrix}
\ \Longrightarrow \
\begin{bmatrix} \mathbf{x} \\ \pi_1(\mathbf{x}) \\ \pi_2(\mathbf{x}) \end{bmatrix}
\ \Longrightarrow \
\cdots
\ \Longrightarrow \
\begin{bmatrix} \mathbf{x} \\ \pi_1(\mathbf{x}) \\ \pi_2(\mathbf{x}) \\ \vdots \\ \pi_{2d}(\mathbf{x}) \\ \end{bmatrix}
\ \Longrightarrow \
\begin{bmatrix} \pi_1(\mathbf{x}) \\ \pi_2(\mathbf{x}) \\ \vdots \\ \pi_{2d}(\mathbf{x}) \\ \pi_{2d+1}(\mathbf{x}) \\ \end{bmatrix} 
\\
\ &\Longrightarrow \
\begin{bmatrix}  \phi_1\circ\pi_1(\mathbf{x}) \\ \pi_2(\mathbf{x}) \\ \vdots \\ \pi_{2d+1}(\mathbf{x}) \\ \end{bmatrix}
\ \Longrightarrow \
\begin{bmatrix}  \phi_1\circ\pi_1(\mathbf{x}) \\ \phi_2\circ\pi_2(\mathbf{x}) \\ \vdots \\ \pi_{2d+1}(\mathbf{x}) \\ \end{bmatrix}
\ \Longrightarrow \
\cdots
\ \Longrightarrow \
\begin{bmatrix}  \phi_1\circ\pi_1(\mathbf{x}) \\ \phi_2\circ\pi_2(\mathbf{x}) \\ \vdots \\ \phi_{2d+1}\circ\pi_{2d+1}(\mathbf{x}) \\ \end{bmatrix}\
\\
\ &\Longrightarrow \
\mathbf{M}_{2d+1} \left(
\begin{bmatrix}  \phi_1\circ\pi_1(\mathbf{x}) \\ \phi_2\circ\pi_2(\mathbf{x}) \\ \vdots \\ \phi_{2d+1}\circ\pi_{2d+1}(\mathbf{x}) \\ \end{bmatrix}
\right)
\ \Longrightarrow \
\Phi(\mathbf{x})
\end{align*}
Its width and depth are 
\begin{align*}
    \operatorname{width}(\Phi) 
    &= 
    \max\left\{\operatorname{width}(\pi_{2d+1})+2d,\operatorname{width}(\phi_{2d+1})+2d,\operatorname{width}(M_{2d+1})\right\} \\
    &= 
    \max\left\{(2n+3)d, 6d+3, 2\lceil (n+1)^{d/2} \rceil+2d \right\}
\end{align*}
\begin{align*}
    \operatorname{depth}(\Phi) &= \sum_{k=1}^{2d+1}\operatorname{depth}(\pi_k) + \sum_{k=1}^{2d+1}\operatorname{depth}(\phi_k) + \operatorname{depth}(M_{2d+1}) \\
    &= 3(2d+1)+3(2d+1)+(11d+3) \\
    &= 23d+9
\end{align*}

\textbf{Step 4}: \textbf{
Estimate the $L^{\infty}$ error
}

We first show that the $L^{\infty}$ error of $\phi_k\circ\pi_k$ is controlled on $Q_k$ for all $k$: pick any $\mathbf{w}\in Q_k$, then since $\pi_k$ maps the cuboid in $Q_k$ containing $\mathbf{w}$ (whose center we denote by $\mathbf{w}_0$) to some constant $j$, which is further mapped by $\phi_k$ to $y_j \eqdef f(\mathbf{w}_0)$, we have $\phi_k\circ\pi_k(\mathbf{w}) = f(\mathbf{w}_0)$, so
$$|\phi_k\circ\pi_k(\mathbf{w})-f(\mathbf{w})| = |f(\mathbf{w}_0)-f(\mathbf{w})| \le \nu \| \mathbf{w}_0-\mathbf{w} \|_1^{\alpha}$$
Since the edge length of any cuboid in $Q_k$ is at most $l \eqdef \frac{2d}{n(2d+1)}$ (this can be seen from Equation \ref{eq:defintion_of_NON_trifling_region}), $\| \mathbf{w}_0-\mathbf{w} \|_1 \le \frac{d}{2}l = \frac{d^2}{n(2d+1)}$, then
$$|\phi_k\circ\pi_k(\mathbf{w})-f(\mathbf{w})| \le \nu \| \mathbf{w}_0-\mathbf{w} \|_1^{\alpha} \le \nu \left( \frac{d^2}{n(2d+1)} \right)^{\alpha}$$

Now, for any $\mathbf{x}=(x_1,x_2,\cdots,x_d)\in[0,1]^d$, we show that it belongs to at least $d+1$ out of $2d+1$ $Q_k$'s: notice that in the definition of the trifling regions $\Omega_k$ (Equation \ref{eq:defintion_of_trifling_region}), each $x_i$ belongs to at most one of the following:
$$\bigcup_{j=0}^{n-1} \left(
\frac{1}{n} \left( \frac{k-1}{2d+1}+j \right) , 
\frac{1}{n} \left( \frac{k}{2d+1}+j \right) 
\right),
\quad \text{for } k=1,2,\cdots,2d+1 $$
so $\mathbf{x}$ belongs to at most $d$ trifling regions, thus $\mathbf{x}$ belongs to at least $d+1$ of the $Q_k$'s.

Therefore, for any $\mathbf{x}=(x_1,x_2,\cdots,x_d)\in[0,1]^d$, at least $d+1$ of $\phi_k\circ\pi_k(\mathbf{x})$ (for $k=1,2,\cdots,2d+1$) belong to the interval $\Big[
f(\mathbf{x}) - \nu \big( \frac{d^2}{n(2d+1)} \big)^{\alpha} ,
f(\mathbf{x}) + \nu \big( \frac{d^2}{n(2d+1)} \big)^{\alpha} 
\Big]$, then by Lemma \ref{lem:median_estimators_have_small_errors}, 
\[
    \Phi(\mathbf{x})
    =
    \operatorname{median}\left\{
        \phi_k\circ\pi_k(\mathbf{x}): 1\le k\le 2d+1
    \right\}
    \in
    \left[
        f(\mathbf{x}) 
        - 
        \nu \Big( \frac{d^2}{n(2d+1)} \Big)^{\alpha}
        ,
        f(\mathbf{x}) 
        + 
        \nu \Big( \frac{d^2}{n(2d+1)} \Big)^{\alpha}
    \right]
\]
hence
$$\|f-\Phi\|_{L^{\infty}([0,1]^d)} \le \nu \left( 
\frac{d^2}{n(2d+1)}
\right)^{\alpha}$$

\textbf{Step 5}: \textbf{
Estimate the Lipschitz constants
}

To compute the Lipschitz constants, we first fix $2\le k\le 2d$, and compute the Lipschitz constant of $\phi_k\circ\pi_k$. It is easy to see that $\phi_1\circ\pi_1$ and $\phi_{2d+1}\circ\pi_{2d+1}$ have at least the same regularity since there are less cuboids in $Q_1$ and $Q_{2d+1}$. Finally, by Lemma \ref{lem:median_function_preserves_regularity}, their median has the same regularity. 

We first compute the Lipschitz constant of $\pi_k$. Let $\mathbf{x}\in[0,1]^d$ and $\Delta\mathbf{x}$ small, we want to find an upper bound for $|\pi_k(\mathbf{x}+\Delta\mathbf{x})-\pi_k(\mathbf{x})|$. By definition of $p_k$, it only increases linearly by $1$ inside intervals of length $\frac{1}{n(2d+1)}$, so $|p_k'|\le n(2d+1)$, thus 
$$\|p_k(\mathbf{x}+\Delta\mathbf{x})-p_k(\mathbf{x})\|_1 \le n(2d+1) \|\Delta\mathbf{x}\|_1$$
where $p_k$ applies to $\mathbf{x}$ elementwise. Let $\mathbf{n} \eqdef \begin{bmatrix} (n+1)^0 & (n+1)^1 & \cdots & (n+1)^{d-1} \end{bmatrix}^{\top}$, then
\begin{align*}
    |\pi_k(\mathbf{x}+\Delta\mathbf{x})-\pi_k(\mathbf{x})| 
    &= 
    |\mathbf{n}\cdot p_k(\mathbf{x}+\Delta\mathbf{x})-\mathbf{n}\cdot p_k(\mathbf{x})| \\
    &=
    |\mathbf{n}\cdot (p_k(\mathbf{x}+\Delta\mathbf{x})-p_k(\mathbf{x}))| \\
    &\le
    (n+1)^{d-1} \cdot \|p_k(\mathbf{x}+\Delta\mathbf{x})-p_k(\mathbf{x})\|_1 \\
    &\le
    (n+1)^{d}(2d+1) \|\Delta\mathbf{x}\|_1
\end{align*}
therefore
$$\operatorname{Lip}(\pi_k) 
\le (n+1)^d(2d+1)$$

Now we compute the Lipschitz constant of $\phi_k$. As mentioned above, let
$$(\overline{a_da_{d-1}\cdots a_1})_{n+1} \eqdef \sum_{i=1}^d (n+1)^{i-1}a_i$$ 
denote the number $\overline{a_1a_2\cdots a_d}$ in base $n+1$. Let $P_{(\overline{a_da_{d-1}\cdots a_1})_{n+1}}^{(k)}$ be the center of the unique cuboid in $Q_k$ such that $p_i\left(P_{(\overline{a_da_{d-1}\cdots a_1})_{n+1}}^{(k)}\right)=a_i$ for $i=1,2,\cdots,d$. Note that $a_i\in\{0,1,2,\cdots,n\}$ for all $i$. Then, $\phi_k$ is the unique continuous piecewise linear function that goes through 
$\{(j, f(P_j^{(k)}) ): j=0,1,2,\cdots,(n+1)^d-1 \}$ and is linear between adjacent points in this set. Then, 

$$\operatorname{Lip}(\phi_k)
= \max\limits_{j=0,1,2,\cdots,(n+1)^d-2}{\left|f(P_{j+1}^{(k)})-f(P_j^{(k)})\right|} 
\le
\max\limits_{j=0,1,2,\cdots,(n+1)^d-2}
{\nu \left\|P_{j+1}^{(k)}-P_j^{(k)}\right\|_1^{\alpha}} 
\le
\nu d^{\alpha}
$$

Finally, by Lemma \ref{lem:median_function_preserves_regularity}, the median function preserves regularity, hence
\begin{align*}
    \operatorname{Lip}(\Phi) 
    &\le \max\limits_{k=1,2,\cdots,2d+1}{\operatorname{Lip}(\phi_k\circ\pi_k)} \\
    &\le \max\limits_{k=1,2,\cdots,2d+1}{\operatorname{Lip}(\phi_k)} \cdot {\operatorname{Lip}(\pi_k)} \\
    &\le \nu (n+1)^d (2d+1) d^{\alpha}.
\end{align*}
This concludes our proof.
\end{proof}


\section{Acknowledgments}
\label{s:Acknowledgments}
\paragraph{Funding Sources} The authors were funded by the NSERC Discovery grant RGPIN-2023-04482.

\bibliographystyle{plain}
\bibliography{0_References}

\end{document}